\documentclass[11pt]{article}

\usepackage[margin=1in]{geometry}
\usepackage{framed} 
\usepackage{mathtools}
\usepackage[]{amsmath,amssymb,epsfig}
\usepackage{amsthm}
\usepackage{amsmath}
\usepackage{amssymb}
\usepackage{graphicx}
\usepackage{epstopdf}
\usepackage{comment}
\usepackage{array}
\usepackage{algorithm}
\usepackage{url}
\usepackage{ifthen}
\usepackage{wrapfig}
\usepackage{lscape}
\usepackage{algpseudocode}
\usepackage{setspace}
\usepackage{multicol}
\usepackage{multirow}
\usepackage{color}
\usepackage{colortbl}
\usepackage{xcolor}
\usepackage{rotating}
\usepackage{caption}
\usepackage{float}
\usepackage{ifthen}
\usepackage{placeins}
\usepackage{framed}
\usepackage{bbm}
\usepackage{enumitem}  

\usepackage[%
    font={small,sf},
    labelfont=bf,
    format=hang,    
    format=plain,
    margin=0pt,
    width=0.8\textwidth,
]{caption}

\usepackage[list=true]{subcaption}
\usepackage{bbold}

\newtheorem{theorem}{Theorem}

\newtheorem{assumption}{Assumption}

\newtheorem{corollary}{Corollary}

\newtheorem{lemma}{Lemma}

\newtheorem{proposition}{Proposition}
\newtheorem{remark}{Remark}

\numberwithin{equation}{section}
\newcommand{\ER}{Erd\H{o}s-R\'enyi }   

\newcommand{\calC}{\ensuremath{\mathcal{C}}}

\newcommand{\calH}{\ensuremath{\mathcal{H}}}

\newcommand{\calG}{\ensuremath{\mathcal{G}}}

\newcommand{\calP}{\ensuremath{\mathcal{P}}}

\newcommand{\calN}{\ensuremath{\mathcal{N}}}

\newcommand{\calV}{\ensuremath{\mathcal{V}}}

\newcommand{\calE}{\ensuremath{\mathcal{E}}}
\newcommand{\calL}{\ensuremath{\mathcal{L}}}

\DeclareMathOperator{\Tr}{Tr}

\newcommand{\norm}[1]{\|{#1}\|}
\newcommand{\abs}[1]{|{#1}|}

\newcommand{\floor}[1]{\Big\lfloor{#1}\Big\rfloor}
\newcommand{\paren}[1]{\Big({#1}\Big)}
\newcommand{\set}[1]{\left\{{#1}\right\}}

\newcommand{\est}[1]{\widehat{#1}}
\newcommand{\expec}{\ensuremath{\mathbb{E}}}
\newcommand{\matR}{\ensuremath{\mathbb{R}}}
\newcommand{\matZ}{\ensuremath{\mathbb{Z}}}
\newcommand{\matC}{\ensuremath{\mathbb{C}}}
\newcommand{\matN}{\ensuremath{\mathbb{N}}}

\newcommand{\prob}{\ensuremath{\mathbb{P}}}

\newcommand{\indic}{\ensuremath{\mathbbm{1}}} 

\newcommand{\real}{\text{Re}}
\newcommand{\imag}{\text{Im}}

\newcommand{\setij}{\ensuremath{\set{i,j}}}
\newcounter{ale}

\newenvironment{liste}{\begin{itemize}}{\end{itemize}}
\newcommand{\aliste}{\begin{liste} \setcounter{ale}{1}}
\newcommand{\zliste}{\end{liste}}

\usepackage[colorlinks=true,linkcolor=red,filecolor=green,citecolor=red]{hyperref}

\definecolor{purple}{RGB}{178,55,250} % 

\newcounter{noteMCctr} 

\newcounter{noteHTctr} 

\definecolor{applegreen}{rgb}{0.55, 0.71, 0.0}
\newcounter{noteEActr} 

\newcommand{\rev}[1]{\textcolor{black}{#1}}
\newcommand{\revpostpub}[1]{\textcolor{black}{#1}}

\newcommand{\conj}[1]{{#1}^{\mathsf{H}}}
\newcommand{\erdos}{Erd\H{o}s}

\newcommand{\concat}{\operatorname{concat}
}

\newcommand{\blkdiag}{\operatorname{blkdiag}}
\newcommand{\tilgt}{\tilde{g}^*}
\newcommand{\gt}{g^*}
\newcommand{\etg}{\est{\tilde{g}}}
\newcommand{\globaltrs}{\texttt{GTRS-DynSync}}
\newcommand{\localtrs}{{\texttt{LTRS-GS-DynSync}}}
\newcommand{\matdenoising}{{\texttt{GMD-LTRS-DynSync}}}
\newcommand{\matdenoisingLocSpec}{{\texttt{GMD-LSPEC-DynSync}}} % HT: Added this in revision
\newcommand{\ppm}{{\texttt{PPM-DynSync}}}
\newcommand{\projperp}{\calP^\perp_{\tau,n}}
\newcommand{\projtaun}{\calP_{\tau,n}}\newcommand{\ones}{\mathbbm{1}}

\newcommand{\opR}{\operatorname{R}}\newcommand{\opI}{\operatorname{I}}

\newcommand{\teonek}{\hat{\theta}_1(k)}
\newcommand{\phik}{\phi(k)}
\newcommand{\Gshift}{\overline{G}}
\newcommand{\pmax}{p_{\max}}
\newcommand{\smin}{s_{\min}}
\newcommand{\pmin}{p_{\min}}

\usepackage{ifthen}

\newboolean{withSmall}
\setboolean{withSmall}{false}
\setboolean{withSmall}{true} 
\ifthenelse{\boolean{withSmall}}{   }{}

\newcommand\withMaybeSmall{\ifthenelse{\boolean{withSmall}}{ \small }{}}
\newboolean{arXivMode}
\setboolean{arXivMode}{true}
\usepackage{cleveref}
\usepackage{lineno}
\begin{document}

\title{Dynamic angular synchronization under smoothness constraints}

\author{Ernesto Araya\footnotemark[1]\\ \texttt{araya@math.lmu.de} \and Mihai Cucuringu\footnotemark[2] \footnotemark[3]\\ \texttt{mihai@math.ucla.edu} \and Hemant Tyagi \footnotemark[4] \footnotemark[5]\\ \texttt{hemant.tyagi@ntu.edu.sg}}

\renewcommand{\thefootnote}{\fnsymbol{footnote}}
\footnotetext[1]{Department of Mathematics, Ludwig-Maximilians-Universit\"at M\"unchen, Geschwister-Scholl-Platz 1, Munich, 80539, Germany}
\footnotetext[2]{Department of Mathematics, University of California Los Angeles, 520 Portola Plaza, Los Angeles, CA 90095, USA}
\footnotetext[3]{Department of Statistics \&  Oxford-Man Institute of Quantitative Finance, University of Oxford, Oxford, UK}
\footnotetext[4]{Division of Mathematical Sciences, 
      School of Physical and Mathematical Sciences, 
      Nanyang Technological University, 
      21 Nanyang Link, 637371, Singapore}      
\footnotetext[5]{Corresponding author \\ \indent \indent Authors are listed in alphabetical order. Part of this work was done while E.A was affiliated to Inria Lille. The first version of this paper was prepared while H.T was affiliated to Inria Lille.}

\renewcommand{\thefootnote}{\arabic{footnote}}

\maketitle

\begin{abstract}
Given an undirected measurement graph $\calH = ([n], \calE)$, 
the classical angular synchronization problem consists of recovering unknown angles $\theta_1^*,\dots,\theta_n^*$ from a collection of noisy pairwise measurements of the form $(\theta_i^* - \theta_j^*) \mod 2\pi$, for all $\set{i,j} \in \calE$. This problem arises in a variety of applications, including computer vision, time synchronization of distributed networks, and ranking from pairwise comparisons. In this paper, we consider a dynamic version of this problem where the angles, and also the measurement graphs evolve over $T$ time points. Assuming a smoothness condition on the evolution of the
latent angles, we derive three algorithms for joint estimation of the angles over all time points. Moreover, for one of the algorithms, we establish non-asymptotic recovery guarantees for the mean-squared error (MSE) under different statistical models. In particular, we show that the MSE converges to zero as $T$ increases under milder conditions than in the static setting. This includes the setting where the measurement graphs are highly sparse and disconnected, and also when the measurement noise is large and can potentially increase with $T$. We complement our theoretical results with experiments on synthetic data.
\end{abstract}

% \textbf{Keywords:} group synchronization, spectral algorithms, matrix perturbation theory, singular value decomposition, random matrix theory. 

{
  \tableofcontents
}

\section{Introduction}\label{sec:introduction}

% \HT{Need to fill this section}

% 1. paragraph on group sync
The group synchronization problem, which consists of recovering group elements $g_1^*, \ldots, g_n^*$ from a potentially sparse set of noisy pairwise ratios $ g_i^* g_j^{*-1}$, has attracted considerable attention in the last decade, starting with the seminal paper of Singer \cite{sync}. The motivation for this problem  originated from the cryo-EM \cite{shkolnisky2012viewing}  and sensor network localization  \cite{asap2d} literatures, but instantiations of this problem occur in a variety of application domains, including computer vision \cite{structFromMotion_Amit} in the context of the structure-from-motion problem \cite{survery_structure_motion}, ranking from pairwise comparisons \cite{syncRank}, community detection in social networks \cite{Z2synchronization}, NMR spectroscopy \cite{asap3d}, temporal ordering and registration of images in image processing applications \cite{temporal_image_reg}, and time synchronization of distributed networks \cite{timeoneSync}. 
% - give here some examples: cryo-EM, SNL, ranking, computer vision (structure from motion problem).

In the most typical setting, the group elements $g^*_1,\ldots, g^*_n$ are represented as the nodes of a simple undirected graph $\calH=([n],\calE)$, where an edge $\{i,j\}$ is present if and only if the observer has access to a noisy version of $g^*_i{g^*_j}^{-1}$. In the particular case where  $\calH$ is connected and the exact value $g^*_i{g^*_j}^{-1}$ is observed, the recovery problem has a straightforward solution. Indeed, it suffices to consider a spanning tree of $\calH$, assign any value to an arbitrarily chosen root, and the remaining $n-1$ values are found by traversing the edges of tree, solving each time a simple equation. Consequently, these $n-1$ values are uniquely determined up to an global ``shift'' defined by the chosen value of the root node. On the other hand, if $\calH$ is disconnected, there is insufficient information to determine a global solution, since there are no comparisons between the elements of different connected components. 
%Later, we will demonstrate that when a smoothly evolving sequence of measurement graphs is available, this restriction can be significantly relaxed. %This implies that the solution is unique up to a global shift, since shifting all angles by a constant value does not alter their pairwise differences.

%------------------------------
% Angular synchronization
%-------------------------------
\subsection{Angular synchronization} 
The simplest groups for which the synchronization problem has been studied are $\mathbb{Z}_2$, in the context of community detection and signed clustering with two clusters \cite{Z2synchronization,unifiedOd_JACHA}, 
and the non-compact group $\mathbb{R}$  arising in the synchronization of clocks in a distributed network \cite{timeoneSync}. 
The special orthogonal group $\text{SO}(d)$ is the most commonly studied group with a particular emphasis on $\text{SO}(2)$ -- often referred to as angular synchronization -- 
which finds applications in various practical domains, including 2D sensor network localization \cite{asap2d} and ranking from pairwise comparisons \cite{syncRank}.

In the angular synchronization problem, which is the focus of this paper, the goal is to recover $n$ unknown angles  $\theta_1^*,\ldots,\theta_n^* \in [0,2\pi)$ from a set of $m$ noisy measurements of their pairwise offsets $\Theta_{ij} = (\theta_i^* - \theta_j^*) \mod 2\pi$. Equivalently, we can represent each angle $\theta^*_i$ by a unit modulus complex number $g^*_i=e^{\iota \theta^*_i}$ and the pairwise offsets by $g^*_i\conj{g^*_j}$, where for any $z\in\matC$ , $\conj{z}$ represent its complex conjugate. As explained above, the available pairwise measurements can be encoded into the edges of a measurement graph $\calH$, which in most typical applications can be modeled as a sparse random geometric graph (e.g., disc graph), a $k$-nearest-neighbor graph, or an {\ER} graph. The difficulty of the problem is amplified, on one hand, by the amount of \textbf{noise} in the available pairwise offset measurements, and on the other hand by the \textbf{sparsity} of the measurements, since typically,  $\abs{\calE}\ll {n \choose 2}$, i.e., only a small subset of all possible pairwise offsets are available to the user. 

Aiming to reconstruct the angle offsets as best as possible, the initial approach of Singer \cite{sync} considered a quadratic maximization problem over the angular domain, that aims to maximize the number of preserved angle offsets. Denote by $A$ the measurement matrix of size  $n \times n$ (which can be seen as a weighted adjacency matrix of $\calH=([n],\calE)$), where for each $\setij \in {[n] \choose 2}$, one considers the angular embedding 
\begin{equation*}
{A}_{ij} = \begin{cases}
 e^{\imath \Theta_{ij}} & \text{if } \setij \in \calE \\
 0 & \text{if } \setij \notin \calE 
\end{cases}, 
\end{equation*}
with diagonal entries $A_{ii} = 0$. Note that $A$ is Hermitian since $\Theta_{ji} = (-\Theta_{ij}) \mod 2\pi$. One then considers the following maximization problem 
\begin{equation}
\underset{  \theta_1,\ldots,\theta_n \in [0,2\pi)   }{\max}  \sum_{i,j=1}^{n} e^{-\iota \theta_i} A_{ij} e^{\iota \theta_j} 
\equiv 
\underset{  g_1,\ldots,g_n \in \mathbb{C};  \;\;\; |g_i| = 1 }{\max} \;\; \sum_{i,j=1}^{n}  g_i^*  A_{ij}  g_j,
\label{eq:angularMaxObj}
\end{equation}
which is in general NP hard to solve \cite{qp_nphard2006}. This intractability is usually circumvented by considering continuous relaxations of \eqref{eq:angularMaxObj}, leading typically to a spectral algorithm \cite{sync}, or a semi-definite program (SDP) \cite{sync,bandeira2017tightness}. Other algorithmic approaches include message passing algorithms \cite{perry2018message,bandeira2018notes}, those based on graph neural networks \cite{GNN_Sync}, and also the generalized power method \cite{Liu2017,boumal2016}. The statistical performance of algorithms is typically studied under the following two noise models.
\begin{itemize}
    \item \textbf{Additive Gaussian noise (AGN) model \cite{bandeira2017tightness}}. Here, $\calH$ is the simple complete graph and the observations are of the form $g^*_i\conj{g^*_j}+\sigma W_{ij}$, where $(W_{ij})_{i < j}$ are independent complex standard Gaussian's, $W_{ii} = 0$ (for all $i$), and $\sigma>0$ represents the noise level. 
    Note that \eqref{eq:angularMaxObj} is the maximum likelihood estimator (MLE) for the AGN model. 
    \item \textbf{Outliers model \cite{sync}.} Here, $\calH$ is modeled as an \ER graph with connection probability $p$. When $p$ is small, the resulting measurement graph will be sparse and potentially disconnected (e.g., when $p = o(\frac{\log n}{n})$). Under this model, whenever $\{i,j\}$ is an edge in $\calH$, we observe either the true value $g^*_i \conj{g^*_j}$ with probability $\eta$, or a value uniformly distributed on the complex unit circle with probability $1 - \eta$. The parameter $\eta$ quantifies the level of noise. 
\end{itemize}
Several results now exist pertaining to the statistical error analysis of the aforementioned algorithms for the above models. For instance, the SDP is known to be tight\footnote{This means that its global solution recovers that of \eqref{eq:angularMaxObj}.} for the AGN model if $\sigma$ is not too large w.r.t $n$ \cite{bandeira2017tightness, zhong2018near}. In particular, this was first shown in \cite{bandeira2016} for $\sigma = O(n^{1/4})$, and was later improved to $\sigma = O(\sqrt{n/\log n})$ in \cite{zhong2018near}. 
Recovery guarantees for the spectral method exist for both the Outliers model  \cite{sync}, and the AGN model (e.g., \cite{boumal2016, zhong2018near}). It was shown in \cite{zhong2018near} for the AGN model that the spectral method recovers $g^*$ with an $\ell_2$ error bound of $O(\sigma)$, and an $\ell_{\infty}$ error bound of $O(\sigma\sqrt{\log n/n})$, provided $\sigma = O(\sqrt{n/\log n})$. For this regime of $\sigma$, it was also shown in \cite{zhong2018near} that (with spectral initialization) the iterates of the generalized power method (GPM) converge linearly to the global solution of \eqref{eq:angularMaxObj}. This improved upon previous convergence results for the GPM \cite{boumal2016,Liu2017} which required stronger conditions on $\sigma$.

\paragraph{Towards dynamic group synchronization.} In many applications of group synchronization, the underlying measurement graph, along with the associated pairwise data evolve with time. This is the case, e.g., for the ranking problem where the strengths of players, and hence the pairwise comparison outcomes evolve with time (e.g., sports tournaments \cite{cattelan2013dynamic,lopez2018often}, recommender systems), and has been the subject of recent work \cite{dynamicBTL_2020, AKT_dynamicRankRSync, KT_JMLR_BTL_NN}. In \cite{syncRank}, this problem was formulated as an instance of the angular synchronization problem, along with an algorithm \textsc{SyncRank} which leads to state-of-art performance empirically. It is also the case for the signed community detection problem (with two communities), where the  network and the underlying communities can evolve with time (e.g., \cite{dyn_signed_comm20}). This motivates us to initiate a principled study of group synchronization in a dynamic setting where the latent group elements, and the measurement graphs can both evolve with time.

%-----------------------------------
% Dynamic angular synchronization
%----------------------------------
\subsection{Dynamic angular synchronization}  
In this paper, we introduce a dynamic version of the angular synchronization problem. Here, we are given a sequence of measurement graphs of the form $\calH_k = ([n], \calE_k)$,  over $T$ time points $k = 1, \ldots, T$. At each $k$, we are given noisy pairwise information about the latent signal $g^{*}(k) \in \matC^n$ as per the AGN model, or the Outliers model; see Section \ref{subsec:prob_setup}. 
While the measurements are assumed to be independently generated over time, we assume that $g^*(k)$'s evolve smoothly in the sense that
\begin{equation} \label{eq:smoothness_def_intro}
    \sum_{k=1}^{T-1} \norm{g^*(k) - g^*(k+1)}_2^2 \leq S_T, 
\end{equation}
for a suitably small smoothness parameter $S_T$. 
The above condition essentially stipulates that the quadratic variation of $g^*(k)$ w.r.t the path graph on $T$ vertices is small. A similar assumption was recently used in \cite{AKT_dynamicRankRSync} in the context of dynamic translation synchronization, which is also the main motivation for the present paper.

Our goal then is to recover the vector $g^* \in \matC^{nT}$, obtained by column-stacking $g^*(k)$'s.  Specifically, we aim to output an estimate $\est g$ of $g^*$ that is weakly-consistent w.r.t $T$ in terms of the mean squared error (MSE), i.e., 
\begin{equation}\label{eq:consistency_intro}
    \frac1T \|\est g-g^*\|^2_2 = \frac1T \sum_{k=1}^T \norm{\est g(k) - g^*(k)}_2^2 \stackrel{T\to\infty}{\longrightarrow}0.
\end{equation}
The following points are useful to keep in mind.
\begin{enumerate}
    \item When all the graphs are connected, one could estimate each $g^{*}(k)$ by only using the pairwise information at time $k$. However this will typically lead to $\frac1T \|\est g-g^*\|^2_2 = \Omega(1)$ as $T\to\infty$. Exploiting the smoothness information over $k$ is then, in a sense, necessary to ensure \eqref{eq:consistency_intro}. This is particularly relevant in the setting where the noise level is large.

    \item If each measurement graph is disconnected, then estimating $g^*(k)$'s individually at each $k$ will be meaningless. However, one might still expect, at least intuitively, that a weakly consistent estimate of $g^*$ can be obtained by leveraging smoothness. For instance, if $S_T = 0$, then this would be possible provided the union graph $([n], \cup_{k} \calE_k)$ is connected, where we would simply take the average of all the pairwise information.
\end{enumerate}
%

%--------------------
% Contributions
%--------------------
\paragraph{Main contributions.} We summarize our main contributions as follows.
\begin{itemize}
    \item We propose three algorithms for this problem, based on different relaxation-based approaches.
    \begin{enumerate} 
        \item \emph{Global sphere relaxation} (\globaltrs; Algorithm \ref{algo:TRS_global}): Optimizes a smoothness-penalized data fidelity term with solutions restricted to a sphere in $\matC^{nT}$. This leads to a trust-region subproblem (TRS) which is efficiently solvable.

        \item \emph{Local sphere relaxation with global smoothing} (\localtrs; Algorithm \ref{algo:TRS_local}): Solves sphere-relaxed synchronization problems at each time point $k$ (via a TRS formulation), stacks the local solutions, and projects the stacked vector onto the low-frequency eigenspace of a smoothness operator (motivated by the Laplacian eigenmaps estimator \cite{SadhanalaTV16}).

        \item \emph{Global matrix denoising and local synchronization} (\matdenoising; Algorithm \ref{algo:mat_denoise_proj}): Projects the measurement matrix onto the low-frequency eigenspace of a smoothness operator (matrix denoising), followed by locally solving (at each time $k$) the \rev{sphere}-relaxation of the angular synchronization problem \rev{via a TRS formulation}.
    \end{enumerate}

    \item \rev{As the TRS estimator is hard to analyze theoretically, we analyze a slight modification of Algorithm \ref{algo:mat_denoise_proj}, namely Algorithm \ref{algo:mat_denoise_proj_locspec}, where the TRS estimator for the local synchronization step is replaced by an alternative spectral approach (see also Remark \ref{rmk:spectral_local_sync})}. We provide theoretical guarantees for \rev{Algorithm \ref{algo:mat_denoise_proj_locspec}}, in the form of non-asymptotic high probability bounds for the global $\ell_2$ estimation error, under the dynamic versions of the AGN and Outliers model (see Corollaries \ref{corr:error_wigner_spec} and \ref{cor:error_outliers_spec} respectively). In particular, under mild smoothness assumptions, we prove that \eqref{eq:consistency_intro} holds for the output of \rev{Algorithm \ref{algo:mat_denoise_proj_locspec}}, and give precise non-asymptotic bounds from which the convergence rates can be deduced. 
    
    \begin{enumerate}
        \item  In the case of the Outliers model, our results apply even when the measurement graphs are very sparse -- each being potentially disconnected (see Remarks \ref{rem:outlier_thm_int_case} and \ref{rem:outlier_cor_spec}). 
        
        \item For the AGN model, our results show that \rev{Algorithm \ref{algo:mat_denoise_proj_locspec}} succeeds even when the noise level $\sigma$ is large w.r.t $n$, and is allowed to grow with $T$ as $O(T^{\alpha})$ for $\alpha \in [0,1/4)$ (see Remarks \ref{rem:wigner_thm_int_case} and \ref{rem:wigner_spec_rec}). 
    \end{enumerate}

    \item Finally, we support our theoretical findings with extensive experiments on synthetic data, for the aforementioned AGN and Outliers models. Notably, our experiments demonstrate that \rev{Algorithm \ref{algo:mat_denoise_proj} usually outperforms other competing methods in terms of accuracy}.

\end{itemize}
While we have focused on the $\text{SO}(2)$ group for concreteness, the framework in the present paper can be extended to other compact groups such as $\matZ_2$ and $\text{SO}(d)$.

%----------------
% Related work
%----------------
\subsection{Related work} \label{subsec:rel_work}
\paragraph{Ranking and angular synchronization.} The work of \cite{syncRank} considered the problem of ranking $n$ items given a set of pairwise comparisons on their rank offsets, over a single measurement graph. The \textsc{SyncRank} algorithm, introduced in \cite{syncRank}, formulates the above problem as an instance of the angular synchronization problem, following a group compactification procedure of the real line. 
Concretely, the approach in \cite{syncRank} starts by mapping the noisy rank offsets $C_{ij} \in \set{-(n-1),\dots,n-1}$ to angles $\Theta_{ij} \in [0, 2 \pi  \delta )$  via  the transformation 
$ 	\Theta_{ij} :=  2 \pi  \delta  \frac{ C_{ij}}{n-1} 
	\label{transfToCircle_bis}
$, where $\delta \in [0,1)$. The resulting problem is then treated as an instance of the angular synchronization problem, and solved via the tools therein.  A post-processing step is applied to mod out the best circular permutation, and recover the final rankings. While we do not integrate our algorithmic framework with \textsc{SyncRank} in this paper, it would be interesting to do so as future work, in order to evaluate its performance on dynamic ranking problems.
More recently, \cite{huang2017translation, SVDRank} investigated algorithms for recovering the latent strengths $r^*_i, i=1, \ldots, n$ from noisy versions of $r^*_i - r^*_j$. This is an instance of the group synchronization over the real line, also dubbed as \textit{translation synchronization}. 
The work \cite{huang2017translation} proposed an iterative procedure which involves solving a truncated least-squares problem at each iteration. Conditions were derived under which the iterates were shown to converge to the ground truth. 
In \cite{SVDRank}, an algorithm based on the singular value decomposition was introduced, along with a detailed theoretical analysis.

%-------------------------------------------
% Ranking from dynamic pairwise information
%
\paragraph{Ranking from dynamic pairwise information.}
A recent line of work initiated a systematic study of ranking from dynamic pairwise information, along with statistical guarantees \cite{dynamicBTL_2020, KT_JMLR_BTL_NN, AKT_dynamicRankRSync}. Here, the latent strengths of the items are assumed to evolve smoothly over $T$ time points, with the pairwise data generated independently over time. These works are the main motivation behind the present paper. 
\begin{itemize}
\item The setup in \cite{dynamicBTL_2020, KT_JMLR_BTL_NN} assumed a dynamic Bradley-Terry-Luce (BTL) model where the strength vector satisfies a Lipschitz condition in the (continuous) time domain $[0,1]$. Pointwise estimation error bounds were derived for a given $t \in [0,1]$ via a two stage procedure: the first stage involved suitably ``smoothing'' the pairwise data while the second stage involved applying an existing ranking method from the static setup. In particular, \cite{KT_JMLR_BTL_NN} proposed a nearest-neighbor based averaging procedure in conjunction with Rank Centrality \cite{negahban2015rank},  while \cite{dynamicBTL_2020} proposed a kernel based smoothing procedure in conjunction with the MLE for the static BTL model. The results in both papers imply pointwise consistency w.r.t $T$, with \cite{KT_JMLR_BTL_NN} showing this even when the individual graphs are very sparse (\ER graphs) and potentially disconnected.

\item The work \cite{AKT_dynamicRankRSync} considered the dynamic translation synchronization problem which is related to the BTL model for ranking. Here, the latent strength vector $r^*(k) \in \matR^n$ ($k=1,2,\dots,T$) satisfies a similar smoothness assumption as \eqref{eq:smoothness_def_intro}. Given a measurement graph $\calH_k = ([n], \calE_k)$ at time $k$, we are given noisy information $r_i^*(k) - r^*_j(k) + \epsilon_{ij}(k)$ for each $\set{i,j} \in \calE_k$,  where $\epsilon_{ij}(k)$ is zero-mean subgaussian noise.
Two estimators were proposed, the first of which involves solving an unconstrained smoothness-penalized least squares problem. The second estimator involves first estimating $r^*(k)$ via least-squares minimization, at each $k$, and then projecting the vector of estimated solutions on to the low-frequency space of a smoothness operator. High-probability bounds were obtained for the MSE, establishing their consistency when $T$ is large and the rest of parameters are fixed. In particular, these results imply that the proposed estimators can tolerate larger levels of noise compared to the static case. Our Algorithms \ref{algo:TRS_global} and \ref{algo:TRS_local} are motivated from these two methods, the key difference being that we now need to incorporate the (relaxations of) the constraints, imposed by the group $\text{SO}(2)$. This in particular leads to solving a \rev{TRS-based sphere-relaxation} which is challenging to analyze theoretically. We also remark that the analysis in \cite{AKT_dynamicRankRSync} requires each measurement graph to be connected due to technical reasons. Our analysis for \rev{Algorithm \ref{algo:mat_denoise_proj_locspec}}, however, 
does not suffer from these issues, and applies even when the graphs are very sparse and disconnected. We believe that a similar procedure can be applied for the dynamic translation synchronization problem, leading to similar guarantees as in the present paper (Theorem \ref{thm:error_outliers} and Corollary \ref{cor:error_outliers_spec}). 
\end{itemize}
The work \cite{li2021recovery} considered a generalization of the translation synchronization model where the time-dependent model parameters evolve in a Markovian manner over time. Theoretical results were obtained for estimating a latent pairwise comparison matrix at time $T$. In general, however, existing results for the dynamic ranking problem are essentially empirical works \cite{fahrmeir1994dynamic,glickman1998state,lopez2018often,maystre2019pairwise, cattelan2013dynamic,jabin2015continuous} with limited theoretical guarantees.

%--------------------
% Paper structure
%---------------------
\subsection{Paper structure} 
The remainder of this paper is structured as follows. In Section \ref{sec:setup} we present the setup for dynamic synchronization. In Section \ref{subsec:prob_setup} we introduce the dynamic versions of the AGN and Outliers model, togheter with our global smoothness assumptions. Our proposed estimators and the algorithms to compute them are described in Section \ref{subsec:smooth_estimators}. Section \ref{sec:analysis}, which consists of the theoretical analysis of \rev{Algorithm \ref{algo:mat_denoise_proj_locspec}}, is divided into three parts. In Sections \ref{sec:analysis_spiked_wigner} and \ref{sec:analysis_outliers} we analyse the matrix denoising stage for the AGN and Outliers model respectively. In Section \ref{sec:recovery} we provide an analysis for the local synchronization stage of \rev{Algorithm \ref{algo:mat_denoise_proj_locspec}}, which is valid for both noise models. Section \ref{sec:experiments} contains numerical experiments on synthetic data and Section \ref{sec:conclusion} contains concluding remarks.

\section{Problem setup and algorithms}\label{sec:setup}
In Section \ref{subsec:notation} we describe general notation used throughout the paper. Section \ref{subsec:prob_setup} contains a formal description of the problem, while Section \ref{subsec:smooth_estimators} formally outlines our proposed algorithms. In particular, Section \ref{subsec:smooth_estimators} details additional notation which is needed for describing the algorithms and is also used in the analyses later in Section \ref{sec:analysis}.

% 
% General notation
%
\subsection{Notation} \label{subsec:notation}
For any $n\in \matN$, define $[n]=\{1,\cdots,n\}$. For $A\in\matC^{n\times p}$, $\conj A$ denotes its Hermitian conjugate transpose. We define the operators $\real$ and $\imag$, from $\matC^{n\times p}$ to $\matR^{n\times p}$, which for any matrix give their real and imaginary parts, respectively. Denote $e_1,\dots,e_n \in \matR^n$ to be the canonical vectors. We represent the vector of all ones as $\ones$, with the dimension being clear from the context. 

Let $\|z\|_2=\sqrt{|z_1|^2+\cdots+|z_n|^2}$ denote the Euclidean norm for any $z\in \matC^{n}$. If $Z,Z'\in \matC^{n\times m}$, their Frobenius inner product is defined by $\langle Z,Z'\rangle_F=\Tr (\conj{Z}Z')$, where $\Tr(\cdot)$ denotes the trace of a matrix. The Frobenius norm $\|\cdot\|^2_F$ is the norm induced by this inner product. We denote the Kronecker product of $Z$ and $Z'$ by $Z \otimes Z'$, and their Hadamard product by $Z \circ Z'$. For a matrix $Z\in \matC^{n\times m}$, $\|Z\|_{op}$ is the operator norm of $Z$, which is defined as its largest singular value.

Define, for $n\in \matN$, the set $\calC_n:=\{z\in \matC^{n}: |z_i|=1,\text{ for all }i\in [n]\}$. Recall that the special orthogonal group $\operatorname{SO}(2)$ can be represented by the group of complex numbers $z$, such that $|z|=1$, with the usual complex multiplication. 
The concatenation operator is defined for matrices $X$ in $\matC^{n\times m}$ and $Y\in \matC^{n'\times m}$, denoted as $\concat(X,Y)=\begin{bmatrix}X\\Y\end{bmatrix}\in \matC^{(n+n')\times m}$. This definition extends to any sequence $(X_k)^K_{k=1}$, with $K\in \matN$, and $X_k\in\matC^{n_k\times m}$, through the recurrence $\concat(X_1,\ldots,X_k)=\concat(\concat(X_1,\ldots,X_{k-1}),X_k)$. Similarly, for any natural number $m$, $\blkdiag(X_1,\ldots,X_m)$ denotes the block-diagonal matrix with square matrices $X_1,\ldots,X_m$ on the main diagonal.

Recall that the subgaussian norm of a random variable $X$ is given by 
$$\norm{X}_{\psi_2}:= \sup_{p \geq 1} p^{-1/2} (\expec \abs{X}^p)^{1/p},$$ see e.g., \cite{vershynin_2012} for other equivalent definitions. Note that if $\abs{X} \leq M$ a.s, then $\norm{X}_{\psi_2} \leq M$. When $X$ follows a normal distribution with mean $\mu$ and variance $\sigma^2$, we write $X\sim \calN(\mu,\sigma^2)$. We say $X \sim C\calN(0,1)$ if $X$ is complex-valued and 
\[
\real(X) , \ \imag (X) \stackrel{\text{ind.}}{\sim} 
    \calN(0,1/2). 
\]

We will use the classic asymptotic notation, defined as follows. We say that $f(t)=O(g(t))$ if there exists positive real numbers $t_0,M$, such that for all $t\geq t_0$ we have $f(t)\leq M g(t)$. Similarly, we write $f(t)=o(g(t))$ is for every constant $c>0$, there exists a constant $t_0$, such that, $f(t)\leq c g(t)$, for all $t\geq t_0$. Additional, we say $f(t)=\Omega (g(t))$, if $f(t)=O(g(t))$, and $f(t)=\omega (g(t))$, if $g(t)=o(f(t))$. Symbols used for denoting constants (e.g., $c, C, c_1$ etc.) may change from line to line. 
%

%---------------------
% Problem setup
%
\subsection{Problem setup} \label{subsec:prob_setup}
Consider a set of items $[n]$ and a sequence of time-indexed undirected comparison graphs $\calH_k = ([n],\calE_k)$ for $k = 1,\dots, T$. The set of edges $\calE_k$ determines which items are being compared at time $k$. For each $i\in [n]$, we associate time-dependent latent phases $g^*_i(k)\in\operatorname{SO}(2)$ which will constitute the ground truth signal that we aim to recover from noisy pairwise measurements. More formally, at time $k\in [T]$, we observe for every pair $\{i,j\}\in \calE_k$ a noisy version of  $g^*_i(k)\conj{ g^*_j(k)}$. Notice that even without the addition of noise, the recovery of the ground truth phases $g^*_i(k)$, from the observation of $g^*_i(k)\conj{g^*_j(k)}$, can only be accomplished up to a global phase-shift. That is, given angles $\theta_1,\cdots, \theta_T\in [0,2\pi)$, the quantities $(g^*_i(k)\cdot e^{\iota\theta_k})^T_{k=1}$ define the same relative phases as $(g^*_i(k))^T_{k=1}$. To avoid identifiability issues, we will make the following assumption on the phase of the first element of each block.
\begin{assumption}[Anchoring]\label{assump:anchoring}
    For all $k\in[T]$, we assume that $g^*_1(k)=1$.  
\end{assumption}
For each $k\in[T]$, denote $g^*(k)=\concat(g^*_1(k),\cdots,g^*_n(k)) \in \calC_n$, 
and $g^*$ to be formed by column-wise stacking as 
\begin{equation}\label{eq:ground_truth_definition}
    g^* =\concat(g^*(1),\cdots,g^*(T)) =\begin{pmatrix} g^*(1) \\ \vdots \\ g^*(T)  \end{pmatrix} \in \calC_{nT}.
\end{equation}
We now describe two statistical models for generating the noisy pairwise measurements at each $k$.

% Additive Gaussian noise model
%
\paragraph{Additive Gaussian noise (AGN) model.} 
Denoting for each $k\in [T]$, $G^*(k):=g^*(k)\conj{g^*(k)}$, we are given the measurements $A(k) \in \matC^{n \times n}$ where
\begin{equation} \label{eq:model_Wigner}
A(k) = (G^*(k) - I_n) + \sigma W(k)\enspace \text{ for }k\in [T], 
\end{equation}  
where $W(k)$ is a random complex Hermitian (noise) matrix with independent entries above the diagonal and $W_{ii}(k) = 0$ (for all $i$). In particular, for each $i \neq j$, $W_{ij}(k) \sim C\calN(0,1)$. The parameter $\sigma\geq0$ denotes the noise level. Implicit in \eqref{eq:model_Wigner} is the assumption that the comparison graphs $\calH_k$ are all complete. We remark that in \eqref{eq:model_Wigner}, $A_{ii}(k) = 0$ for each $i$, in contrast to the Spiked Gaussian Wigner model discussed in Section \ref{sec:introduction}. Furthermore, we assume that  $W(1),\dots,W(T)$ are independent.

%-------------------
% Outliers model
%-------------------
% 

\paragraph{Outliers model.}
In this model, we assume that the observation graphs are generated independently from \erdos-R\'enyi distributions, i.e., $\calH_k \stackrel{\text{ind.}}{\sim} ER(n, p(k))$, where $p(k)$ represents the connection probability at time $k$. For each $k\in [T]$ and edge $\{i,j\}\in \calE_k$, we observe the exact relative phase $g^*_i(k)\conj{g^*_j(k)}$, with probability $1 - \eta$, and pure noise (chosen uniformly at random in SO($2$)), with probability $\eta$.  The matrix of observations $A(k)$ can then be written as
\begin{equation*}
     A_{ij}(k) =
    \begin{cases}
        g^*_i(k)\conj{g^*_j(k)} \text{ with probability }\left(1-\eta\right)p(k),\\
        e^{\iota \varphi_{ij}(k)} \enskip \enskip\quad\text{ with probability }\eta p(k),\\
        0 \enskip\qquad\qquad\text{ with probability }(1-p(k)),
    \end{cases}
\end{equation*}
where $\phi_{ij}(k) \stackrel{\text{ind.}}{\sim} \text{Unif}[0,2\pi)$. 
Since there are no observations when $i=j$, we choose to set $(A(k))_{ii}=0$, for all $i\in [n]$. Note that $(A_{ij}(k))_{i \leq j}$ are independent, and $A(k)$ can be expressed as
\begin{equation}\label{eq:model_outlier}
    A(k)= \underbrace{\left(1-\eta\right)p(k)(G^*(k) - I_n)}_{=\expec{[A(k)]}}+R(k),
\end{equation}
where $R(k):=A(k)-\expec{[A(k)]}$. It is easy to see that $(R_{ij}(k))_{i \leq j}$ are independent centered sub-Gaussian random variables.  

%------------------------
% Smoothness
%
\paragraph{Smoothness assumption.} 
As discussed in the introduction, we will assume that the latent phases $g^*_i(k)$ do not change too quickly with time as outlined below.
\begin{assumption}[Global $\ell_2$ smoothness]\label{assump:smooth}
For a given parameter $S_T>0$,
\begin{equation}\label{eq:smoothness}
    \sum_{k=1}^{T-1} \norm{g^*(k) - g^*(k+1)}_2^2 \leq S_T.
\end{equation}
Denoting \rev{$M\in\set{0,1,-1}^{T\times(T-1)}$} to be the \rev{incidence matrix of any directed path graph\footnote{\rev{The rows of $M$ correspond to the $T$ vertices, and the columns correspond to the $T-1$ directed edges. For the column corresponding to edge $i \rightarrow j$, we set its $i$th entry to $+1$, the $j$th entry to $-1$, and the rest of the entries in that column to $0$.}}} on $T$ vertices, this can be written equivalently as
\[ 
\rev{\norm{(M^{\top}\otimes I_n) g^*}_2^2} = \conj{g^*} ((MM^\top) \otimes I_n)g^* \leq S_T,
\]
\rev{where $MM^\top$ corresponds to the Laplacian matrix of the undirected path graph.}
\end{assumption}
The above assumption states that $g^*(k)$ does not change too quickly with $k$ on ``average'', and resembles the notion of quadratic variation of a function. \rev{It can also be viewed as stating that $g^*$ lies ``close'' (depending on $S_T$) to the subspace spanned by the ``smallest few'' eigenvectors of $(MM^{\top}) \otimes I_n$. The closed-form expressions of the eigenvectors (and also eigenvalues) of $MM^{\top}$ are well-known (e.g., \cite{brouwer12}), and can be viewed as the Fourier basis for signals residing on that graph (with the smallest eigenvector corresponding to the lowest frequency signal, and so on).} Since Assumption \ref{assump:anchoring} is in place, we do not need to worry about the invariance of \eqref{eq:smoothness} with respect to constant shifts.  

%--------------------------------------
% Smoothness constrained estimators
%
\subsection{Smoothness constrained estimators} \label{subsec:smooth_estimators}
We now describe algorithms to estimate $g^*$, each of which incorporates Assumption \ref{assump:smooth} in different ways. We first recall that in the static setting (for some fixed time $k\in [T]$), the SO($2$)-synchronization problem can be formulated as solving
\begin{equation}\label{eq:synchro_static}
    \max_{g(k)\in \calC_n}\conj{g(k)}A(k)g(k),
\end{equation}
which, under the AGN model, corresponds to finding the MLE. As a sanity check, one could solve \eqref{eq:synchro_static} for each $k$, and stack the solutions to produce an estimate of $g^*$. However, this suffers from the drawback of treating each time block independently. In light of Assumption \ref{assump:smooth}, there is a temporal smoothness in the observations, which can be exploited to improve the estimation. Before proceeding, it will be useful to introduce some notation. 
\begin{itemize}
    \item We will use the diacritical mark $\sim$ to denote the removal of information relative to the first item of each time block. Thus $\tilde{g}^*(k)$ and $\tilde{g}(k)$, are defined as
\begin{equation*}
  g^{*}(k) = \concat(1,\tilde{g}^*(k))=\begin{bmatrix} 1 \\  \tilde{g}^{*}(k)  \end{bmatrix},  \quad  \mbox{ where } \tilde{g}^{*}(k) \in \calC_{n-1}, \; 
\end{equation*} 
\begin{equation*}
  g(k) = \concat(1,\tilde{g}(k))=\begin{bmatrix} 1 \\  \tilde{g}(k)  \end{bmatrix},  \quad  \mbox{ where } \tilde{g}(k) \in \matC^{n-1}. \; 
\end{equation*} 
Similarly, $\tilde{A}(k)\in \matC^{(n-1)\times (n-1)}$ and $b(k)\in \matC^{n-1}$ are defined as 
\begin{equation*}  
A(k) = \begin{bmatrix}  A_{11}(k) & \conj{b}(k)  \\ b(k)  & \widetilde{A}(k)  \end{bmatrix}. 
\end{equation*}

\item Let us define the set 
\[
\calG:=
\Big\{
\begin{bmatrix} 1 \\  \tilde{g}  
\end{bmatrix}
\begin{bmatrix} 1 &  \conj{\tilde{g}} 
\end{bmatrix}
: \tilde{g}\in \calC_{n-1}\Big\}.
\]
Note that $G^*(k) \in \calG$ for each $k$.

\item Whenever we do not specify the time index, it will mean that we stack all time blocks, as 
\begin{equation}  \label{eq:concat_mats}
\begin{split}
A &= \concat\big(A(1), \cdots A(T)\big)=\begin{bmatrix} A(1) \\ A(2) \\ \vdots \\ A(T)  \end{bmatrix},\; G^* = \concat\big(G^*(1),\cdots,G^*(T)\big)=\begin{bmatrix} G^*(1) \\ G^*(2) \\ \vdots \\ G^*(T)  \end{bmatrix}\;,  \text{ and }  \\
 G &=\concat\big(G(1),\cdots,G(T)\big)= \begin{bmatrix} G(1) \\ G(2) \\ \vdots \\ G(T)  \end{bmatrix}\;\ \in\matC^{nT\times n}.
 \end{split}
\end{equation}

\item For any  $z=(z_1,\cdots,z_m)\in \matC^m$ we define its projection onto $\calC_m$ by 
\begin{equation}\label{eq:proj_so2}
    \calP_{\calC_m} (z)_{i}: = \left\lbrace\begin{array}{cc}
        \frac{z_{i}}{\abs{z_{i}}}& \mbox{if} \; z_{i}\neq 0\\
        1 & \mbox{otherwise.} 
    \end{array}\right. 
\end{equation}

\item Finally, for any $\tau \in [T]$, define 
$$\mathcal{L}_{\tau,n}: = \operatorname{span}\left\{v_{T-k} \otimes e_j\right\}_{k=0, j=1}^{\tau-1, n},$$ 
where $\{v_k\}^T_{k=1}$ are the eigenvectors of $MM^\top$, the Laplacian of the path graph on $T$ vertices, with corresponding eigenvalues $\lambda_1\geq\lambda_2 \cdots > \lambda_T (=0)$. Denote $\calP_{\tau,n}$ to be the projection matrix onto the space $\calL_{\tau,n}$, i.e., 
\begin{equation} \label{eq:proj_smooth_op}
\calP_{\tau,n}=\left(\sum^{\tau-1}_{k=0}v_{T-k}v^\top_{T-k}\right)\otimes I_n.
\end{equation}
\end{itemize}

The following simple observation shows that Assumption \ref{assump:smooth} can be equivalently stated (up to factors depending only on $n$) as a condition on the matrices $G^*(k)$. Note that, for all $k\in[T]$,
\begin{equation*}
     \|\tilde{g}^*(k)-\tilde{g}^*(k+1)\|^2_2=\|g^*(k)-g^*(k+1)\|^2_2,
\end{equation*}
and 
\begin{equation*}
    \|G^*(k)-G^*(k+1)\|^2_F=2\|\tilgt(k)-\tilgt(k+1)\|^2_2+\|\tilgt(k)\conj{\tilgt(k)}-\tilgt(k+1)\conj{\tilgt(k+1)}\|^2_F,
\end{equation*}
hence 
\begin{equation}\label{eq:matrix_smooth_equiv}
     2\|\gt(k)-\gt(k+1)\|^2_2\leq \|G^*(k)-G^*(k+1)\|^2_F\leq (2+4n)\|\gt(k)-\gt(k+1)\|^2_2.
\end{equation}
This implies, in particular, that the matrices $G^*(k)$ do not evolve too quickly, as given by 
\begin{equation}\label{eq:matrix_smooth}
     \sum^{T-1}_{k=1}\|G^*(k)-G^*(k+1)\|^2_F\leq (2+4n)S_T.
\end{equation}

%---------------------------
% Global sphere relaxation
%
\subsubsection{Global sphere relaxation}
\paragraph{Denoising optimization problem.} Since \eqref{eq:matrix_smooth} holds, under Assumption \ref{assump:smooth}, it is natural to consider the following denoising optimization problem, which promotes smoothness via penalization.
\begin{align}\label{eq:opt_denoising}
    \mathop{\min} \limits_{\substack{G=\concat(G(1),\cdots,G(T))\\ \text{ s.t }\forall k: G(k)\in \calG}} \quad  \underbrace{\sum_{k=1}^T\|A(k)-G(k)\|_F^2}_{\mbox{data fidelity term}} +  \underbrace{ \lambda \sum_{k=1}^{T-1}\|G(k)-G(k+1)\|_F^2}_{\mbox{smoothness term}},
\end{align}
where $\lambda\geq0$ is the parameter that controls the amount of regularization in \eqref{eq:opt_denoising}. On the other hand, for $G(k)\in \cal \calG$, we have
\begin{equation}\label{eq:matrix_tovector}
    \|A(k)-G(k)\|^2_F=\|A(k)\|^2_F+\underbrace{\|G(k)\|^2_F}_{=n^2}-\conj{\tilde{g}(k)}\tilde{A}(k) \tilde{g}-2\real{\big(\conj{b}(k)\tilde{g}(k)\big)},
\end{equation}
which will help us write \eqref{eq:opt_denoising} more compactly (with variables $\tilde{g}(k)$). Specifically, using \eqref{eq:matrix_tovector} and \eqref{eq:matrix_smooth_equiv}, the following problem is equivalent to \eqref{eq:opt_denoising}
\begin{equation}
    \label{eq:opt_denoising_vector}
    \mathop{\max} \limits_{\substack{\tilde{g}=\concat(\tilde{g}(1),\cdots,\tilde{g}(T))\\ \text{ s.t }\forall k: \tilde{g}(k)\in \calC_{n-1}}} \quad  \underbrace{\conj{\tilde{g}}\tilde{A}_{block} \tilde{g}+2\real{(\conj{b}\tilde{g})}}_{\mbox{data fidelity term}} - \underbrace{\lambda \conj{\tilde{g}}(MM^\top\otimes I_{n-1})\tilde{g}}_{\mbox{smoothness term}},
\end{equation}
where $\tilde{A}_{block}:=\blkdiag(\tilde{A}(1),\cdots,\tilde{A}(T))$.
\begin{remark}[Static anchored synchronization]
In the case $T=1$, the smoothness penalty term will be absent from \eqref{eq:opt_denoising}, and we recover the usual anchored synchronization MLE for the AGN model 
\begin{equation} \label{eq:opt_static}
\min _{G \in \mathcal{G}} \|A-G\|_F^2  \equiv \max _{\tilde{g} \in \calC_{n-1}}\, \conj{\tilde{g}}\tilde{A} \tilde{g}+2\real{(\conj{b}\tilde{g})}.
\end{equation} 
\end{remark}

\begin{remark}[Complexity of \eqref{eq:opt_denoising}]\label{rmk:complexity_denosing}
As mentioned in \cite[Sec.2]{bandeira2017tightness}, the problem \eqref{eq:opt_static} is an instance of \emph{complex quadratic programming}, which is known to be NP-Hard \cite[Prop.3.5]{qp_nphard2006}. Without the penalty term, problem \eqref{eq:opt_denoising} amounts to solve $T$ independent instances of \eqref{eq:opt_static}, clearly remaining NP-Hard. However, it is unclear whether the hardness holds for all $\lambda>0$. In particular, if one allows the value $\lambda=\infty$, the problem becomes trivial as any perfectly smooth $G$ (that is $G(k)=G(1)$ for all $k\in[T]$) will be a solution. Nevertheless, the problem remains highly non-convex even for large values of $\lambda$, which makes the existence of an efficient algorithm for solving \eqref{eq:opt_denoising} (for an arbitrary $A$) unlikely.
\end{remark}
In light of Remark \ref{rmk:complexity_denosing}, we consider the following relaxation of \eqref{eq:opt_denoising_vector}
\begin{equation}\label{eq:opt_sphere_relax}
    \max_{\substack{\tilde{g}=\concat(\tilde{g}(1),\cdots,\tilde{g}(T)) \\\text{ s.t  }  \|\tilde{g}\|^2_2=(n-1)T}} \conj{\tilde{g}}\big(\tilde{A}_{block}-\lambda(MM^\top\otimes I_{n-1})\big)\tilde{g}+2\real{(\conj{b}\tilde{g})}.
\end{equation}
Problem \eqref{eq:opt_sphere_relax} is an instance of the trust region subproblem (TRS) with equality constraints, which consists of minimizing a general quadratic function (non-necessarily convex), under a sphere constraint. Although it will not be used in the sequel, it is worth mentioning that a version of TRS with $\ell_2$ ball constraints exists (and it is, in some respect, more standard). There is a vast literature on the theory of TRS, including the characterization of its solutions \cite{TRS_Hager}, and several efficient algorithms that solve it, in both the spherical and ball constraints \cite{TRS_nakatsukasa}. %Having said that, since $G$ has size $O(n^2T)$, solving \eqref{eq:opt_sphere_relax} will likely be slow even for moderately large values of $n,T$. 
The complete procedure is outlined in Algorithm \ref{algo:TRS_global}.
\begin{algorithm}
    \caption{Global TRS for Dynamic Synchronization (\globaltrs)}\label{algo:TRS_global}
    \begin{algorithmic}[1]
    \State {\bf Input:} Matrix $A\in\matC^{nT\times n}$ of observations, smoothness parameter $\lambda>0$.
    \State Form the matrix $\tilde{A}_{block}=\blkdiag\big(\tilde{A}(1),\cdots,\tilde{A}(T)\big)$.
    \State Obtain $\est{\tilde{g}}=\concat(\est{\tilde{g}}(1),\cdots,\est{\tilde{g}}(T))\in \matC^{(n-1)T}\leftarrow$ a solution of $\eqref{eq:opt_sphere_relax}$
     \For{$k=1\dots ,T$}
     \State Define $\est g(k)=\begin{bmatrix}
         1\\\calP_{\calC_{n-1}}\est{\tilde{g}}(k)
     \end{bmatrix}\in \calC_n$.
%     \State Redefine $\est g(k)\leftarrow \calP_{\calC_n}\big(\est g(k)\big)$, where $\calP_{\calC_n}$ is the projection onto $\calC_n$.
     \EndFor
    \State {\bf Output:} Estimator $\est g\in \calC_{nT}$.
    \end{algorithmic}
\end{algorithm}
\subsubsection{Local sphere relaxation}
Instead of the ``global relaxation'' in \eqref{eq:opt_sphere_relax}, we now consider a different methodology which entails, as a first step, solving the following sphere relaxation of \eqref{eq:opt_static}  
\begin{equation}\label{eq:opt_sphere_relax_loc}
    \est{\tilde{g}}(k):=\mathop{\arg \max}\limits_{\substack{\tilde{g}(k)\in \matC^{n-1}\\ \|\tilde{g}(k)\|_2^2=n-1}}\conj{\tilde{g}(k)}\tilde{A}(k)\tilde{g}(k)+2\real\big(\conj{b(k)}\tilde{g}(k)\big).
\end{equation}
This leads to a local estimate $\est{\tilde{g}}(k)$ for each $k\in [T]$. Stacking them, we obtain
\[\etg:=\concat\big(\, \est{\tilde{g}}(1)
,\cdots,  \est{\tilde{g}}(T)\big)
,\] 
which satisfies the sphere constraint for each time block. The smoothness constraint has not been used yet, so as a second step, we consider an additional projection step onto the space of ``low frequency'' eigenvectors of $(MM^\top) \otimes I_{n-1}$. To achieve this, our proposed algorithm takes a regularization parameter $\tau\in [T]$ as input, and forms the projection matrix $\calP_{\tau,n-1}$, which projects onto $\calL_{\tau,n-1}$. It is worth noting that when $\tau$ is small, this projection step has a ``smoothing'' effect. For instance, for $\tau=1$, the resulting projected-vector has equal block components. After the projection step onto $\calL_{\tau,n-1}$, we finally need to use the projection $\calP_{\calC_{(n-1)T}}$ to ensure that the estimator lies in $\calC_{(n-1)T}$. The complete procedure is outlined in Algorithm \ref{algo:TRS_local}. 
\begin{algorithm}
    \caption{Local TRS $+$ Global Smoothing for Dynamic Synchronization (\localtrs)}\label{algo:TRS_local}
    \begin{algorithmic}[1]
    \State {\bf Input:} Matrix $A\in\matC^{nT\times n}$ of observations, smoothness parameter $\tau\in \matN$.
    \For{$k=1,\dots ,T$}
     \State Obtain $\est{\tilde{g}}(k)\leftarrow$ solution of \eqref{eq:opt_sphere_relax_loc}.
     \EndFor
     \State Form $\etg=\concat\big(\etg(1),\cdots,\etg(T)\big)\in\matC^{(n-1)T\times 1}$
    \State Define $\tilde{h}=\calP_{\tau,n-1}\etg$ where $\calP_{\tau,n-1}\in \matC^{(n-1)T\times (n-1)T}$ is the projection onto $\calL_{\tau,n-1}$ as in \eqref{eq:proj_smooth_op}.
    
     \For{$k=1,\dots ,T$}
     \State Define $\est g(k)=\begin{bmatrix}
         1\\\calP_{\calC_{n-1}}\tilde{h}(k)
     \end{bmatrix}$, where $\tilde{h}(k)$ is such that $\tilde h=\concat(\tilde{h}(1),\cdots, \tilde{h}(T))$.
     \EndFor
    \State {\bf Output:} Estimator $\est g\in\calC_{nT}$.
    \end{algorithmic}
\end{algorithm}
\subsubsection{Denoising measurement matrices}\label{sec:matrix_denoising}
An alternative procedure involves denoising the measurement matrices $A$ (recall \eqref{eq:concat_mats}) by projecting  its columns onto the low-frequency eigenspace of  $(MM^\top) \otimes I_n$, in the spirit of the Laplacian eigenmaps estimator (see e.g.,  \cite{SadhanalaTV16}). Our proposed method consists of the following two steps. 
\begin{itemize}
    \item \textbf{Matrix denoising stage.} We project the columns of $A$ onto the low frequency space $\calL_{\tau,n}$, leading to the intermediate estimator $\est G=\calP_{\tau,n} A$.
    
    \item \textbf{Local synchronization stage.} \rev{For each $k\in [T]$ we first ``Hermitianize'' the block  $\est G(k)$ to obtain $\est G'(k)$, and then solve  a ``local'' synchronization problem, as in \eqref{eq:opt_static}, but with input data \rev{$\est G'(k)$}.} In Algorithm \ref{algo:mat_denoise_proj}, we solve its TRS relaxation as in \eqref{eq:opt_sphere_relax_loc}.
\end{itemize}
The matrix denoising stage can be regarded as a ``global'' denoising step, where the temporal smoothness captured by $(MM^\top) \otimes I_n$ is used to 
produce an estimation of $G^*\in \matC^{nT\times nT}$. Notice that after this step, the obtained $\est G$ does not necessarily lie in $\calG$. This constraint is enforced in the local anchored synchronization step, where we first solve a relaxation such as \eqref{eq:opt_sphere_relax_loc} and then project the solution onto $\calC_{nT}$. The complete procedure is outlined in Algorithm \ref{algo:mat_denoise_proj}.
\begin{remark}[Spectral local synchronization]\label{rmk:spectral_local_sync}
It is possible to replace the TRS-based approach with a spectral method to solve the local synchronization problem in Algorithm \ref{algo:mat_denoise_proj}. More specifically, lines 5 and 6 in Algorithm \ref{algo:mat_denoise_proj} are replaced by the following lines
{\normalfont{
\begin{itemize}
    \item[{\footnotesize 5:}] Find $\est{g}'(k)\leftarrow$ the eigenvector associated to the largest eigenvalue of \rev{$\est G'(k)$}.
    \item[{\footnotesize 6:}] Define $\est{g}(k)= \calP_{\calC_{n}}(\est{g}'(k))e^{-\iota \teonek}$, where $\teonek$ is defined by $\est{g}'_1(k)=|\est{g}'_1(k)|e^{\iota \teonek}$ \rev{(with $\teonek = 0$ if $\est{g}'_1(k) = 0$)}.
\end{itemize}
}}
While the spectral method may sacrifice some accuracy compared to the TRS-based approach, it offers simplicity in analysis. Our theoretical analysis in Section \ref{sec:analysis} is based on the spectral approach (\rev{outlined separately as Algorithm \ref{algo:mat_denoise_proj_locspec} for clarity}), whereas in our experiments, we opt for the TRS-based method.
\end{remark}
\begin{remark}[Denoising smooth signals on a graph] \label{rem:connec_den_smooth_sig_graph}
 \rev{As alluded to earlier, the matrix denoising step in Algorithm \ref{algo:mat_denoise_proj} is motivated by the Laplacian eigenmaps estimator (see e.g.,  \cite{SadhanalaTV16}) for denoising smooth signals on a graph. To make the connection explicit using our notation, we recall that in this latter problem, we are given an undirected graph $\calH = ([T], \calE)$, along with noisy measurements $y = x^* + \eta$ of a signal $x^* \in \matR^T$ with $\eta$ denoting the zero-mean noise. Denoting $L$ to be the Laplacian matrix of $\calH$, here $x^*$ is assumed to be smooth w.r.t $\calH$ in the sense that
\begin{equation} \label{eq:smoothness_den_sig_graph}
    (x^*)^\top L x^* = \sum_{\set{i,j} \in \calE} (x^*_i - x^*_j)^2 \leq S_T.
\end{equation}
 The Laplacian eigenmaps estimator then estimates $x^*$ by projecting $y$ on to the space spanned by the $\tau$ smallest eigenvectors of $L$. }

 \rev{In our setting, the aforementioned graph $\calH$ is a path graph\footnote{\rev{since we consider its vertices to denote time instants, although the analysis can be extended to general connected graphs as well.}} and each vertex $k \in [T]$ has a latent signal $g^*(k) \in \calC_n$. We observe at each $k$ a noisy matrix $A(k)$ which is also in the form of a signal (scalar times $G^*(k) - I_n$) plus zero-mean noise term; recall \eqref{eq:model_Wigner} and \eqref{eq:model_outlier}. Moreover, $g^*(k)$'s satisfy a smoothness condition (recall \eqref{eq:smoothness}) which, as discussed earlier,  translates to a smoothness condition on $G^*(k)$'s (recall \eqref{eq:matrix_smooth}) that is analogous to \eqref{eq:smoothness_den_sig_graph}.}

 \rev{As we will see in Section \ref{sec:analysis}, the outline of the analysis is -- at a top-level -- similar to that in \cite{SadhanalaTV16} in the sense of bounding the bias and variance error terms, and finding the optimal $\tau$. The difference lies in the technical details: we observe matrix-valued information (with complex-valued entries) at each node of $\calH$ which leads to more cumbersome calculations -- especially for the Outliers model -- in deriving high probability error bounds.}
\end{remark}

%---------------------------------------------------
% Difference with second estimator in Araya et al.
%---------------------------------------------------
\begin{remark}[Comparison with \cite{AKT_dynamicRankRSync}]
    \rev{As discussed in Section \ref{subsec:rel_work}, the work of \cite{AKT_dynamicRankRSync} considered an estimator for the dynamic translation synchronization problem, which also relied on the idea of projecting on to the low-frequency subspace. In that work, the estimator first solved the least-squares problem locally at each time-point, and then projected the stacked vector of individual-solutions onto the low-frequency space. This required each individual graph to be connected for the analysis to work. In Algorithm \ref{algo:mat_denoise_proj}, however, we first project the (stacked) input data matrices onto the low-frequency space, and then solve the local synchronization problem. This allows the individual graphs to be disconnected in the Outliers model -- something we will prove in Section \ref{sec:analysis} for a modified version of Algorithm \ref{algo:mat_denoise_proj}, namely Algorithm \ref{algo:mat_denoise_proj_locspec}. }
\end{remark}

%------------------------------------------------
% Global matrix denoising with local TRS based synchronization
%------------------------------------------------
\begin{algorithm}
    \caption{Global Matrix Denoising $+$ Local TRS for Dynamic Synchronization (\matdenoising)}\label{algo:mat_denoise_proj}
    \begin{algorithmic}[1]
    \State {\bf Input:} Matrix $A\in\matC^{nT\times n}$ of observations, parameter $\tau\in\matN$.
    \State Define $\est G=\calP_{\tau,n} A$, where $\calP_{\tau,n}\in \matC^{nT\times nT}$ is the projection onto $\calL_{\tau,n}$.
     \For{$k=1\dots ,T$}
      \State Hermitianize $\rev{\est G'(k) :=} (\est G(k)+\conj{\est G(k)})/2$.
     \State Find $\etg(k)\leftarrow$ solution of  \eqref{eq:opt_sphere_relax_loc} with data \rev{$\est G'(k)$}.
     \State Form $\est g(k)=\begin{bmatrix}
         1\\\calP_{\calC_{n-1}}\etg(k)
     \end{bmatrix}\in \calC_{n}$.
     \EndFor
    \State {\bf Output:} Estimator $\est g=\concat\big(\est g(1),\cdots,\est g(T)\big)\in \calC_{nT}$.
    \end{algorithmic}
\end{algorithm}

%--------------------- 
% Analysis 
%
\section{Analysis}\label{sec:analysis}
In this section, we present theoretical guarantees for recovering \rev{$g^*(1),\dots,g^*(T)$} under the AGN and Outliers models. Sections \ref{sec:analysis_spiked_wigner} and \ref{sec:analysis_outliers} contain results specifically addressing the \emph{matrix denoising} stage of Algorithm \ref{algo:mat_denoise_proj} (as elaborated in Section \ref{sec:matrix_denoising}). \rev{The main results therein are Theorem \ref{thm:error_wigner} for the AGN model, and Theorem \ref{thm:error_outliers} for the Outliers model.} Then, Section \ref{sec:recovery} uses these estimates to provide guarantees for recovering $g^*(1),\dots,g^*(T)$ -- it provides an analysis of the \emph{local synchronization} stage by employing the spectral approach outlined in Remark \ref{rmk:spectral_local_sync}. \rev{The main results therein are Corollary \ref{corr:error_wigner_spec} for the AGN model, and Corollary \ref{cor:error_outliers_spec} for the Outliers model.}
\subsection{Matrix denoising under AGN model}\label{sec:analysis_spiked_wigner}
Given \eqref{eq:model_Wigner}, we can write (c.f. the notation introduced in Section \ref{subsec:smooth_estimators})
\begin{equation*}
    A = \underbrace{(G^* - \ones_T \otimes I_n)}_{=: \Gshift^*} +\sigma W,
\end{equation*}
where $A, G^*$ are as in \eqref{eq:concat_mats}, and $W=\concat(W(1),\ldots,W(T))$ stacks the (Hermitian) noise matrices. We will focus on bounding the distance between the intermediate estimator $\est{G} = \calP_{\tau,n} A$ and the (shifted) ground truth $\Gshift^*$. To this end, note that 
\begin{equation*}
    \est{G}=\calP_{\tau,n} A= \calP_{\tau,n} \Gshift^* +\sigma \calP_{\tau,n} W,
\end{equation*}
from which we deduce  
\begin{equation}\label{eq:bias_variance_proj}
    \|\est{G}-\Gshift^*\|^2_F=\|\projperp \Gshift^*\|^2_F+\sigma^2\|\projtaun W\|^2_F,
\end{equation}
where $\projperp:=I_{nT}-\projtaun$. The first term on the right-hand side (RHS) of \eqref{eq:bias_variance_proj} represents the bias which depends on the smoothness level. 
Clearly $\Gshift^*$ has the same quadratic variation as $G^*$ and hence satisfies \eqref{eq:matrix_smooth} as well.
Thus, we expect this term to decrease as $\tau$ increases. The second term therein corresponds to the variance introduced by noise, which we expect to increase with $\tau$. Hence, selecting the optimal value for $\tau$ will achieve the right bias-variance trade-off.
The following lemma gives a control over the bias term. 
\begin{lemma}[Bias bound]\label{lem:bias_term}
    Let $G^*\in \matC^{nT\times n}$ be a matrix satisfying \eqref{eq:matrix_smooth}, for some $S_T\geq0$. Then, it holds for all $1 \leq \tau \leq T$
    \begin{equation}\label{eq:bias_control}
        \|\projperp \Gshift^*\|^2_F\leq \frac{20n}{\pi^2}\frac{T^2S_T}{\tau^2}\ones_{\tau<T}.
    \end{equation}
\end{lemma}
\begin{proof}
When $\tau = T$, we have $\projperp=0$, from which the final bound follows. So assume now that $1 \leq \tau \leq T-1$. 
Denote the $j$-th column of $G^*$ by $G^*_{:,j}$. Since 
    \begin{equation*}
        \|\projperp G^*\|^2_F=\sum^n_{j=1}\|\projperp G^*_{:,j}\|^2_2,
    \end{equation*}
    we have for each $j\in[n]$ that 
    \begin{equation*}
         \|\projperp \Gshift^*_{:,j}\|^2_F=\sum^n_{i=1}\sum^{T-1}_{k=\tau}\langle v_{T-k}\otimes e_i, \Gshift^*_{:,j}\rangle^2. 
    \end{equation*}
     On the other hand, it is easy to see that \eqref{eq:matrix_smooth} can be rewritten as
    \begin{equation*}
        \sum^n_{j=1}\sum^n_{i=1}\sum^{T-1}_{k=0}\lambda_{T-k}\langle v_{T-k}\otimes e_i, \Gshift^*_{:,j}\rangle^2\leq (2+4n)S_T,
    \end{equation*}
    from which we deduce (since $n \geq 2$) 
    \begin{equation*}
        \lambda_{T-\tau}\sum^n_{j=1}\sum^n_{i=1}\sum^{T-1}_{k=\tau}\langle v_{T-k}\otimes e_i, \Gshift^*_{:,j}\rangle^2\leq 5nS_T.
    \end{equation*}
    This implies (notice that $\lambda_{T-\tau}\neq 0$ since $1 \leq \tau \leq T-1$) 
    \begin{equation*}
        \sum^n_{j=1}\|\projperp \Gshift^*_{:,j} \|_2^2\leq \frac{5nS_T}{\lambda_{T-\tau}}.
    \end{equation*}
    By the known expression for the eigenvalues of the Laplacian of the path graph \cite{brouwer12}, we have $\lambda_{T-\tau}=4\sin^2{(\frac{\tau\pi}{2T})}$, and using that $\sin(x)\geq \frac x2$, for $x\in[0,\pi/2]$, we obtain 
    \begin{equation*}
        \sum^n_{j=1}\|\projperp \Gshift^*_{:,j} \|_2^2\leq \frac{20n}{\pi^2}\frac{T^2S_T}{\tau^2}.
    \end{equation*}

\end{proof}
We now control the variance term. 
\begin{lemma}[Variance bound]\label{lem:variance_term}
    Let $W(1),\dots,W(T)$ be independent, random Hermitian matrices as in the AGN model in Section \ref{subsec:prob_setup}. Then there exists a constant $C > 0$, such that for all $\tau\in[T]$ and $\delta\in(0,1)$, it holds with probability larger than $1-\delta$ that
    \begin{equation*}
        \|\projtaun W\|^2_F\leq 
        C\left(n^2\tau +\log\Big({\frac4\delta}\Big)\right).   
    \end{equation*}

\end{lemma}
\begin{proof}
    To simplify the analysis, we split $W$ as follows
\begin{equation*}
    W = \frac1{\sqrt2}(W_1+W_2),
\end{equation*}
where $W_1,W_2\in \matC^{nT\times n}$ satisfy $W_1(k)=\conj{W_2(k)}$, for all $k\in [T]$, and the entries of $W_1(k)$ are i.i.d random variables from $C\calN(0,1)$, for all $k\in[T]$. Furthermore, we consider the decomposition 
\begin{equation*}
    W_s= \underbrace{\real(W_s)}_{=:W_{s,R}}+\iota\underbrace{\imag(W_s)}_{=:W_{s,I}}, 
\end{equation*}
for $s=1,2$. It will be enough to bound $\|\projtaun W_1\|_F$, since the same bound will hold for $\|\projtaun W_2\|_F$,  analogously. Denoting $(W_{1,R})_{:,j}$ (resp. $(W_{1,I})_{:,i}$) to be the $j$-th column of $W_{1,R}$(resp. $W_{1,I}$), we obtain 
\begin{align*}
    \|\projtaun W_1\|^2_F&= \|\projtaun W_{1,R}\|^2_F+\|\projtaun W_{1,I}\|^2_F \\
    &= \sum^n_{j=1}\|\projtaun (W_{1,R})_{:,j}\|^2_2 +\sum^n_{j=1}\|\projtaun (W_{1,I})_{:,j}\|^2_2 \\
    &= \underbrace{\conj{W'_1}(I_{2n}\otimes\projtaun)W'_1}_{=:\calV_1},
\end{align*}
where $W'_1\in\matR^{2n^2T\times 1}$ is defined as 
\[
W'_1:=\begin{bmatrix}
    (W_{1,R})_{:,1}\\
    \vdots\\
    (W_{1,R})_{:,n}\\
    (W_{1,I})_{:,1}\\
    \vdots\\
    (W_{1,I})_{:,n}
\end{bmatrix}.
\]
Observe that $W'_1$ has independent entries, with each entry either $0$ or drawn from $\calN(0,1/2)$. To bound $\calV_1$, we will use the Hanson-Wright inequality \cite[Thm. 1.1]{rudelson_vershynin}. It is easy to see that
\begin{align*}
    \|I_{2n}\otimes \projtaun\|_{op}=1, \quad 
    \|I_{2n}\otimes \projtaun\|^2_F=2n^2\tau.
\end{align*}
Thus, from Hanson-Wright inequality, we obtain, for all $t\geq 0$
\begin{equation*}
    \prob(|\calV_1-\expec[\calV_1]|>t)\leq 2\exp\left(-c_1\left(\frac{t^2}{n^2\tau}\wedge t\right)\right).
\end{equation*}
On the other hand, notice that $\expec[\calV_1]\leq c_2n^2\tau$. We deduce that, for any $\delta>0$, with probability larger than $1-\delta$, 
\begin{align*}
    \calV_1
    \leq C\left(n^2\tau+\sqrt{n^2\tau\log\paren{\frac2\delta}}+\log\paren{\frac2\delta}\right) 
    \leq 2 C \left(n^2\tau+\log\paren{\frac2\delta}\right)
\end{align*}
for a constant $C>0$. To pass from the first to the second inequality, we used the fact $a^2+ab+b^2\leq 2(a^2+b^2)$ for any $a,b\in\matR$. Arguing in an analogous way, we obtain that 
\begin{equation*}
    \|\projtaun W_2\|^2_F\leq
    2 C \left(n^2\tau+\log\paren{\frac2\delta}\right), 
    %C\left(n^2\tau+\sqrt{n^2\tau\log{\frac2\delta}}+\log{\frac2\delta}\right)\,
\end{equation*}
with probability greater than $1-\delta$. Since $\|\projtaun W\|^2_F\leq \|\projtaun W_1\|^2_F+\|\projtaun W_2\|^2_F$, the result follows via the union bound.
\end{proof}
\paragraph{Optimal choice of $\tau$.} Combining \eqref{eq:bias_variance_proj} with Lemmas \ref{lem:bias_term} and \ref{lem:variance_term} we obtain the following bound (with probability larger than $1-\delta$) 
\begin{equation}\label{eq:error_proj_wigner}
    \|\est{G}-\Gshift^*\|^2_F\leq \frac{20n}{\pi^2}\frac{T^2 S_T}{\tau^2}\ones_{\tau<T}+ C\sigma^2\left(n^2\tau+\log\paren{\frac4\delta}\right).
\end{equation}
It remains to choose the optimal value of $\tau \in [T]$ which minimizes the RHS of \eqref{eq:error_proj_wigner}, up to constant factors. To this end, using $\indic_{\tau < T} \leq 1$, note that the global minimizer of the ensuing bound is given by $(\frac{T^2 S_T}{n \sigma^2})^{1/3}$. But this quantity does not necessarily lie in $[T]$. Hence we define  
\begin{equation} \label{eq:opt_tau_agn}
\tau^* := \min\set{1 + \floor{\Big(\frac{T^2 S_T}{n \sigma^2}\Big)^{1/3}}, T} \in [T].
\end{equation}
Clearly $\tau^*$ is not the global minimizer of the bound in \eqref{eq:error_proj_wigner} but it will lead to an error bound which achieves the right scaling in terms of $T$ and $S_T$. 
This is shown in the following theorem which is the first main result of this section. Its proof is deferred to Appendix \ref{app:proof_thm_wigner}.
\begin{theorem}[\rev{Main result: $\Gshift^*$ recovery under AGN model}]\label{thm:error_wigner}
    Let $A$ be the stacked measurement matrix for the AGN model in \eqref{eq:model_Wigner}, with a ground truth $g^*$ satisfying Assumption \ref{assump:smooth}. Then choosing $\tau = \tau^*$ with $\tau^*$ as in \eqref{eq:opt_tau_agn}, the following holds for the estimator $\est G=\projtaun A$. There exists a constant $C>0$ such that for any $\delta\in (0,1)$, it holds with probability at least $1-\delta$ that 
    \begin{align}\label{eq:error_bound_Wigner}
        \|\est G-\Gshift^*\|^2_F
        &\leq  C\Big(\max\set{\frac{nT^2 S_T}{\paren{1 + (\frac{T^2 S_T}{n \sigma^2}\Big)^{2/3}}}, n S_T} \indic_{\set{\tau^* < T}}   
        + \sigma^2 n^2 \min\set{1 + \Big(\frac{T^2 S_T}{n \sigma^2}\Big)^{1/3}, T}  
        \nonumber \\
        &+ \sigma^2 \log\paren{\frac4\delta} \Big). 
    \end{align}
    Furthermore, there exist constants $C_1, C_2 >0$ such that if  
    \begin{equation} \label{eq:T_cond_wigner}
        T \geq C_1 \max\set{\paren{\frac{n\sigma^2}{S_T}}^{1/2}, \paren{\frac{S_T}{n\sigma^2}}, \sigma^4 S_T n^5}
    \end{equation}
    then \eqref{eq:error_bound_Wigner} implies
    \begin{equation} \label{eq:error_bound_wigner_simp}
      \|\est G-\Gshift^*\|^2_F
        \leq  C_2 \paren{\sigma^{4/3} T^{2/3} S_T^{1/3} n^{5/3} + \sigma^2 \log\paren{\frac{4}{\delta}}}.
    \end{equation}
\end{theorem}
The following remarks are in order for interpreting Theorem \ref{thm:error_wigner}.
\begin{remark} \label{rem:wigner_thm_perf_noise_smooth}
 When $\sigma = 0$ then we obtain $\tau^* = T$ and $\|\est G-\Gshift^*\|_F = 0$. This is also seen directly from \eqref{eq:error_proj_wigner}.
 If $S_T = 0$, we obtain $\tau^* = 1$ and $\|\est G-\Gshift^*\|^2_F \lesssim \sigma^2 n^2 + \sigma^2 \log(4/\delta)$. This can also be obtained directly from \eqref{eq:error_proj_wigner}. Note that  $\frac{1}{\sqrt{T}} \|\est G-\Gshift^*\|_F = O(1/\sqrt{T})$ which is the parametric rate.
\end{remark}

\begin{remark} \label{rem:wigner_thm_int_case}   
 The first part of Theorem \ref{thm:error_wigner} holds for all $\sigma, S_T \geq 0$. The interesting situation, however, is when $\sigma, S_T > 0$ -- especially when both these quantities are sufficiently bounded away from zero. In this case, the second part of Theorem \ref{thm:error_wigner} is meaningful. For instance, suppose $\sigma \asymp T^{\alpha}$ for some $\alpha \in [0,1/4)$. Then provided 
\begin{equation} \label{eq:wig_rem_int_smooth_conds}
S_T \gtrsim \frac{n}{T^{2(1-\alpha)}} \ \text{ and } \ S_T = o(T^{1-4\alpha})
\end{equation}
it is easy to see that \eqref{eq:T_cond_wigner} will be satisfied for large enough $T$. 
Moreover, \eqref{eq:error_bound_wigner_simp} then implies 
$$\frac{1}{T} \|\est G-\Gshift^*\|^2_F \stackrel{T \rightarrow \infty}{\longrightarrow} 0.
$$ 
Finally, we remark that the $O(T^{2/3} S_T^{1/3})$ term in \eqref{eq:error_bound_wigner_simp} matches that obtained in \cite[Theorem B]{AKT_dynamicRankRSync} and also matches the rate for denoising smooth signals on a path graph \cite[Theorem 6]{SadhanalaTV16}.
\end{remark}
%
%
%-----------------------------------------------
% Matrix denoising under  the Outliers model
%-----------------------------------------------
\subsection{Matrix denoising under  the Outliers model}\label{sec:analysis_outliers}
We now focus on the outliers model, defined in \eqref{eq:model_outlier}, which can be written in stacked matrix notation as 
\begin{equation}\label{eq:model_outlier2}
    A = D(G^* - \revpostpub{\ones_T \otimes I_n}) + R = D \Gshift^* + R,
\end{equation}
where $D=\big(1-\eta\big)\blkdiag{\left(p(1)I_n,\ldots,p(T)I_n\right)}$ and $R=\concat\big(R(1),\ldots, R(T)\big)$. Note that, for each $k\in[T]$, $R(k)$ is a Hermitian matrix with $(R_{ij}(k))_{i \leq j}$ being independent centered sub-Gaussian random variables, and in particular, $R_{ii}(k) = 0$ for each $i \in [n]$. 
 
Our goal is to bound the error $\norm{\est{G} - D \Gshift^*}_F^2$. Since $\est{G} = \projtaun A = \projtaun(D\Gshift) + \projtaun R$, we obtain the bound
\begin{align} \label{eq:bias_variance_proj_outliers}
   \norm{\est{G} - D \Gshift^*}_F^2 \leq 2\underbrace{\norm{\projperp (D\Gshift^*)}_F^2}_{\text{Bias}} + 2\underbrace{\norm{\projtaun(R)}_F^2}_{\text{Variance}}, 
\end{align}
where we recall $\projperp = I_{nT} - \projtaun$. The error bound decomposition in \eqref{eq:bias_variance_proj_outliers} is similar to the bias-variance decomposition in \eqref{eq:bias_variance_proj}. In order to control the bias term $\norm{\projperp (D\Gshift^*)}_F^2$, we need to place an assumption on the level of fluctuation of $p(1), \dots, p(T)$. Indeed, consider the ideal scenario where $p(k) = p$ for all $k \in [T]$. Then, $\norm{\projperp (D\Gshift^*)}_F^2 = (1-\eta)^2 p^2 \norm{\projperp \Gshift^*}_F^2$ which can then be bounded using Lemma \ref{lem:bias_term}. This motivates the following assumption on the smoothness of the sequence $(p(k))_{k=1}^T$, or equivalently, on the diagonal entries of matrix $D$.
\begin{assumption}\label{assump:D_smooth}
The matrix $D$ is $\mu$-smooth, i.e., for some $\mu \geq 1$ it holds for all $\tau \in [T]$ that
\begin{equation*}
    \|\revpostpub{\projperp (D\Gshift^*)}\|_F^2 \leq \mu \bar{d}^2 \|\revpostpub{\projperp \Gshift^*}\|_F^2, 
\end{equation*}
where $\bar{d} := \frac{\sum_{k=1}^T d(k)}{T}$, and $d(k) := (1-\eta) p(k)$.
\end{assumption}
If $p(k) = p$ for all $k$, then note that Assumption \ref{assump:D_smooth} holds with $\mu = 1$. Now, if $D$ satisfies Assumption \ref{assump:D_smooth}, we can readily bound the bias term using  \eqref{eq:bias_control} as 
\begin{equation} \label{eq:bias_bd_outliers}
    \|\revpostpub{\projperp (D\Gshift^*)}\|_F^2 \leq C_1 \frac{\mu \bar{d}^2 n T^2 S_T}{\tau^2} \indic_{\tau < T}
\end{equation}
for some constant $C_1 > 0$. For the variance term we have the following analog to Lemma \ref{lem:variance_term}. The proof mirrors that of Lemma \ref{lem:variance_term}, with a technical difference arising due to the different distribution of the entries of $R$. Its proof is deferred to Appendix \ref{app:proof_lemma_var_outliers}.
\begin{lemma}[Variance bound]\label{lem:variance_term_outliers}
Let $R(1),\ldots, R(T)$ be independent random Hermitian matrices arising from the Outliers model in Section \ref{subsec:prob_setup}. Denote $Q(p)$ to be the $\psi_2$ norm of a centered Bernoulli random variable with parameter $p$, and define
\begin{align*}
    \pmax &:= \max_{k \in [T]} p(k), \quad  Q := \eta + Q(\eta) + \max_{k \in [T]} Q(p(k)) \\
    V &:= \max_{k \in [T]} \set{p(k)(1-p(k))}, \quad \quad f(\eta, \pmax, V) := \pmax\eta + (1-\eta)^2 V + \pmax V (1-\eta).
\end{align*}
Then, there exist constants $C_1 \in (0,1)$ and $C_2 > 0$ such that for any $\delta \in (0, C_1)$, it holds with probability at least $1-\delta$ that
\begin{equation*}
\|\projtaun R\|^2_F \leq  C_2 \left(n^2\tau f(\eta, \pmax, V) + Q^2 n \sqrt{\tau} \log\paren{\frac1\delta}\right).
\end{equation*}
\end{lemma}
Using \eqref{eq:bias_bd_outliers} and Lemma \ref{lem:variance_term_outliers}, we arrive at the following theorem which is the second main result of this section. The proof uses similar calculations as for that of Theorem \ref{thm:error_wigner} and is deferred to Appendix \ref{app:proof_thm_outliers}.
\begin{remark}
    The exact expression for $Q(p)$ is known to be \cite{Ostrovsky2014ExactVF}
    \begin{equation*}
        Q(p) = \sqrt{\frac{1-2p}{4 \log(\frac{1-p}{p})}}.
    \end{equation*}
    Moreover, $Q(\cdot)$ is non-negative continuous on $[0,1]$, with $Q(0+0) = Q(1-0) = 0$, and $Q^{2}(1/2) = 1/8$ (see \cite{Ostrovsky2014ExactVF}). It is easy to also verify that $Q(p) \leq \frac{1}{\sqrt{2\log(1/p)}}$ if $p \leq 1/4$.
\end{remark}
%
%-----------------------------------------------------
% Main theorem for matrix denoising in outliers model
%-----------------------------------------------------
%
\begin{theorem}[\rev{Main result: $\Gshift^*$ recovery under Outliers model}]\label{thm:error_outliers}
Let $A$ be the stacked measurement matrix under the outliers model \eqref{eq:model_outlier}, with a ground truth $g^*$ satisfying Assumption \ref{assump:smooth}, and the matrix $D$ (c.f. \eqref{eq:model_outlier2}) satisfies Assumption \ref{assump:D_smooth} with parameter $\mu \geq 1$. Then there exist constants $C_1 \in (0,1)$ and $C_2,C_3, C_4, C_5 > 0$ such that for any $\delta \in (0,C_1)$, the following is true (recall the notation of Lemma \ref{lem:variance_term_outliers}). 
\begin{enumerate}
    \item Denote $\tilde{f}(\eta,\pmax,V,Q,\delta) := f(\eta,\pmax,V) + Q^2 \log(1/\delta)$, and choose 
    \begin{equation} \label{eq:opt_tau_outlier}
        \tau = \tau^* := \min\set{1+\floor{\paren{\frac{\mu \bar{d}^2 T^2 S_T}{n \tilde{f}(\eta,\pmax,V,Q,\delta)}}^{1/3}}, T}.
    \end{equation}
    Then with probability at least $1-\delta$, 
    \begin{align} \label{eq:outlier_gen_matden_bd}
        \norm{\est{G} - D\Gshift^*}_F^2 
        &\leq C_2 \max\set{\frac{\mu\bar{d}^2 n T^2 S_T}{1+\paren{\frac{\mu \bar{d}^2 T^2 S_T}{n \tilde{f}(\eta,\pmax,V,Q,\delta)}}^{2/3}}, \mu \bar{d}^2 n S_T} \indic_{\tau^* < T} \nonumber \\
        &+ C_3 \tilde{f}(\eta,\pmax,V,Q,\delta) n^2 \min\set{1+\paren{\frac{\mu \bar{d}^2 T^2 S_T}{n \tilde{f}(\eta,\pmax,V,Q,\delta)}}^{1/3}, T}.
    \end{align}

    \item Furthermore, if $T$ satisfies 
    \begin{equation} \label{eq:T_cond_outliers_simp}
        T \geq C_4 \max\set{\paren{\frac{n \tilde{f}(\eta,\pmax,V,Q,\delta)}{\mu \bar{d}^2 S_T}}^{1/2}, \frac{\mu \bar{d}^2 S_T}{n \tilde{f}(\eta,\pmax,V,Q,\delta)}}
    \end{equation}
    then \eqref{eq:outlier_gen_matden_bd} implies
\begin{align} \label{eq:outlier_simp_matden_bd}
 \norm{\est{G} - D\Gshift^*}_F^2 
  &\leq C_5 \tilde{f}^{2/3}(\eta,\pmax,V,Q,\delta) (\mu \bar{d}^2)^{1/3} n^{5/3} T^{2/3} S_T^{1/3}.
\end{align}
\end{enumerate}
\end{theorem}
%
%\HT{Correct square in denominator as pointed out by EA} The following remarks are useful to note.
%
\begin{remark} \label{rem:outlier_thm_perf_smooth}
If $S_T = 0$, then $\tau^* = 1$ and \eqref{eq:outlier_gen_matden_bd} simplifies to $\norm{\est{G} - D\Gshift^*}_F^2 \lesssim \tilde{f}(\eta,\pmax,V,Q,\delta) n^2$. Hence $\frac{1}{\sqrt{T}}\norm{\est{G} - D\Gshift^*}_F = O(1/\sqrt{T})$ which is the parametric rate. On the other hand, if $\eta = 0$ and $p(k) = 1$ for all $k$, then note that $Q = V = \tilde{f}(\eta,\pmax,V,Q,\delta) = f(\eta,\pmax,V) = 0$. This implies $\tau^* = T$ and $\norm{\est{G} - D\Gshift^*}_F = 0$.
\end{remark}

\begin{remark} \label{rem:outlier_thm_int_case}
Suppose $\eta$ is constant and let $\pmax \asymp 1/n$. This means that a given measurement graph $\calH_k$ is disconnected with high probability. Moreover, suppose $\bar{d} \asymp 1/n$, and $\mu \asymp 1$ (e.g., this will happen if $p(k) = p \asymp 1/n$ for all $k$). Then we have $Q \asymp 1$, $V \asymp 1/n$ which implies $f(\eta,\pmax,V) \asymp 1/n$  and $\tilde{f}(\eta,\pmax,V,Q,\delta) \asymp \log(1/\delta)$ (assuming $\delta$ is a constant for simplicity). Consequently, the condition in \eqref{eq:T_cond_outliers_simp} simplifies to 
\begin{equation} \label{eq:T_cond_outlier_simp_ex1}
        T \gtrsim \max \set{\paren{\frac{n^3 \log(1/\delta)}{S_T}}^{1/2}, \paren{\frac{S_T}{n^3 \log(1/\delta)}}}
    \end{equation}
    and \eqref{eq:outlier_simp_matden_bd} translates to 
    $\norm{\est{G} - D\Gshift^*}_F^2 \lesssim  n T^{2/3} S_T^{1/3} \log^{2/3}(1/\delta)$. Thus for the smoothness regime $S_T = \omega(1/T^2)$ and $S_T = o(T)$, we see that  \eqref{eq:T_cond_outlier_simp_ex1} will be satisfied for $T$ suitably large w.r.t $n$.
\end{remark}
%
%
%------------------------------------------
% Spectral local synchronization guarantee
%
\subsection{Spectral local synchronization guarantee}\label{sec:recovery}
We now proceed to establish guarantees for the local synchronization phase within Algorithm \ref{algo:mat_denoise_proj}. While Algorithm \ref{algo:mat_denoise_proj} employs a TRS-based methodology at this stage, it is easier to analyze the spectral method, as delineated in Remark \ref{rmk:spectral_local_sync}. \rev{This is outlined as a separate method in Algorithm \ref{algo:mat_denoise_proj_locspec} for clarity. Our main results for recovering $g^*$ are Corollaries \ref{corr:error_wigner_spec} and \ref{cor:error_outliers_spec} for the AGN and Outliers models respectively.}

%
%
%------------------------------------------------
% Global matrix denoising with local Spectral method based synchronization
%------------------------------------------------

\begin{algorithm}
    \caption{\rev{Global Matrix Denoising $+$ Local Spectral method for Dynamic Synchronization (\matdenoisingLocSpec)}}\label{algo:mat_denoise_proj_locspec}
    \begin{algorithmic}[1]
    \State \rev{{\bf Input:} Matrix $A\in\matC^{nT\times n}$ of observations, parameter $\tau\in\matN$.}
    
    \State \rev{Define $\est G=\calP_{\tau,n} A$, where $\calP_{\tau,n}\in \matC^{nT\times nT}$ is the projection onto $\calL_{\tau,n}$.}
     
     \For{\rev{$k=1\dots ,T$}}
      \State \rev{Hermitianize $\est{G}'(k) := (\est G(k)+\conj{\est G(k)})/2$.}
     
        \State \rev{Find $\est{g}'(k)\leftarrow$ the eigenvector associated to the largest eigenvalue of $\est{G}'(k)$.}
    
    \State \rev{Find $\est{g}(k)= \calP_{\calC_{n}}(\est{g}'(k))e^{-\iota \teonek}$, where $\teonek$ is defined by $\est{g}'_1(k)=|\est{g}'_1(k)|e^{\iota \teonek}$  (with $\teonek = 0$ if $\est{g}'_1(k) = 0$)}.
    \EndFor
     
    \State \rev{{\bf Output:} Estimator $\est g=\concat\big(\est g(1),\cdots,\est g(T)\big)\in \calC_{nT}$}.
    \end{algorithmic}
\end{algorithm}

Our strategy involves utilizing the Davis-Kahan inequality ``locally'', focusing on specific time instances where the error converges to zero. 
To begin with, denote for each $k$ a scaling $s(k) \geq 0$ where $s(k) = 1$ for the AGN model, and $s(k) = d(k) = (1-\eta)p(k)$ for the outliers model (recall Assumption \ref{assump:D_smooth}). Then note that since $\Gshift^*(k)$ is symmetric, \rev{we have for the Hermitian matrix $\est{G}'(k)$ defined in line $4$ of Algorithm \ref{algo:mat_denoise_proj_locspec} that}
\begin{equation} \label{eq:symmetr_norm_ineq}
\|\est{G}'(k)-s(k)\Gshift^*(k)\|_F \leq \|\est{G}(k)-s(k)\Gshift^*(k) \|_F.
\end{equation}
The next step is to use the Davis-Kahan theorem \cite{daviskahan} to relate the largest eigenvector of $\est{G}'(k)$ to $g^*(k)/\sqrt{n}$, for each $k$. Since $s(k) \Gshift^*(k)= s(k) (g^*(k)\conj{g^*(k)} - I_n)$, the spectral gap between the largest and second largest eigenvalues of $s(k) \Gshift^*(k)$ is $s(k) n$. 
%Following the notation in Section \ref{sec:matrix_denoising}, 
%
\rev{Recall $\est{g}'(k)$ as defined in line $5$ of Algorithm \ref{algo:mat_denoise_proj_locspec}.} Then, by the Davis-Kahan theorem \cite{daviskahan}, we have for each $k$ that
\begin{align} \label{eq:dk_thm_classic}
 \left\|\Big(I - \frac{g^*(k) g^*(k)^{H}}{n} \Big) \est{g}'(k)\right\|_2 
 %= \left\|\sin \Theta\Big(\est{g}'(k), \frac{g^*(k)}{\sqrt{n}} \Big)\right\|_{op} 
 \leq \frac{4\|\est{G}'(k)-s(k)\Gshift^*(k)\|_{op}}{s(k) n}.  
\end{align}
A standard calculation reveals (see e.g., \cite[Appendix D.1]{SVDRank}) that there exists $\phik\in[0,2\pi)$ (depending on $g^*(k), \est{g}'(k)$) such that 
\begin{align*}
    \left\|\est{g}'(k)-e^{\iota \phik}\frac{g^*(k)}{\sqrt n}\right\|_2 
    &\leq 2 \left\|\Big(I - \frac{g^*(k) (g^*(k))^{H}}{n} \Big) \est{g}'(k)\right\|_2 
\end{align*}
which together with \eqref{eq:dk_thm_classic} implies 
\begin{equation}
     \left\|\est{g}'(k)-e^{\iota \phik}\frac{g^*(k)}{\sqrt n}\right\|_2  
     \leq 8\frac{\|\est{G}'(k)-s(k)\Gshift^*(k)\|_{op}}{s(k) n}  
    \leq 8\frac{\|\est{G}'(k)-s(k)\Gshift^*(k)\|_{F}}{s(k) n}. \label{eq:block_daviskahan} 
\end{equation}
Using \eqref{eq:block_daviskahan}, \eqref{eq:symmetr_norm_ineq}, we arrive at the following result.
\begin{proposition}\label{prop:recovery_spectral}
   Let $g^*$ be a ground truth vector satisfying Assumption \ref{assump:anchoring} as defined in \eqref{eq:ground_truth_definition}, \rev{and let $\est{g} = \concat\big(\est g(1),\cdots,\est g(T)\big)\in \calC_{nT}$ be obtained from Algorithm \ref{algo:mat_denoise_proj_locspec}.} 
    Then there exists  a constant $C > 0$ such that  
      \begin{equation*}
        \|\est g(k) - g^*(k)\|_2 
        \leq C \frac{\|\est{G}(k)-s(k)\Gshift^*(k)\|_{F}}{s(k)}.
       \end{equation*}
Denoting $\smin := \min_{k} s(k)$, it then follows that 
      \begin{equation*}
          \norm{\est{g} - g^*}_2^2 =  \sum_{k=1}^T \norm{\est{g}(k) - g^*(k)}_2^2
          \leq C^2 \frac{1}{\smin^2} \sum_{k=1}^T \|\est{G}(k)-s(k)\Gshift^*(k)\|_{F}^2.
      \end{equation*}
\end{proposition}
The proof is deferred to Appendix \ref{proof:prop_recovery}. 
Combining Proposition \ref{prop:recovery_spectral} with Theorems \ref{thm:error_wigner} and \ref{thm:error_outliers} directly leads to error guarantees for Algorithm \ref{algo:mat_denoise_proj}, with spectral local synchronization, for the AGN and Outliers models. Recall that for the AGN model, $\smin = 1$, while in the outliers model, $\smin = (1-\eta)\min_k p(k)$.
%
%
%------------------------------------
% Spectral recovery result for wigner
%------------------------------------
\begin{corollary}[\rev{Main result: $g^*$ recovery under AGN model}]\label{corr:error_wigner_spec}
    Let $A$ be the stacked measurement matrix for the AGN model in \eqref{eq:model_Wigner}, with a ground truth $g^*$ satisfying Assumption \ref{assump:smooth}. Choosing $\tau = \tau^*$ with $\tau^*$ as in \eqref{eq:opt_tau_agn}, consider the estimator $\est g = \concat\left(\est g(1),\ldots, \est g(T)\right)$ \rev{obtained in Algorithm \ref{algo:mat_denoise_proj_locspec}}. There exist  constants $C_1, C_2, C_3, C_4 >0$ such that for any $\delta\in (0,1)$ the following is true. 
    \begin{align}  \label{eq:error_bound_spec_Wigner}
        \norm{\est{g} - g^*}_2^2
        &\leq  C\Big(\max\set{\frac{nT^2 S_T}{\paren{1 + (\frac{T^2 S_T}{n \sigma^2}\Big)^{2/3}}}, n S_T} \indic_{\set{\tau^* < T}}   
        + \sigma^2 n^2 \min\set{1 + \Big(\frac{T^2 S_T}{n \sigma^2}\Big)^{1/3}, T}  
        \nonumber \\
        &+ \sigma^2 \log\paren{\frac4\delta} \Big). 
    \end{align}
    Furthermore, there exist constants $C_1, C_2 >0$ such that if  
    \begin{equation*}  
        T \geq C_1 \max\set{\paren{\frac{n\sigma^2}{S_T}}^{1/2}, \paren{\frac{S_T}{n\sigma^2}}, \sigma^4 S_T n^5}
    \end{equation*}
    then \eqref{eq:error_bound_spec_Wigner} implies
    \begin{equation*}  
      \norm{\est{g} - g^*}_2^2
        \leq  C_2 \paren{\sigma^{4/3} T^{2/3} S_T^{1/3} n^{5/3} + \sigma^2 \log\paren{\frac{4}{\delta}}}.
    \end{equation*}
\end{corollary}
\begin{remark} \label{rem:wigner_spec_rec}
The following points are useful to note regarding Corollary \ref{corr:error_wigner_spec}.
\begin{enumerate}
\item  Following Remark \ref{rem:wigner_thm_perf_noise_smooth}, when $\sigma = 0$ then we obtain $\tau^* = T$ and $\|\est{g} - g^*\|_2 = 0$. 
 If $S_T = 0$, we obtain $\tau^* = 1$ and $$\norm{\est{g} - g^*}_2^2 \lesssim  \paren{\sigma^2 n^2 + \sigma^2 \log(4/\delta)}.$$

\item Following Remark \ref{rem:wigner_thm_int_case}, suppose $\sigma \asymp T^{\alpha}$ for some $\alpha \in [0,1/4)$. Then provided $S_T$ satisfies \eqref{eq:wig_rem_int_smooth_conds}, part 2 of Corollary \ref{corr:error_wigner_spec} implies that 
 \begin{align*}
     \frac{1}{T} \|\est{g} - g^*\|_2^2 \stackrel{T \rightarrow \infty}{\longrightarrow} 0.
 \end{align*}
In particular, when $\sigma > 0$ is a constant, then $\|\est{g} - g^*\|_2^2 = O(T^{2/3} S_T^{1/3})$ which matches the \rev{(optimal)} rate for denoising smooth signals on a path graph \rev{on $T$ vertices} \cite[Theorems 5,6]{SadhanalaTV16}; \rev{recall the discussion in Remark \ref{rem:connec_den_smooth_sig_graph}}.
\end{enumerate}
\end{remark}

%
%---------------------------------------
% Spectral recovery result for outliers model
%---------------------------------------
\begin{corollary}[\rev{Main result: $g^*$ recovery under Outliers model}]\label{cor:error_outliers_spec}
Let $A$ be the stacked measurement matrix under the outliers model \eqref{eq:model_outlier}, with a ground truth $g^*$ satisfying Assumption \ref{assump:smooth}, and the matrix $D$ (c.f. \eqref{eq:model_outlier2}) satisfies Assumption \ref{assump:D_smooth} with parameter $\mu \geq 1$. 
For $\tau = \tau^*$ as in \eqref{eq:opt_tau_outlier}, consider the estimator $\est g = \concat\left(\est g(1),\ldots, \est g(T)\right)$ \rev{obtained in Algorithm \ref{algo:mat_denoise_proj_locspec}} and recall the notation of Theorem \ref{thm:error_outliers}. 
There exist constants $C_1 \in (0,1)$ and $C_2, C_3, C_4 > 0$ such that for any $\delta \in (0,C_1)$, the following is true.
\begin{enumerate}
    \item  Denoting $\pmin := \min_k p(k)$ it holds with probability at least $1-\delta$,   
    \begin{align}  
     \|\est{g} - g^*\|_2^2
        &\leq \frac{C_2}{(1-\eta)^2 \pmin^2} \Bigg( \max\set{\frac{\mu\bar{d}^2 n T^2 S_T}{1+\paren{\frac{\mu \bar{d}^2 T^2 S_T}{n \tilde{f}(\eta,\pmax,V,Q,\delta)}}^{2/3}}, \mu \bar{d}^2 n S_T} \indic_{\tau^* < T} \nonumber \\
        &+ \tilde{f}(\eta,\pmax,V,Q,\delta) n^2 \min\set{1+\paren{\frac{\mu \bar{d}^2 T^2 S_T}{n \tilde{f}(\eta,\pmax,V,Q,\delta)}}^{1/3}, T} \Bigg). \label{eq:error_bound_spec_outlier}
    \end{align}

    \item Furthermore, if $T$ satisfies \eqref{eq:T_cond_outliers_simp} (with a constant $C_3 > 0$), then \eqref{eq:error_bound_spec_outlier} implies that 
\begin{align*}  
 \norm{\est{g} - g^*}_2^2 
        \leq  
   C_4 \left(\frac{\tilde{f}^{2/3}(\eta,\pmax,V,Q,\delta)}{(1-\eta)^2 \pmin^2}\right) (\mu \bar{d}^2)^{1/3} n^{5/3} T^{2/3} S_T^{1/3}.
\end{align*}
\end{enumerate}
\end{corollary}
%\HT{Correct square in denominator in part (i) as pointed out by EA}
%------------
% Remarks
%------------
%
\begin{remark}  \label{rem:outlier_cor_spec}
Let us note the following points regarding Corollary \ref{cor:error_outliers_spec}.
\begin{enumerate}
\item Following Remark \ref{rem:outlier_thm_perf_smooth}, if $S_T = 0$, then $\tau^* = 1$ and we obtain 
$$\norm{\est{g} - g^*}_2^2  \lesssim \frac{\tilde{f}(\eta,\pmax,V,Q,\delta)}{(1-\eta)^2 \pmin^2} n^2.$$ 

On the other hand, if $\eta = 0$ and $p(k) = 1$ for all $k$, then $\tau^* = T$ and $\norm{\est{g} - g^*}_2 = 0$.
 
\item Consider the setup of Remark \ref{rem:outlier_thm_int_case} so that $\pmin \asymp 1/n$. Then if $T$ satisfies \eqref{eq:T_cond_outlier_simp_ex1} (which is the case for $T$ large enough when  $S_T = \omega(1/T^2)$ and $S_T = o(T)$ holds),  
 we obtain the bound    
$$\norm{\est{g} - g^*}_2^2 \lesssim  n^3 T^{2/3} S_T^{1/3} \log^{2/3}(1/\delta).$$ 
\end{enumerate}
\rev{The bound exhibits the same scaling\footnote{\rev{Note that the rates in \cite[Theorems 5,6]{SadhanalaTV16} are obtained for additive Gaussian noise. We believe the dependency on $T$ should be unchanged for more general sub-Gaussian noise models, however this needs to be established formally.}} w.r.t $T, S_T$ as in the AGN model; recall Remark \ref{rem:wigner_spec_rec}.}
\end{remark}

%
% Experiments
%
\section{Experiments}\label{sec:experiments}
We now provide some empirical results\footnote{The code is available at \url{https://github.com/ErnestoArayaV/dynamic_synch}.} on synthetic data generated by the AGN and Outliers models. Apart from Algorithms {\globaltrs}, {\localtrs} and {\matdenoising},  we also evaluate the block-wise spectral approach described in Remark \ref{rmk:spectral_local_sync}. 

\paragraph{Experiment setup.} The following remarks are in order. 
\begin{enumerate}
%
% Smooth data generation
\item For specified values of $n, T$ and smoothness $S_T$, we generate a smooth signal $g^* \in \calC_{nT}$, with $g^*_1(k) = 1$ for each block $k$, as follows. Define $\tau' := \min \set{\lfloor 1 + \sqrt{T S_T} \rfloor, T}$, define $\calP_{\tau',n-1}$ as in \eqref{eq:proj_smooth_op}. Generate a Gaussian $u \sim \calN(0, I_{(n-1)T})$ and compute 
$$u' = \sqrt{T} \frac{\calP_{\tau',n-1} u}{\norm{\calP_{\tau',n-1} u}_2}.$$
Then denoting $\tilde{g}^* := \exp(\iota (0.5 \pi) u')$ which is computed entry-wise, we form $g^*$ by prepending each block of $\tilde{g}^*$ with $1$. It is a simple exercise to verify that $g^*$ (almost surely) satisfies the smoothness condition in Assumption \ref{assump:smooth} for the specified $S_T$, up to a constant.

% Error criteria: RMSE
\item For an estimate $\est{g} \in \calC_{nT}$ of the ground truth truth $g^* \in \calC_{nT}$, we will measure the Root Mean Square Error (RMSE) defined as
$\frac{1}{\sqrt{nT}} \norm{\est{g} - g^*}_2$. We remark that an additional $\sqrt{n}$ factor appears here as compared to the quantity being bounded in the analyses earlier.

% Grid search for regularization parameters
\item In order to choose the appropriate value of the regularization parameters for the algorithms ($\lambda$ for \globaltrs and $\tau$ for \localtrs, \matdenoising), we consider a data-driven grid-search based heuristic. In particular, define a discretization of $[0,T]$, where (i) the first $\lfloor \sqrt{T} \rfloor + 1$ elements of the grid are the numbers $0,1,\dots,\lfloor \sqrt{T} \rfloor$; (ii) the next five elements are equi-spaced numbers in the interval $[\lfloor \sqrt{T} \rfloor, \lfloor T^{2/3} \rfloor]$; (iii) the last five elements are equi-spaced numbers in the interval $[\lfloor T^{2/3} \rfloor, T]$. For each value on the grid, call it $\beta$, we set $\tau = \min\set{\beta+1, T}$ and $\lambda = \beta \times \lambda_{\text{scale}}$ for a fixed choice of the scaling parameter $\lambda_{\text{scale}} > 0$. For the corresponding output $\est{g}$ of each algorithm, we compute a `data-fidelity' value using the input pairwise data, as for the data-fidelity term in \eqref{eq:opt_denoising_vector}  (with $\tilde{g}$ therein replaced by $\est{g}$). Then we select the optimal $\beta$ (thus optimal $\tau$ and $\lambda$) for each algorithm as follows.
\begin{enumerate}
\item For \globaltrs,  the optimal $\beta$ is the value that leads to the largest data-fidelity.

\item For \localtrs and \matdenoising, the optimal $\beta$ is chosen to be the value that leads to the maximum change in the slope of the data-fidelity (w.r.t $\beta$). 
\end{enumerate}
\end{enumerate}
The above data-driven rule for selecting the optimal $\beta$ is motivated by empirical evidence as shown in Figs. \ref{fig:plots_1_outlier_t20_singrun} and \ref{fig:plots_1_wigner_t20_singrun}. Notice that the data-fidelity for \localtrs \ and {\matdenoising}  exhibits a sharp change in the slope at $\beta = 1$ (which is where the RMSE is minimized for these methods). But it is not clear why the data-fidelity keeps increasing beyond $\beta = 1$, this needs further investigation. For \globaltrs, notice that the data fidelity peaks at the RMSE-minimizer in the outliers model (see Fig. \ref{fig:plots_1_outlier_t20_singrun}), while for the Wigner model, it peaks at a slightly sub-optimal $\beta$ (see Fig. \ref{fig:plots_1_wigner_t20_singrun}).

\begin{figure}[!htp]
\centering
  \includegraphics[scale = 0.15]{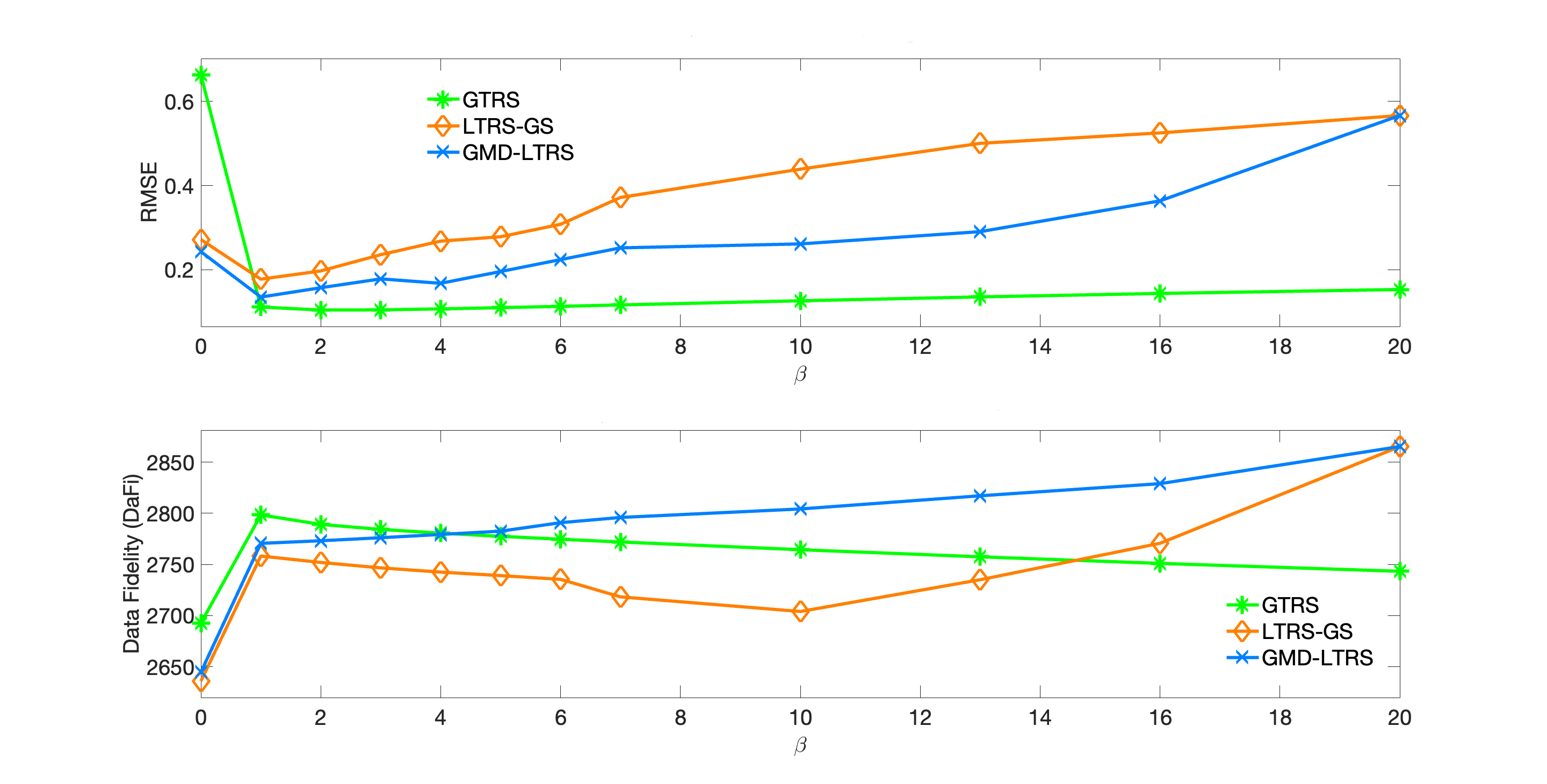}
  \caption{RMSE (top) and Data Fidelity (bottom) for the Outliers model for a single MC run ($n = 30, T = 20, S_T = 1/T$, $\eta = p = 0.2$). Notice a sharp change in slope of the data-fidelity for \localtrs \ and \matdenoising \ at $\beta = 1$, which is also where the RMSE is minimized for these two methods. For \globaltrs \ (with $\lambda_{\text{scale}} = 10$) the data-fidelity peaks at $\beta = 1$ which is where the RMSE is minimized.}
  \label{fig:plots_1_outlier_t20_singrun}
\end{figure}

\begin{figure}[!ht]
\centering
  \includegraphics[scale = 0.15]{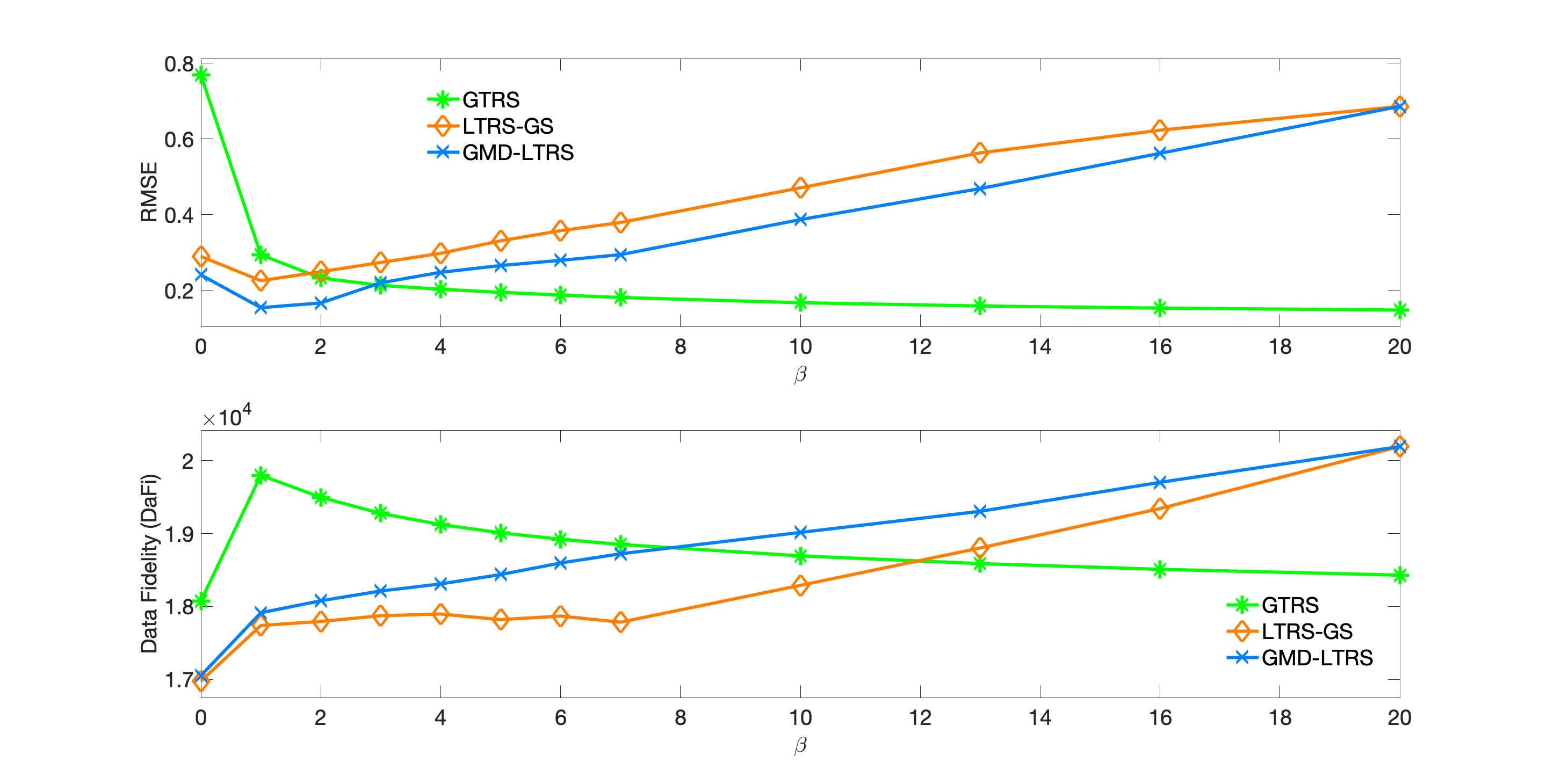}
  \caption{RMSE (top) and Data Fidelity (bottom) for the AGN model for a single MC run ($n = 30, T = 20, S_T = 1/T$, $\sigma = 3$). Notice a sharp change in slope of the data-fidelity for \localtrs \ and \matdenoising \ at $\beta = 1$, which is also where the RMSE is minimized for these two methods. For \globaltrs \ (with $\lambda_{\text{scale}} = 10$) the data-fidelity peaks at $\beta = 1$, but the RMSE is minimized for $\beta = 20$.}
  \label{fig:plots_1_wigner_t20_singrun}
\end{figure}

%
% RMSE versus T for all algos
%
\paragraph{Experiment 1: RMSE versus $T$.} In this experiment, we evaluate the RMSE versus $T$ for all the above methods. Throughout, we fix $n = 30$, and vary $T$ from $10$ to $100$. We perform $20$ Monte Carlo (MC) runs for each $T$. In a single MC run, we randomly generate the ground-truth for a specified smoothness $S_T$ (as described earlier), then generate the noisy pairwise data (as per specified noise model), and then compute the RMSE for each algorithm for the regularization parameters selected by the aforementioned data-fidelity rule. The RMSE values are then averaged over the $20$ runs. The results are shown in Figs. \ref{fig:plots_123_rmse_T_wigner_sigma_3_20runs} for the Wigner model, and Fig. \ref{fig:plots_123_rmse_T_outlier_eta_0p2_p_0p2_20runs} for the Outliers model, for different smoothness conditions $S_T$. As a sanity check, note that the RMSE of the spectral method remains high for all $T$. We can make some additional observations.
\begin{enumerate}
\item For the AGN model (Fig. \ref{fig:plots_123_rmse_T_wigner_sigma_3_20runs}), we see that \matdenoising \ performs the best out of the three methods, while the RMSE for \globaltrs \ does not shrink to zero with $T$. The performance of \localtrs \  is seen to improve with $T$, albeit with a larger error than \matdenoising.

\item For the Outliers model (Fig. \ref{fig:plots_123_rmse_T_outlier_eta_0p2_p_0p2_20runs}), the RMSE for all three methods decreases with $T$. Interestingly, \globaltrs \ typically performs the best out of the three methods (followed by \matdenoising  \ and then \localtrs). 
\end{enumerate}

%-------------------------------------------------------
% Figures for RMSE versus T for Wigner (PLOTS_{1,2,3})
%-------------------------------------------------------
 \begin{figure}[!htp]
\centering
\begin{subfigure}{.5\textwidth}
  \centering
  \includegraphics[width=\linewidth]{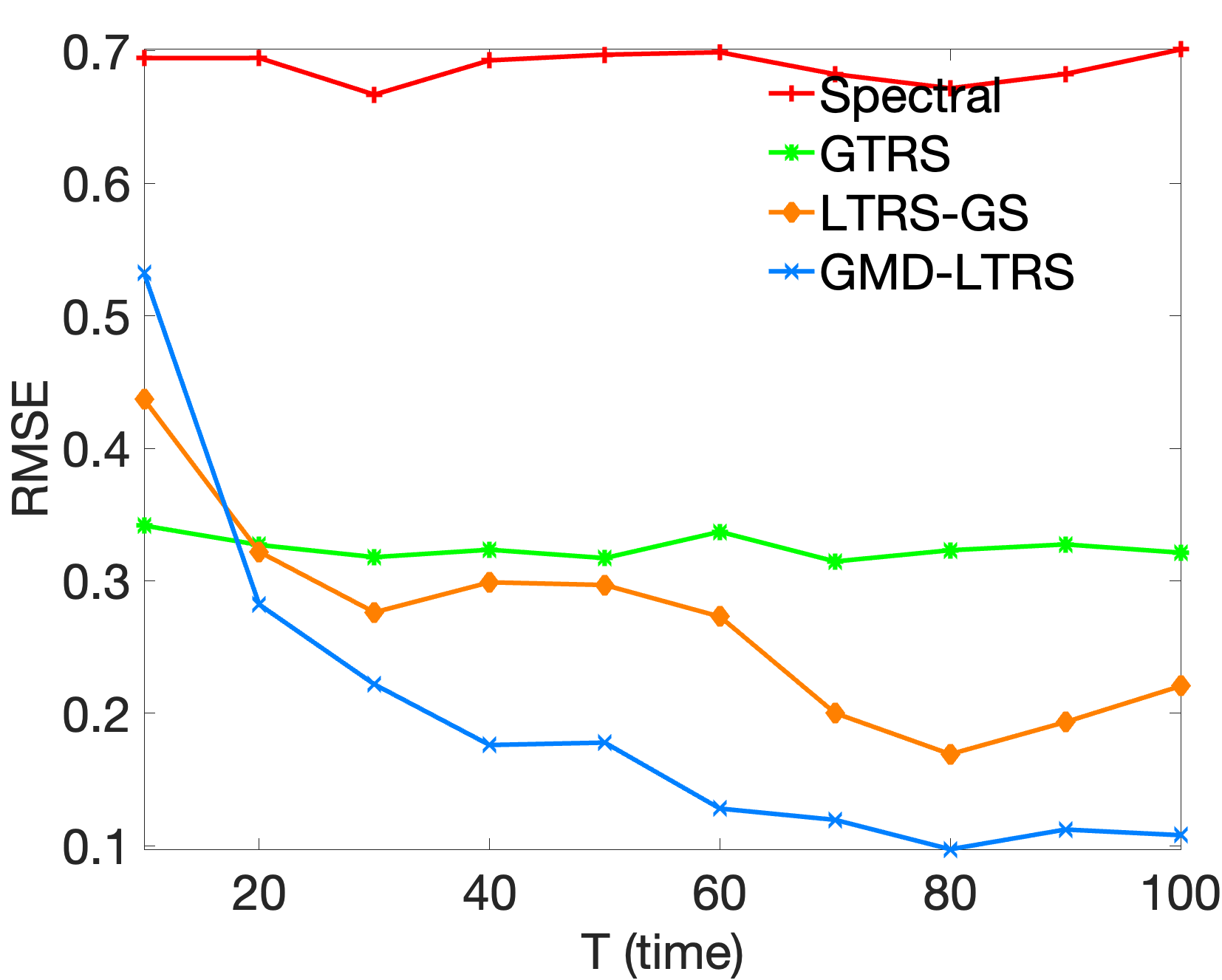}
  \caption{$S_T = 1/T$}
  %\label{fig:transync_alpha_ST_1}
\end{subfigure}%
\begin{subfigure}{.5\textwidth}
  \centering
  \includegraphics[width=\linewidth]{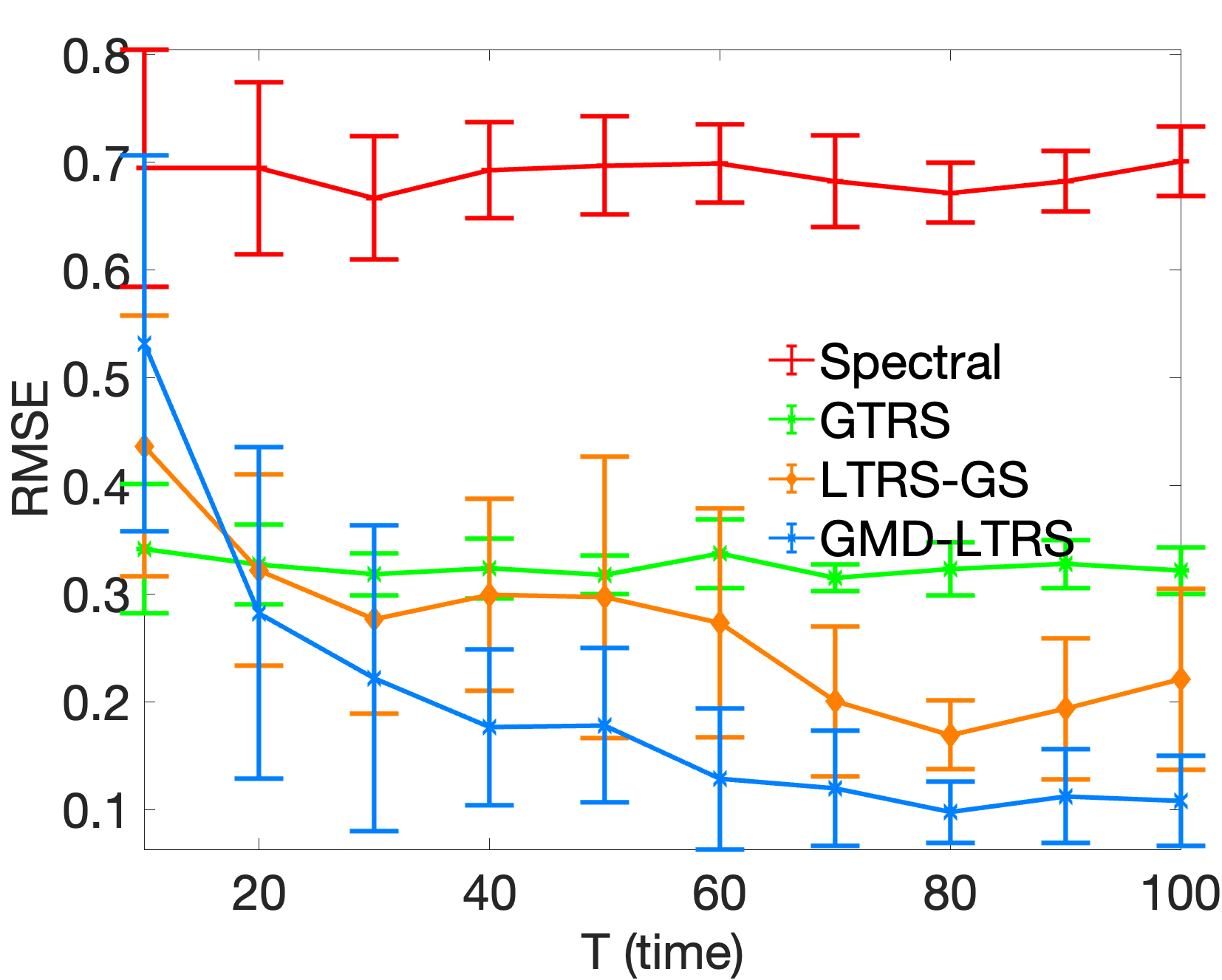}
  \caption{$S_T = 1/T$}
  %\label{fig:transync_alpha_ST_05}
\end{subfigure}%
\hfill
\begin{subfigure}{.5\textwidth}
  \centering
  \includegraphics[width=\linewidth]{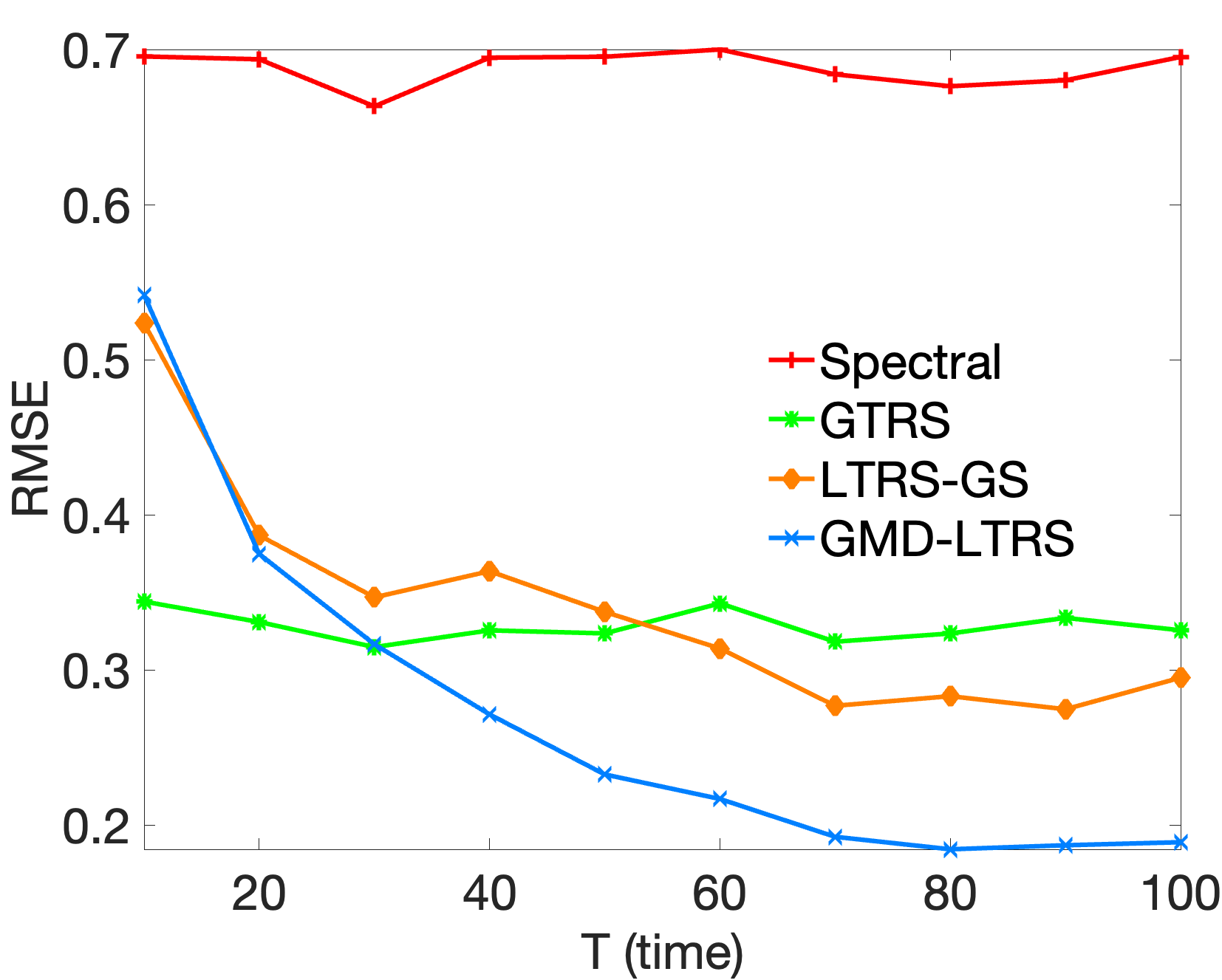}
  \caption{$S_T = 1$}
  %\label{fig:transync_alpha_ST_1}
\end{subfigure}%
\begin{subfigure}{.5\textwidth}
  \centering
  \includegraphics[width=\linewidth]{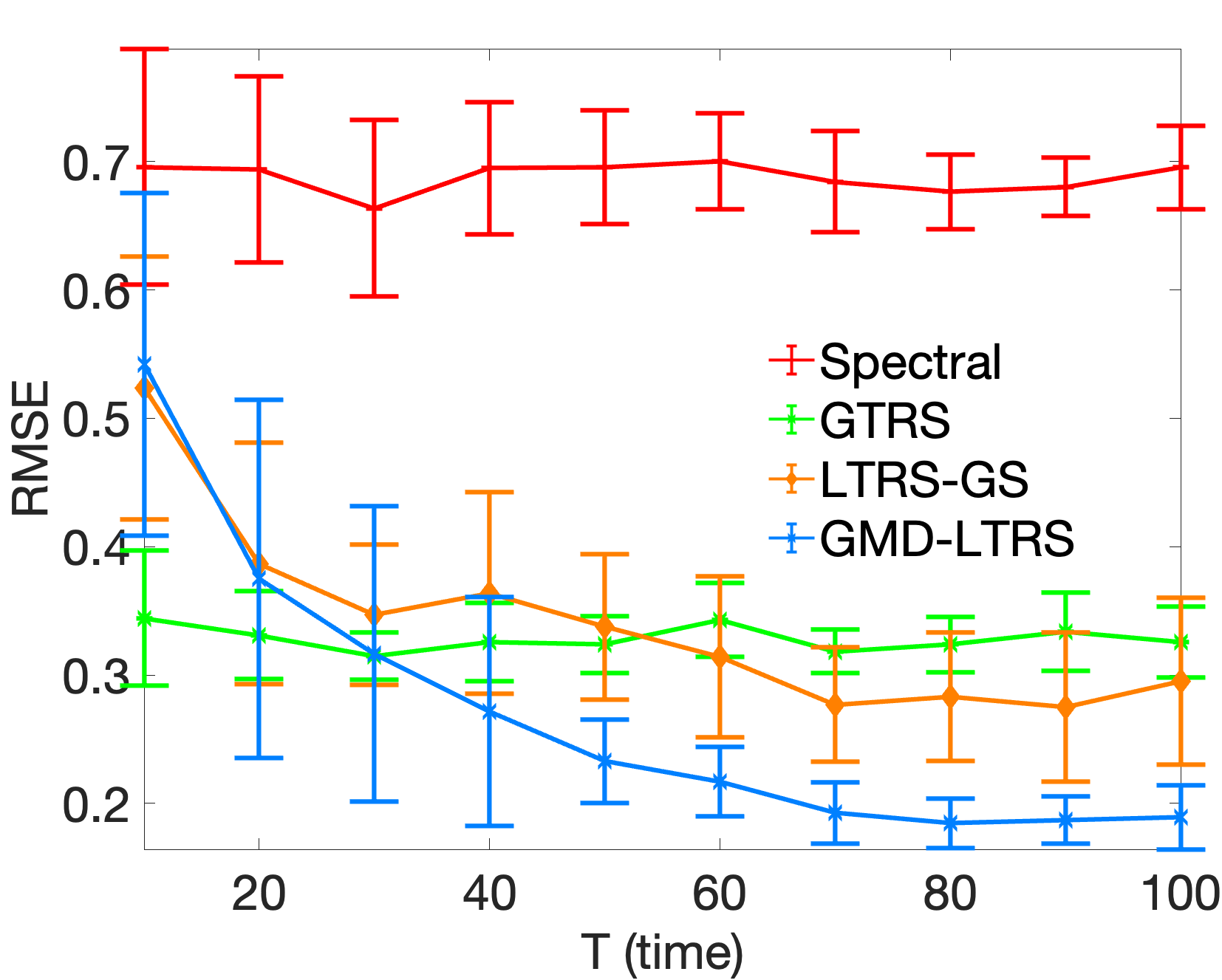}
  \caption{$S_T = 1$}
  %\label{fig:transync_alpha_ST_05}
\end{subfigure}%
\hfill
\begin{subfigure}{.5\textwidth}
  \centering
  \includegraphics[width=\linewidth]{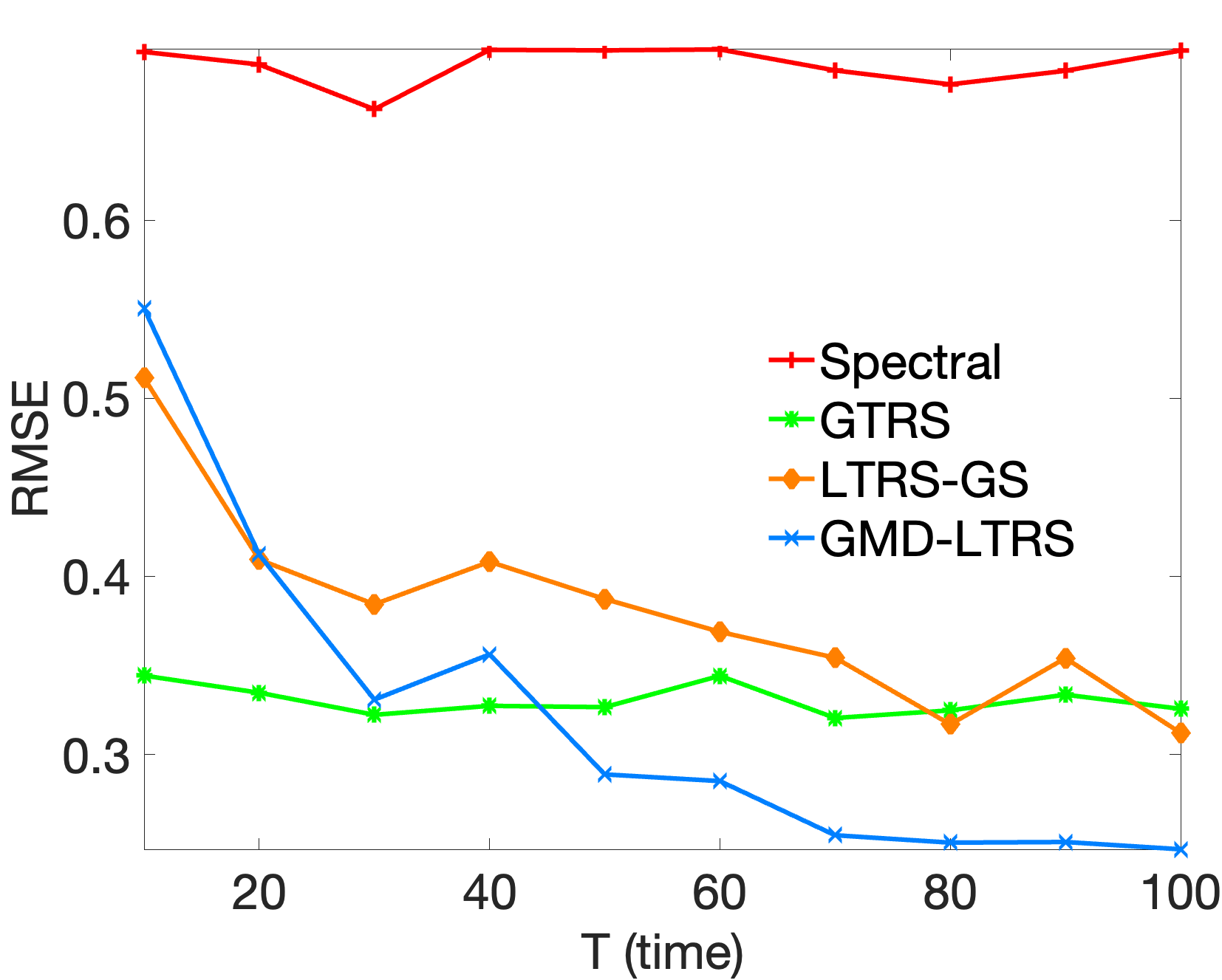}
  \caption{$S_T = T^{1/4}$}
  %\label{fig:transync_alpha_ST_1}
\end{subfigure}%
\begin{subfigure}{.5\textwidth}
  \centering
  \includegraphics[width=\linewidth]{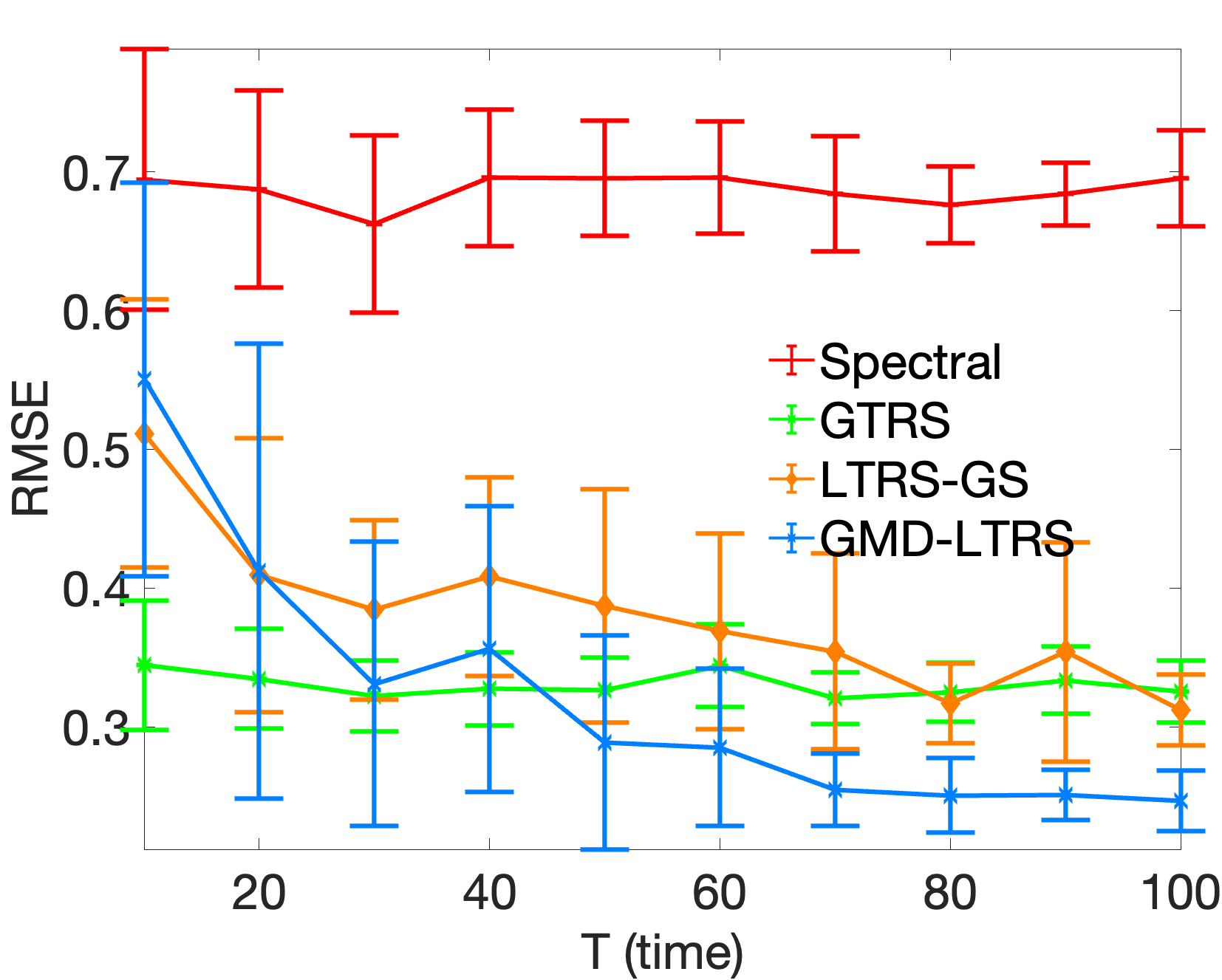}
  \caption{$S_T = T^{1/4}$}
  %\label{fig:transync_alpha_ST_05}
\end{subfigure}%
\caption{RMSE versus $T$ for AGN model ($n=30$, $\sigma = 3$) with $S_T \in \set{1/T,1,T^{1/4}}$, results averaged over $20$ MC runs. We choose $\lambda_{\text{scale}} = 10$ for $\globaltrs$. Results on left show average RMSE,  the right panel shows average RMSE $\pm$ standard deviation.}
\label{fig:plots_123_rmse_T_wigner_sigma_3_20runs}
\end{figure}

%------------------------------------------------------------------
% Figures for RMSE versus T for Outliers model (PLOTS_{1,2,3})
%------------------------------------------------------------------
 \begin{figure}[!htp]
\centering
\begin{subfigure}{.5\textwidth}
  \centering
  \includegraphics[width=\linewidth]{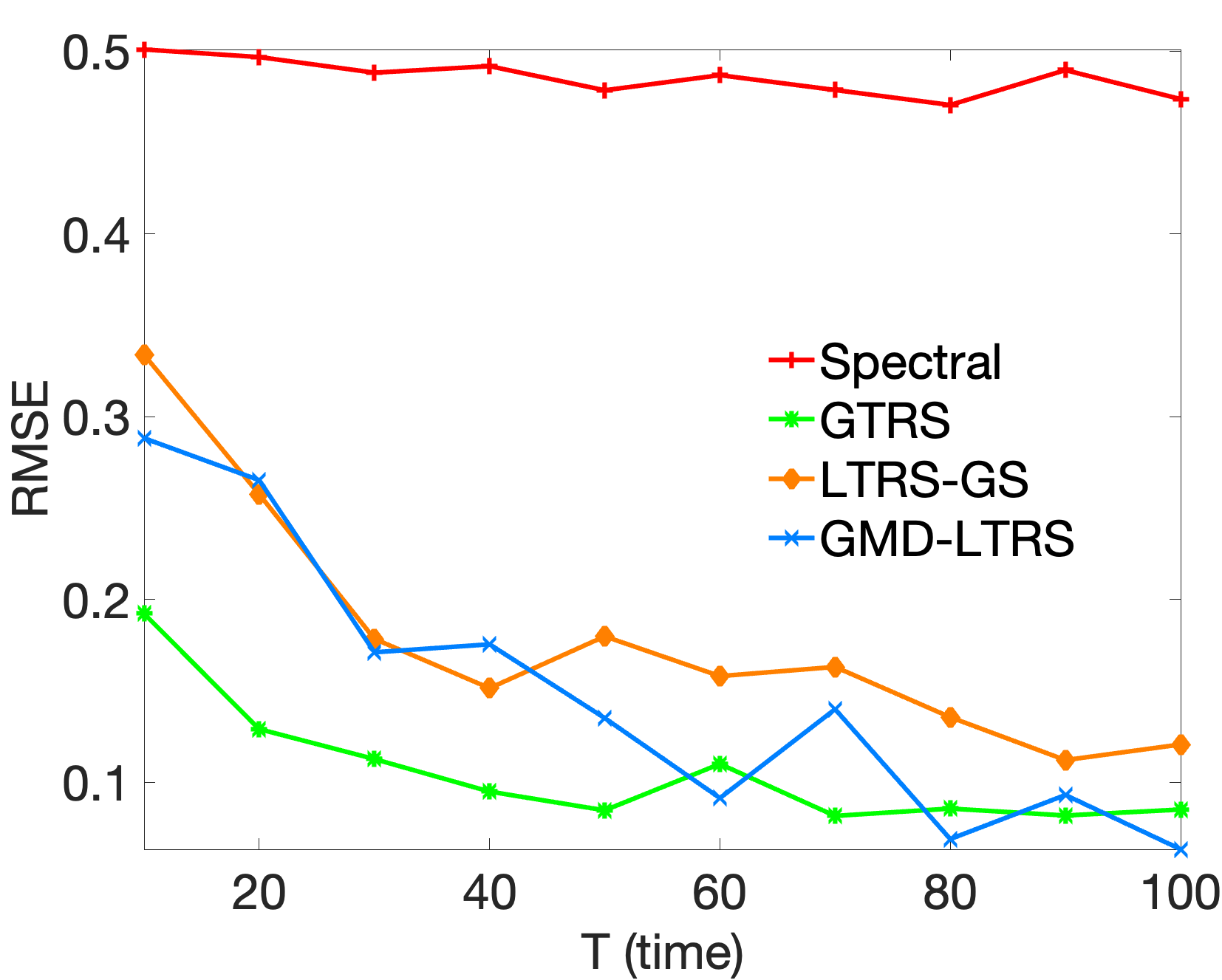}
  \caption{$S_T = 1/T$}
  %\label{fig:transync_alpha_ST_m05}
\end{subfigure}%
\begin{subfigure}{.5\textwidth}
  \centering
  \includegraphics[width=\linewidth]{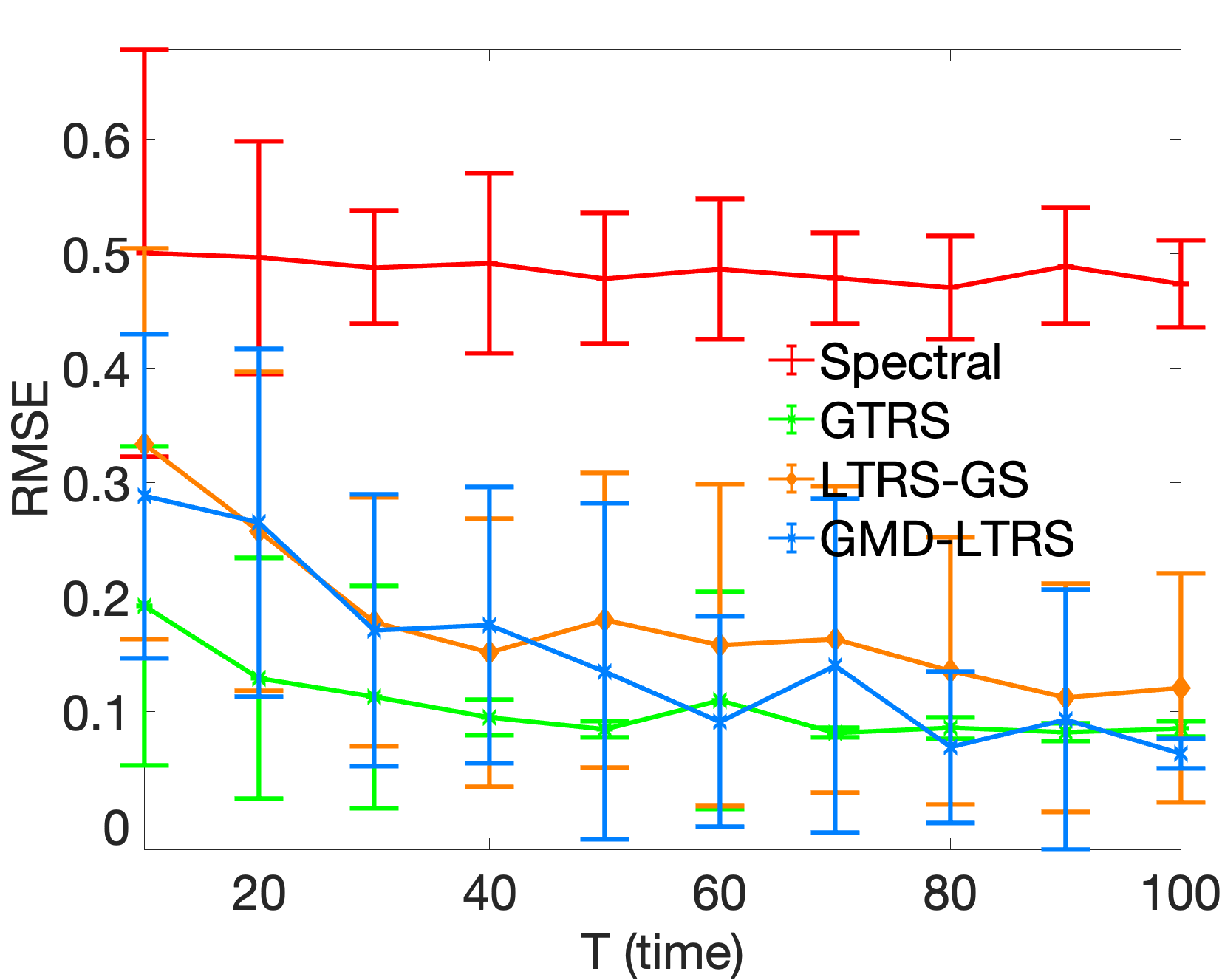}
  \caption{$S_T = 1/T$}
  %\label{fig:transync_alpha_ST_0}
\end{subfigure}%
\hfill
\begin{subfigure}{.5\textwidth}
  \centering
  \includegraphics[width=\linewidth]{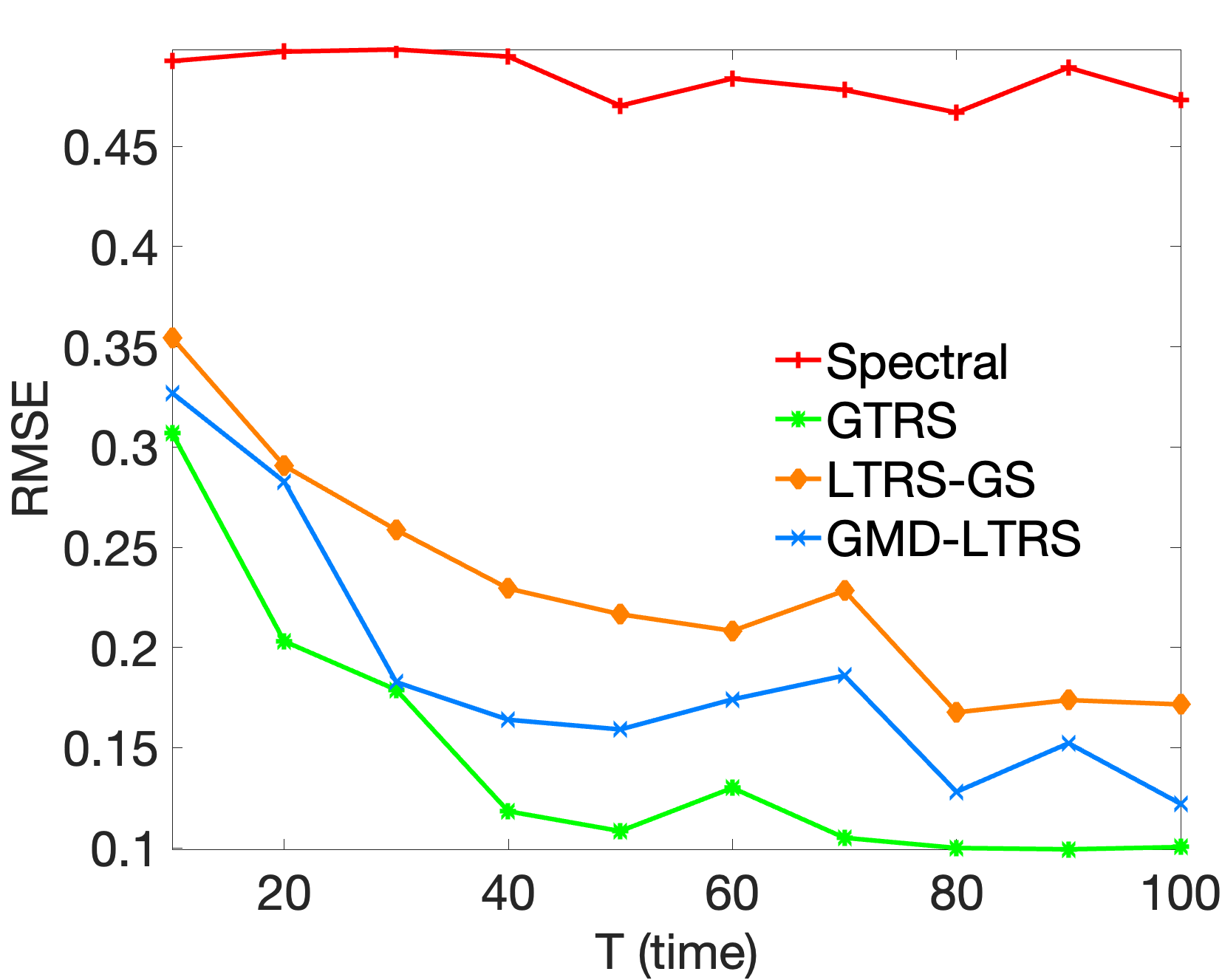}
  \caption{$S_T = 1$}
  %\label{fig:transync_alpha_ST_m05}
\end{subfigure}%
\begin{subfigure}{.5\textwidth}
  \centering
  \includegraphics[width=\linewidth]{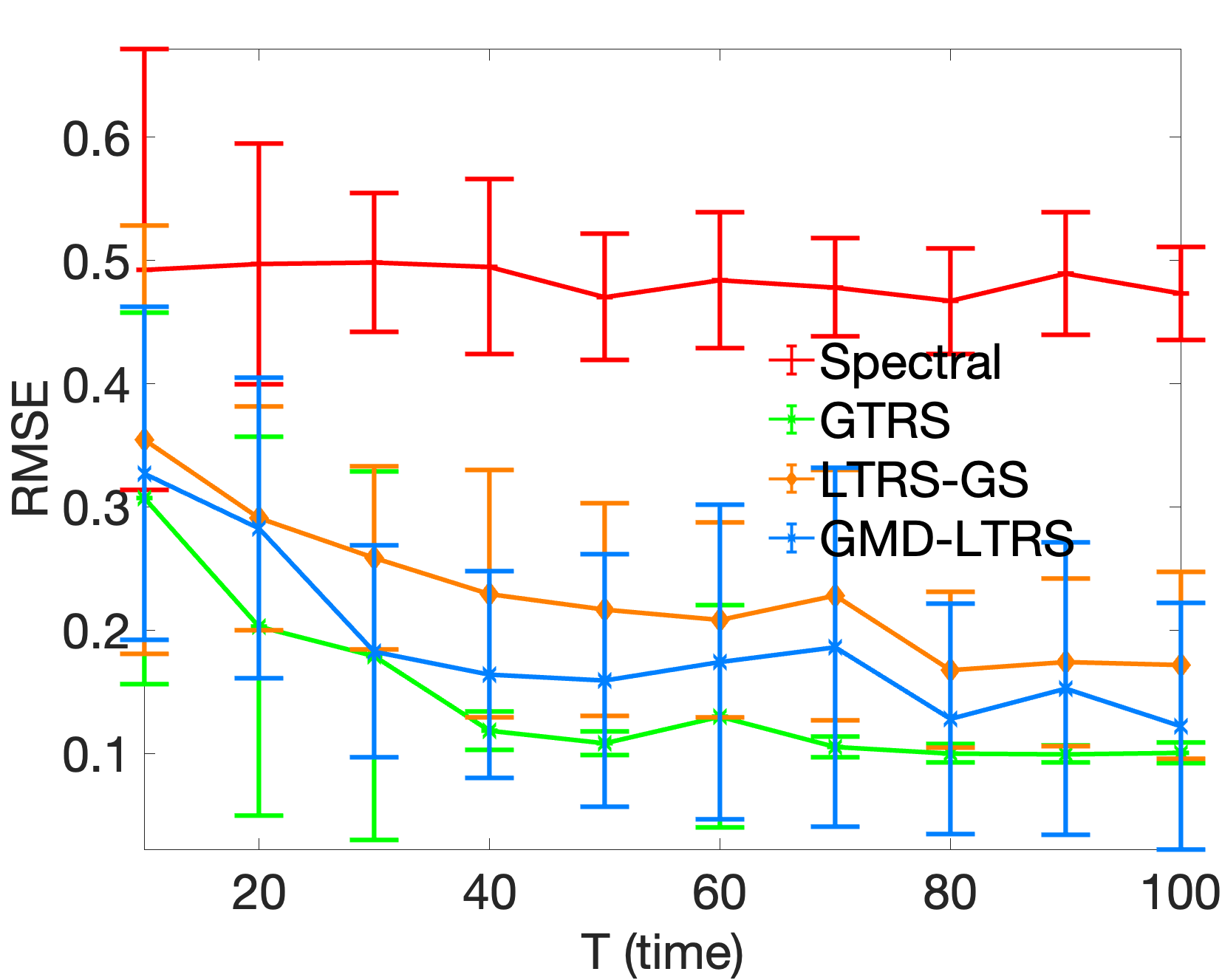}
  \caption{$S_T = 1$}
  %\label{fig:transync_alpha_ST_0}
\end{subfigure}
\hfill
\begin{subfigure}{.5\textwidth}
  \centering
  \includegraphics[width=\linewidth]{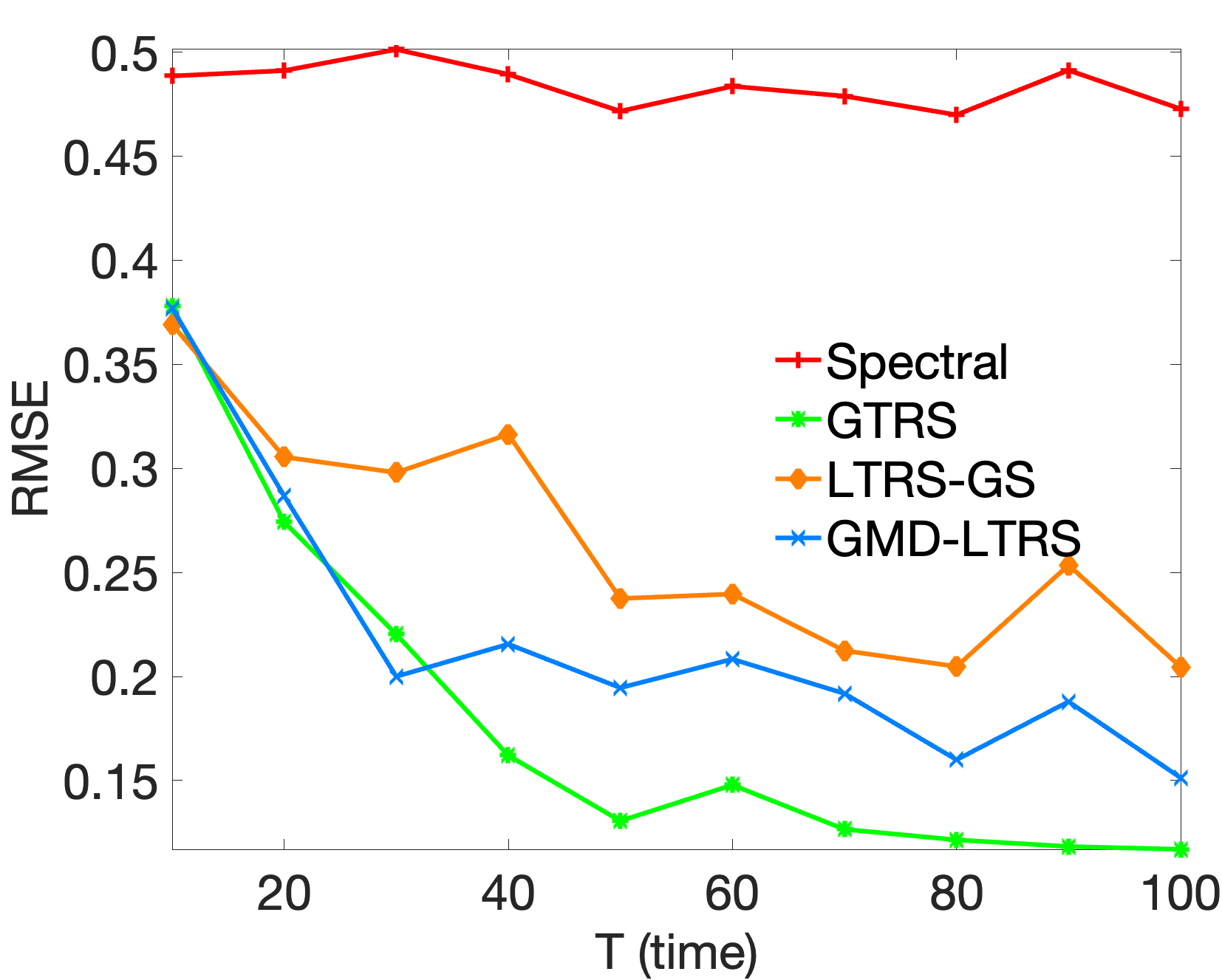}
  \caption{$S_T = T^{1/4}$}
  %\label{fig:transync_alpha_ST_m05}
\end{subfigure}%
\begin{subfigure}{.5\textwidth}
  \centering
  \includegraphics[width=\linewidth]{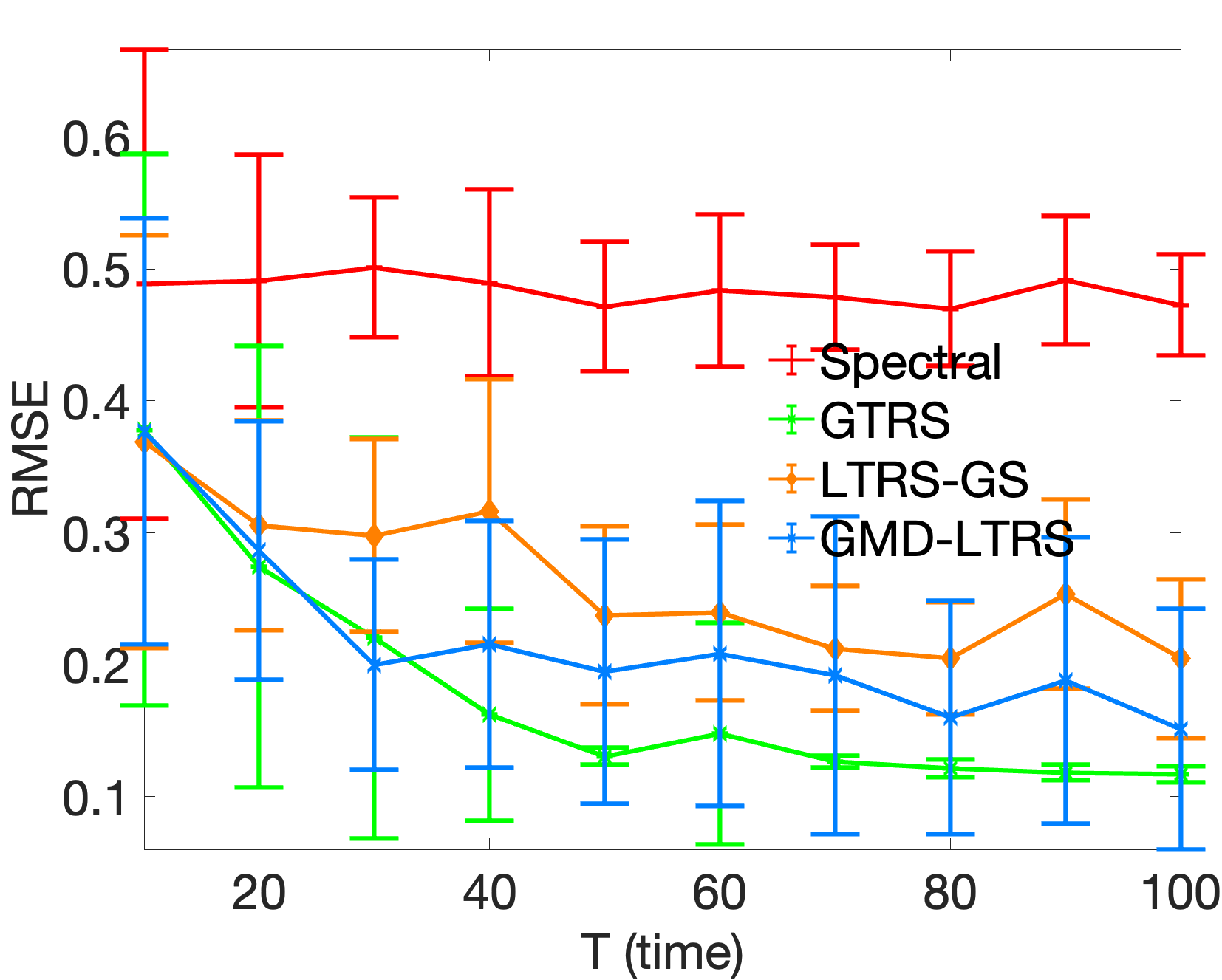}
  \caption{$S_T = T^{1/4}$}
  %\label{fig:transync_alpha_ST_0}
\end{subfigure}
\caption{RMSE versus $T$ for Outliers model ($n=30$, $\eta = p = 0.2$) with $S_T \in \set{1/T, 1, T^{1/4}}$, results averaged over $20$ MC runs. We choose $\lambda_{\text{scale}} = 10$ for $\globaltrs$. Results on left show average RMSE, the right panel shows average RMSE $\pm$ standard deviation.}
\label{fig:plots_123_rmse_T_outlier_eta_0p2_p_0p2_20runs}
\end{figure}

%
%
%------------------------------------------
% RMSE versus noise-level for all algos
%------------------------------------------
\paragraph{Experiment 2: RMSE versus noise-level.}In this experiment, we evaluate the RMSE versus the noise-level for all the above methods. The noise-level is denoted by $\gamma$ for ease of exposition, with $\gamma = \sigma > 0$ for the AGN model (varied from $0$ to $10$), and $\gamma = \eta \in [0,1]$ for the Outliers model (varied from $0$ to $0.8$). Throughout, we fix $n = 30$, $T = 20$, and perform
$20$ Monte Carlo (MC) runs for each value of $\gamma$. In a single MC run, we randomly generate the ground-truth for a specified smoothness $S_T$ (as described earlier), then generate the noisy pairwise data (as per specified noise model), and then compute the RMSE for each algorithm for the regularization parameters selected by the aforementioned data-fidelity rule. The RMSE values are then averaged over the $20$ runs. The results are shown in Fig. \ref{fig:plots_123_rmse_noise_wigner_T20_20runs} for the Wigner model, and Fig. \ref{fig:plots_123_rmse_noise_outliers_p0p2_T20_20runs} for the Outliers model (with $p = 0.2$). The following observations can be made from these plots.
\begin{enumerate}
    \item For the AGN model (Fig. \ref{fig:plots_123_rmse_noise_wigner_T20_20runs}), note that the performance of the spectral method and \globaltrs \ are similar for low noise levels ($\sigma$ less than $1$), both being worse than the other two methods in this noise regime. The performance of the spectral method degrades for $\sigma > 1$, while that of the other three methods is similar at moderate noise levels ($1 < \sigma < 5$). For larger values of $\sigma$, we see that $\matdenoising$ performs best, while $\localtrs$ and $\globaltrs$ have similar error values.

    \item For the Outliers model (Fig. \ref{fig:plots_123_rmse_noise_outliers_p0p2_T20_20runs}), we observe that the spectral method performs much worse than the other three methods for almost the whole noise regime.  We further observe that $\globaltrs$ has a similar performance to $\matdenoising$ and $\localtrs$ up to moderate noise levels ($\eta < 0.35$), but its performance degrades for $\eta > 0.35$. On the other hand, $\matdenoising$ appears to be more robust to high noise levels ($\eta > 0.35$) as compared to the other two methods.
\end{enumerate}

%
%----------------------------------------------------------------
% Figures for RMSE versus noise level for Wigner (PLOTS_{1,2,3})
%----------------------------------------------------------------
 \begin{figure}[!htp]
\centering
\begin{subfigure}{.5\textwidth}
  \centering
  \includegraphics[width=\linewidth]{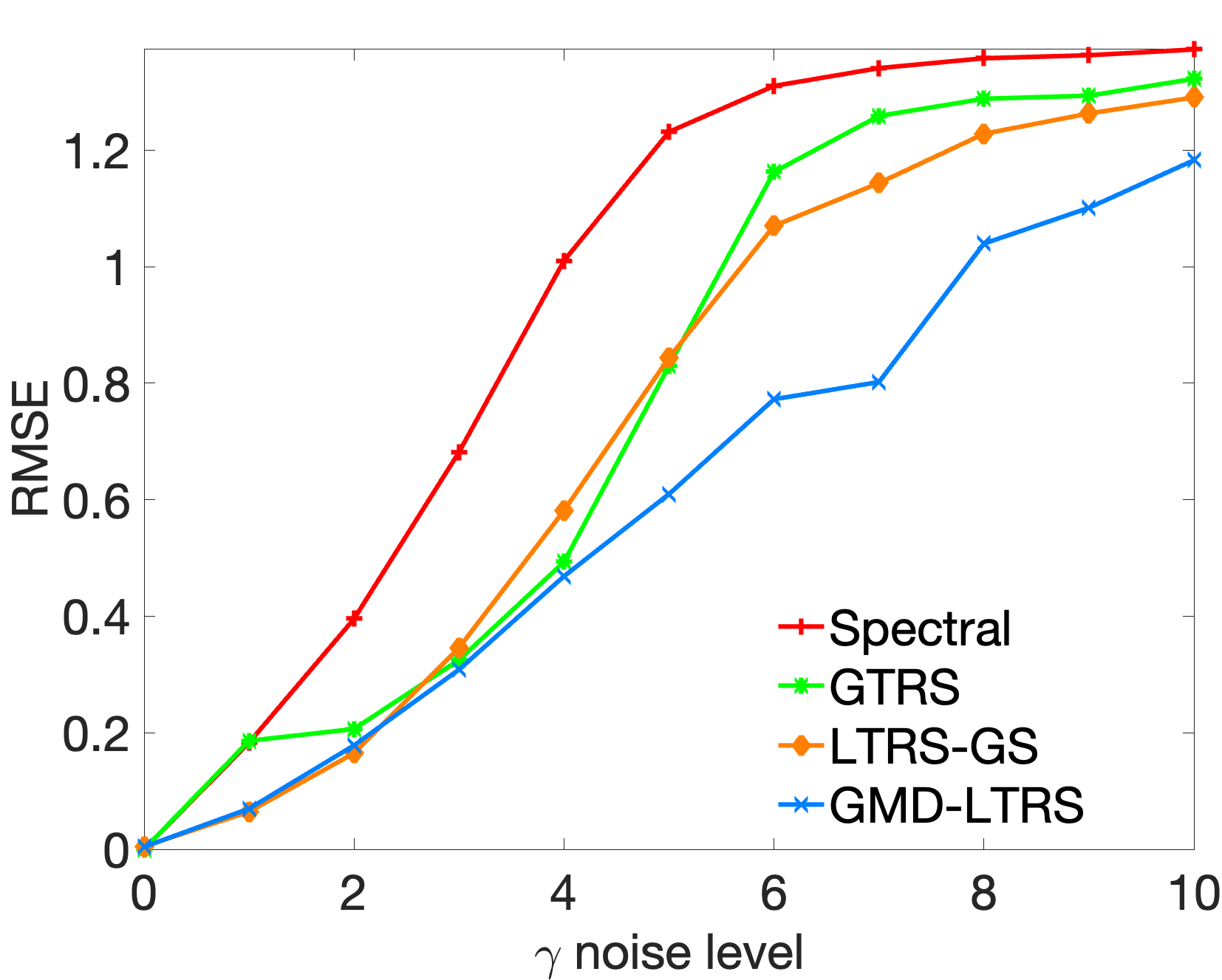}
  \caption{$S_T = 1/T$}
  %\label{fig:transync_alpha_ST_1}
\end{subfigure}%
\begin{subfigure}{.5\textwidth}
  \centering
  \includegraphics[width=\linewidth]{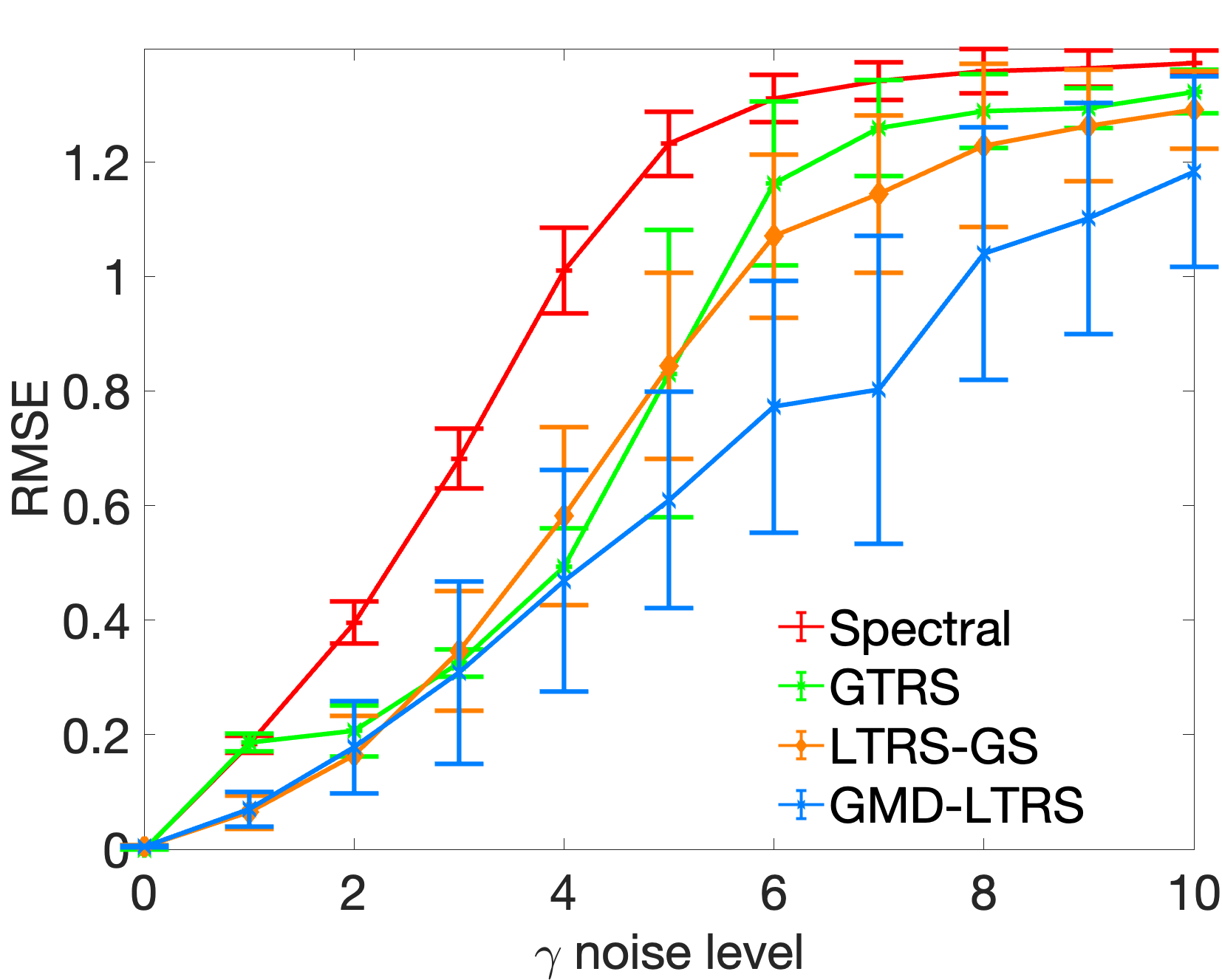}
  \caption{$S_T = 1/T$}
  %\label{fig:transync_alpha_ST_05}
\end{subfigure}%
%\caption{}
%\label{fig:plots_3_rmse_T_wigner_sigma_3_20runs}
\hfill
\begin{subfigure}{.5\textwidth}
  \centering
  \includegraphics[width=\linewidth]{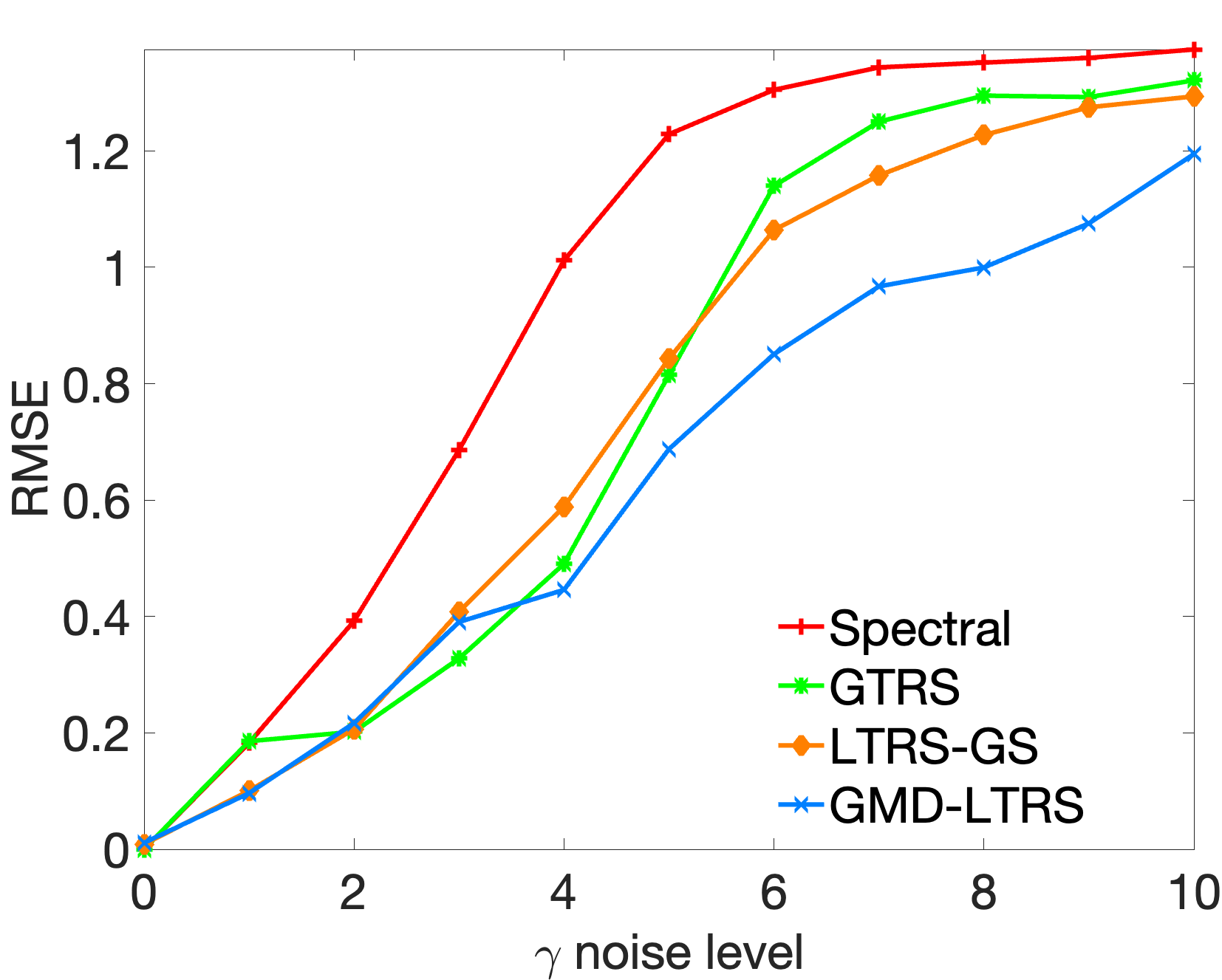}
  \caption{$S_T = 1$}
  %\label{fig:transync_alpha_ST_m05}
\end{subfigure}%
\begin{subfigure}{.5\textwidth}
  \centering
  \includegraphics[width=\linewidth]{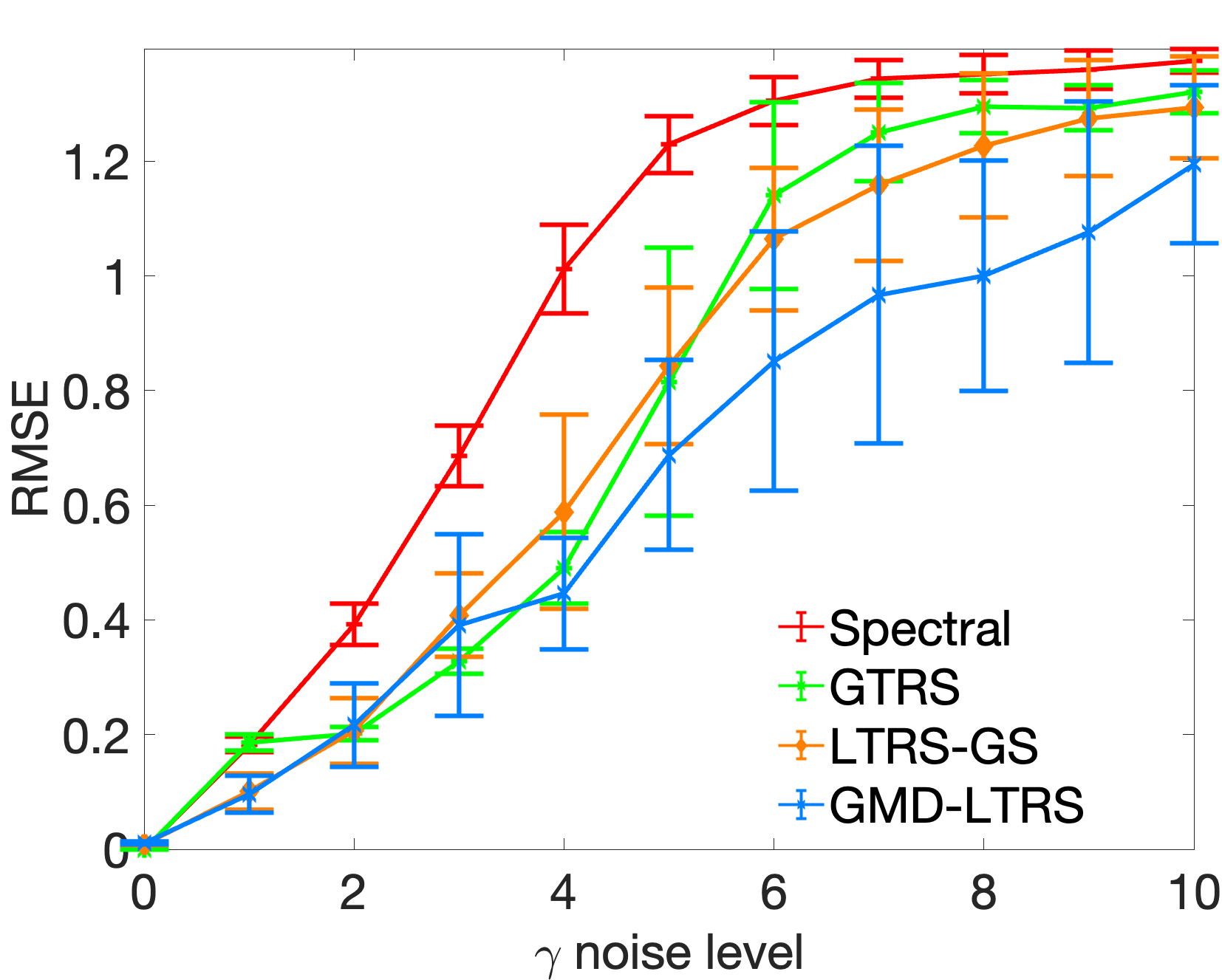}
  \caption{$S_T = 1$}
  %\label{fig:transync_alpha_ST_0}
\end{subfigure}%
\hfill
\begin{subfigure}{.5\textwidth}
  \centering
  \includegraphics[width=\linewidth]{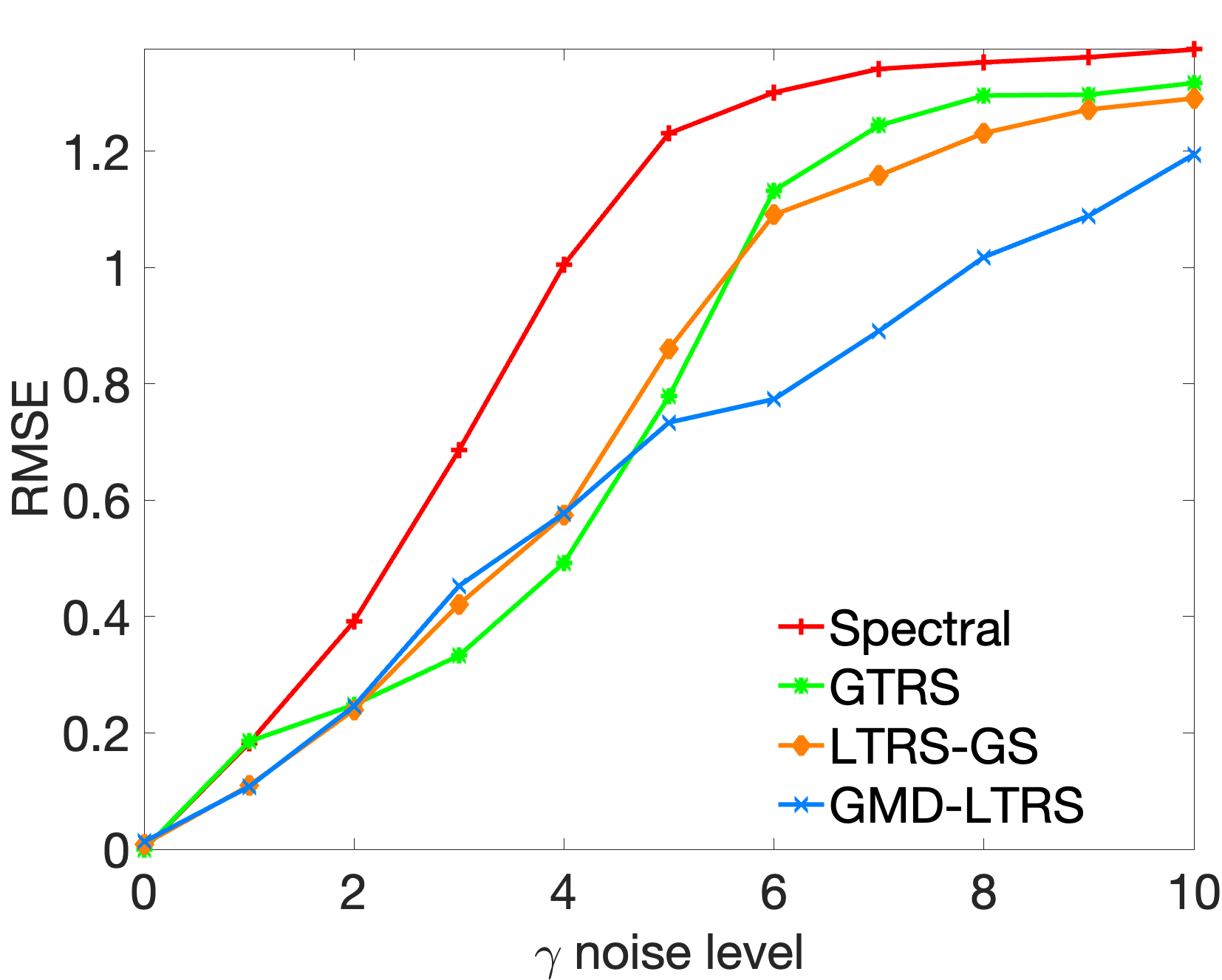}
  \caption{$S_T = T^{1/4}$}
  %\label{fig:transync_alpha_ST_m05}
\end{subfigure}%
\begin{subfigure}{.5\textwidth}
  \centering
  \includegraphics[width=\linewidth]{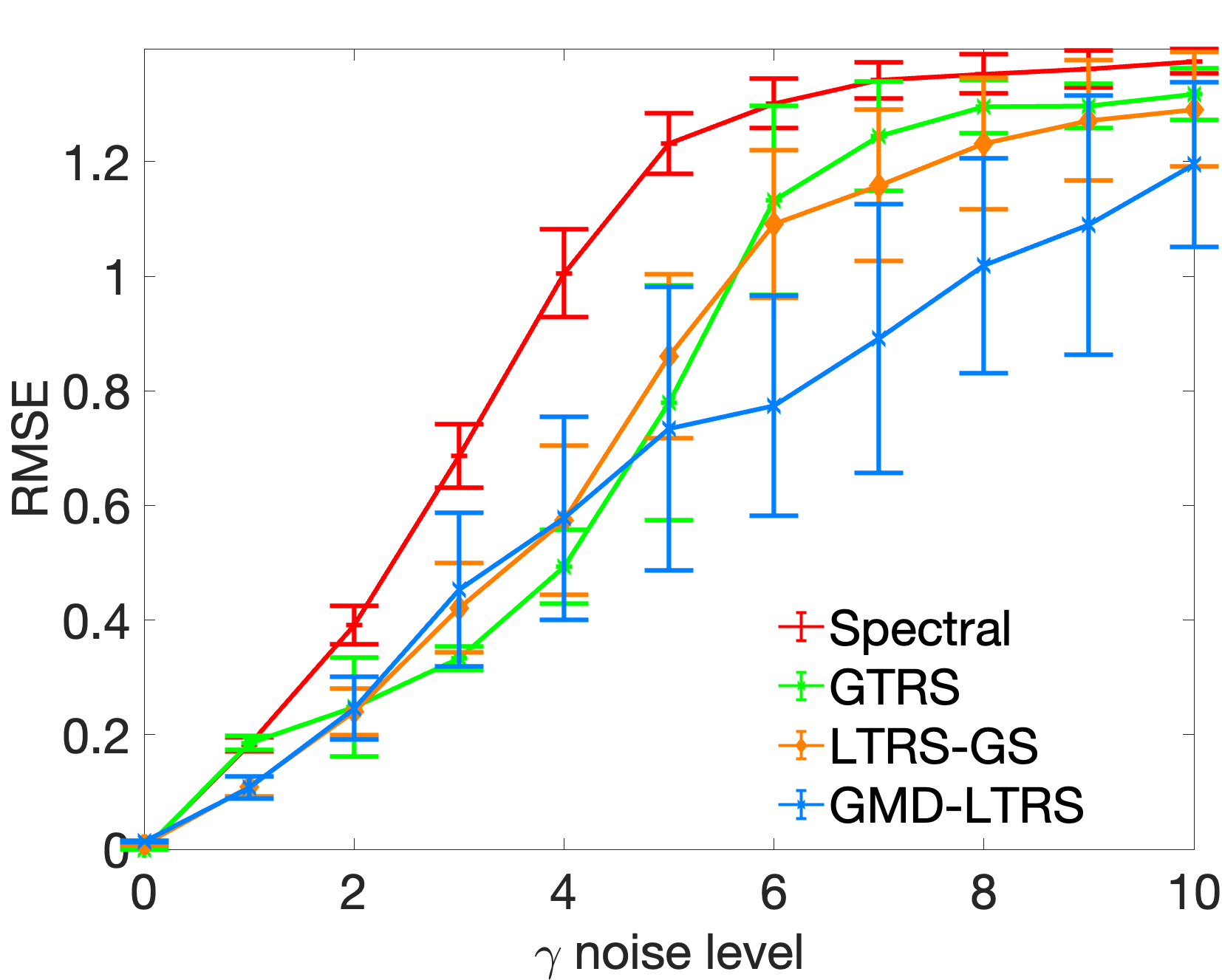}
  \caption{$S_T = T^{1/4}$}
  %\label{fig:transync_alpha_ST_0}
\end{subfigure}%
\caption{RMSE versus $\gamma (= \sigma)$ for AGN model ($n=30$, $T = 20$) with $S_T \in \set{1/T, 1, T^{1/4}}$, results averaged over $20$ MC runs. We choose $\lambda_{\text{scale}} = 10$ for $\globaltrs$. Results on left show average RMSE,  the right panel shows average RMSE $\pm$ standard deviation.}
\label{fig:plots_123_rmse_noise_wigner_T20_20runs}
\end{figure}

%
%----------------------------------------------------------------
% Figures for RMSE versus noise level for Outliers model (PLOTS_{1,2,3})
%----------------------------------------------------------------
 \begin{figure}[!htp]
\centering
\begin{subfigure}{.5\textwidth}
  \centering
  \includegraphics[width=\linewidth]{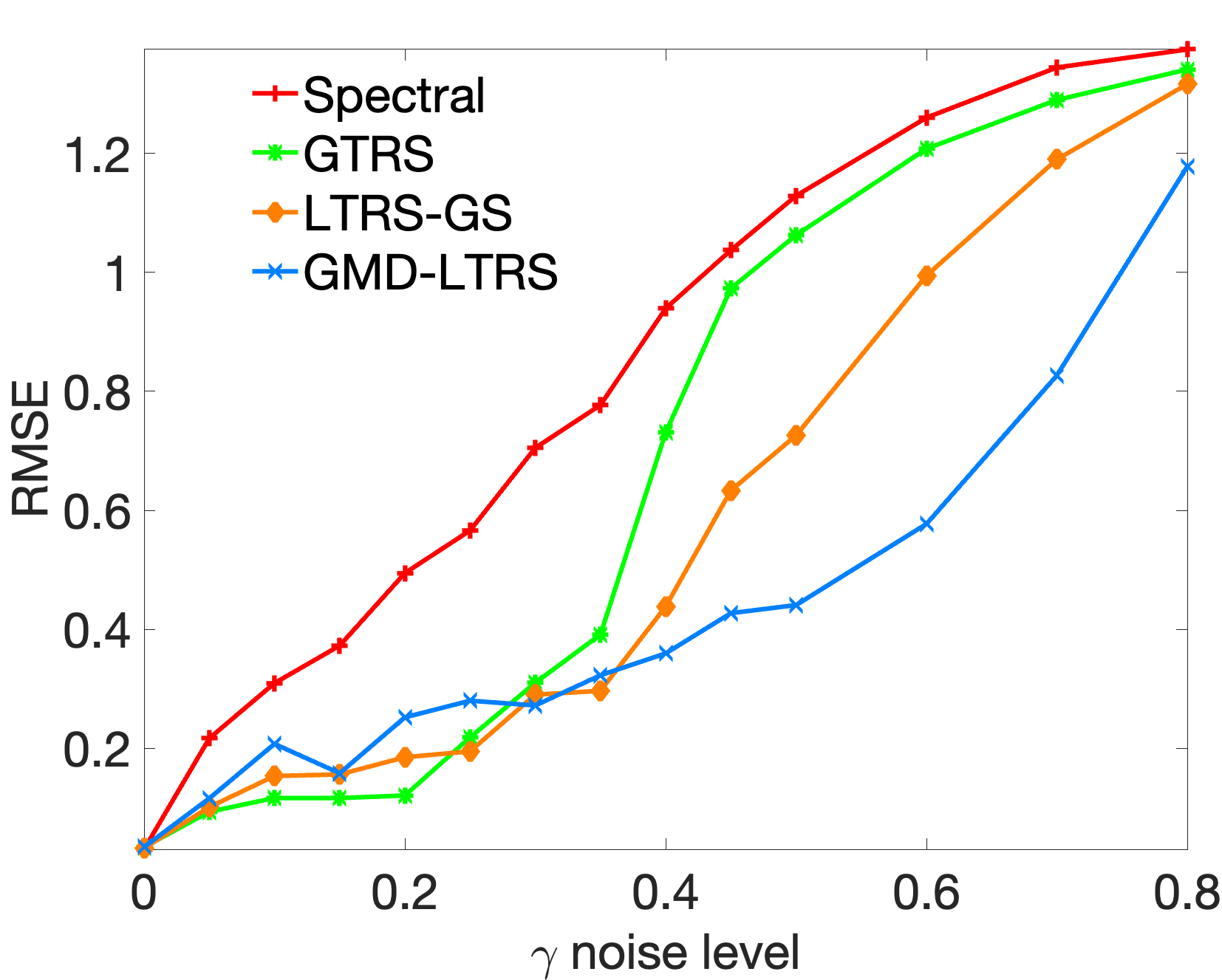}
  \caption{$S_T = 1/T$}
  %\label{fig:transync_alpha_ST_1}
\end{subfigure}%
\begin{subfigure}{.5\textwidth}
  \centering
  \includegraphics[width=\linewidth]{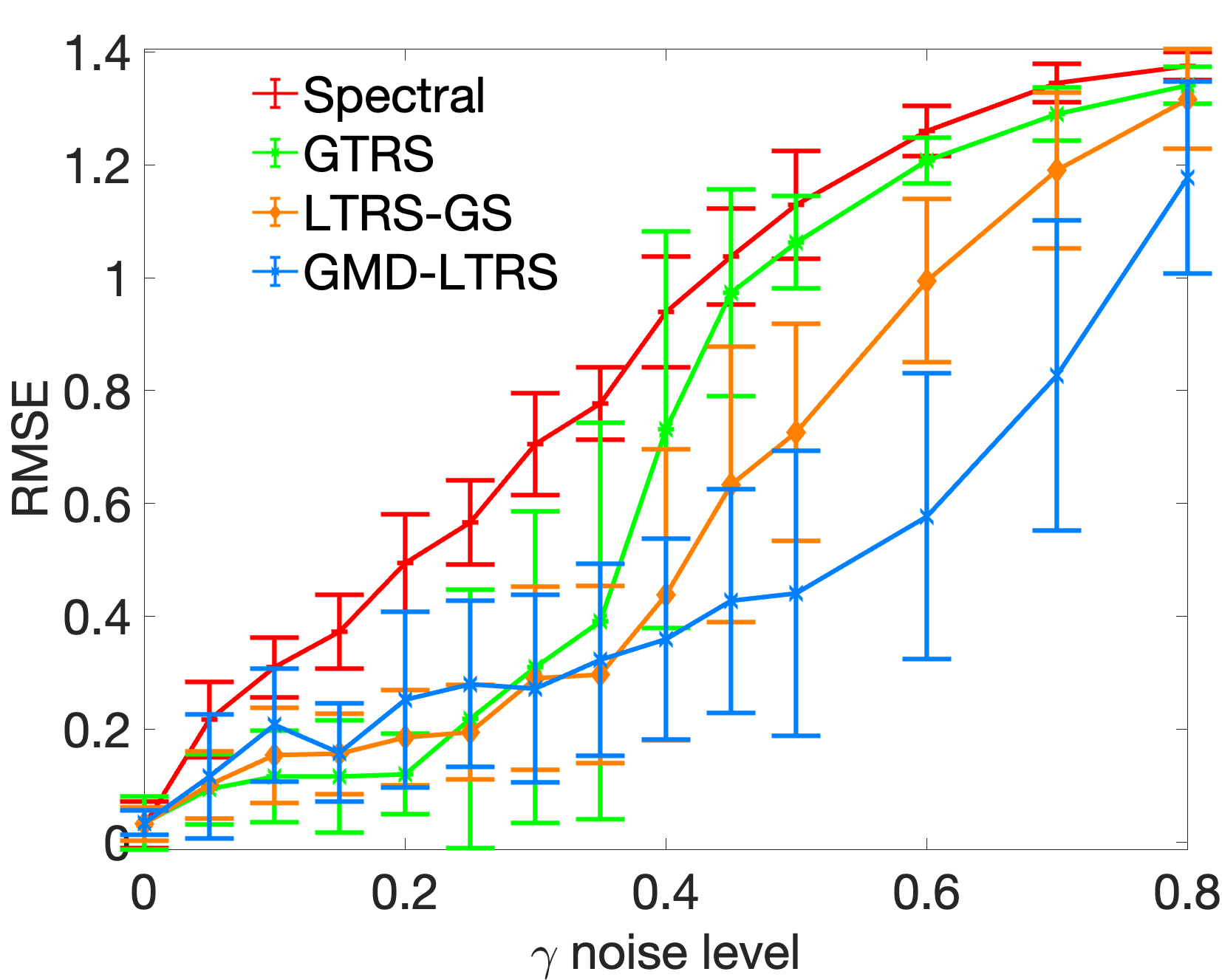}
  \caption{$S_T = 1/T$}
  %\label{fig:transync_alpha_ST_05}
\end{subfigure}%
%\caption{}
%\label{fig:plots_3_rmse_T_wigner_sigma_3_20runs}
\hfill
\begin{subfigure}{.5\textwidth}
  \centering
  \includegraphics[width=\linewidth]{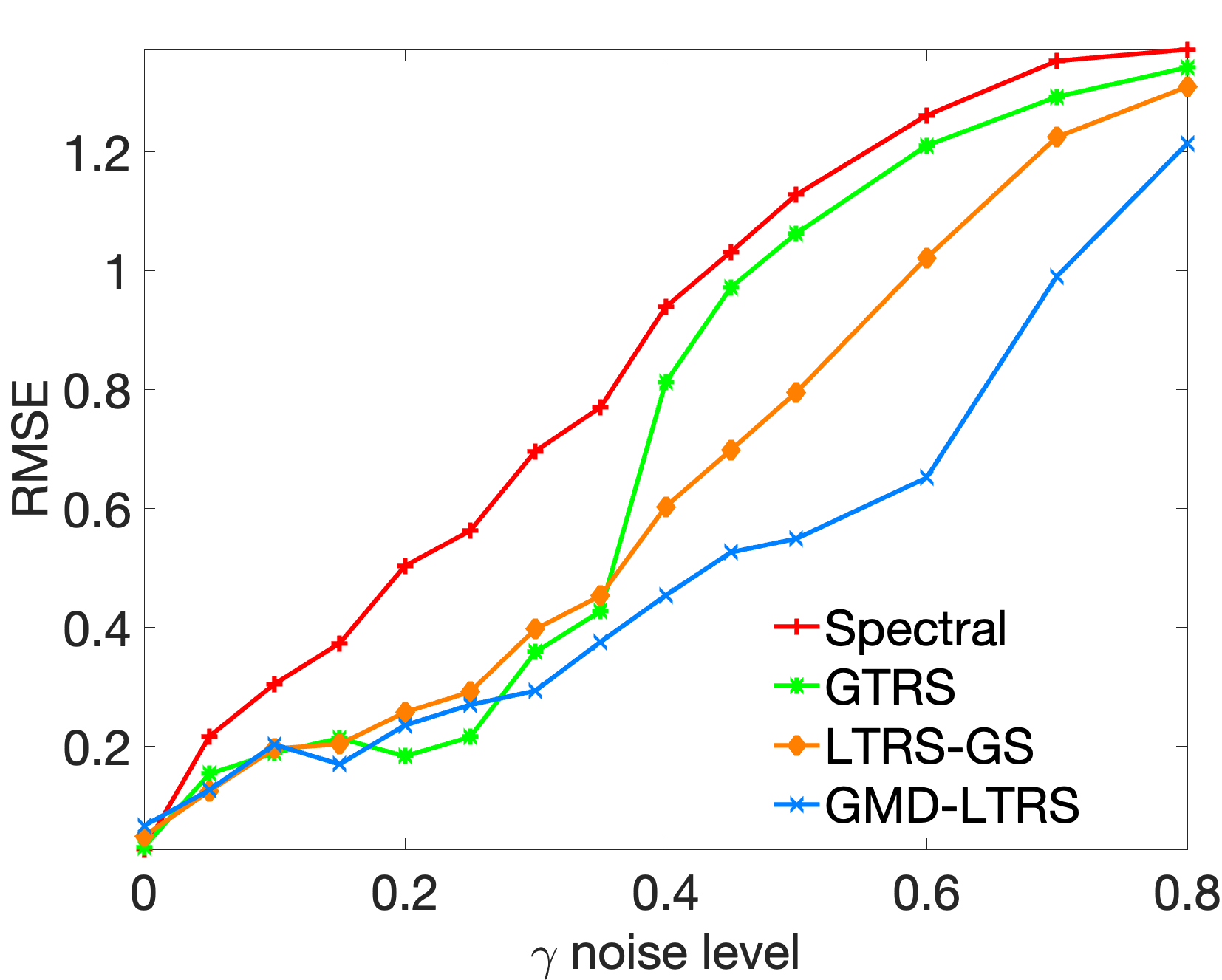}
  \caption{$S_T = 1$}
  %\label{fig:transync_alpha_ST_m05}
\end{subfigure}%
\begin{subfigure}{.5\textwidth}
  \centering
  \includegraphics[width=\linewidth]{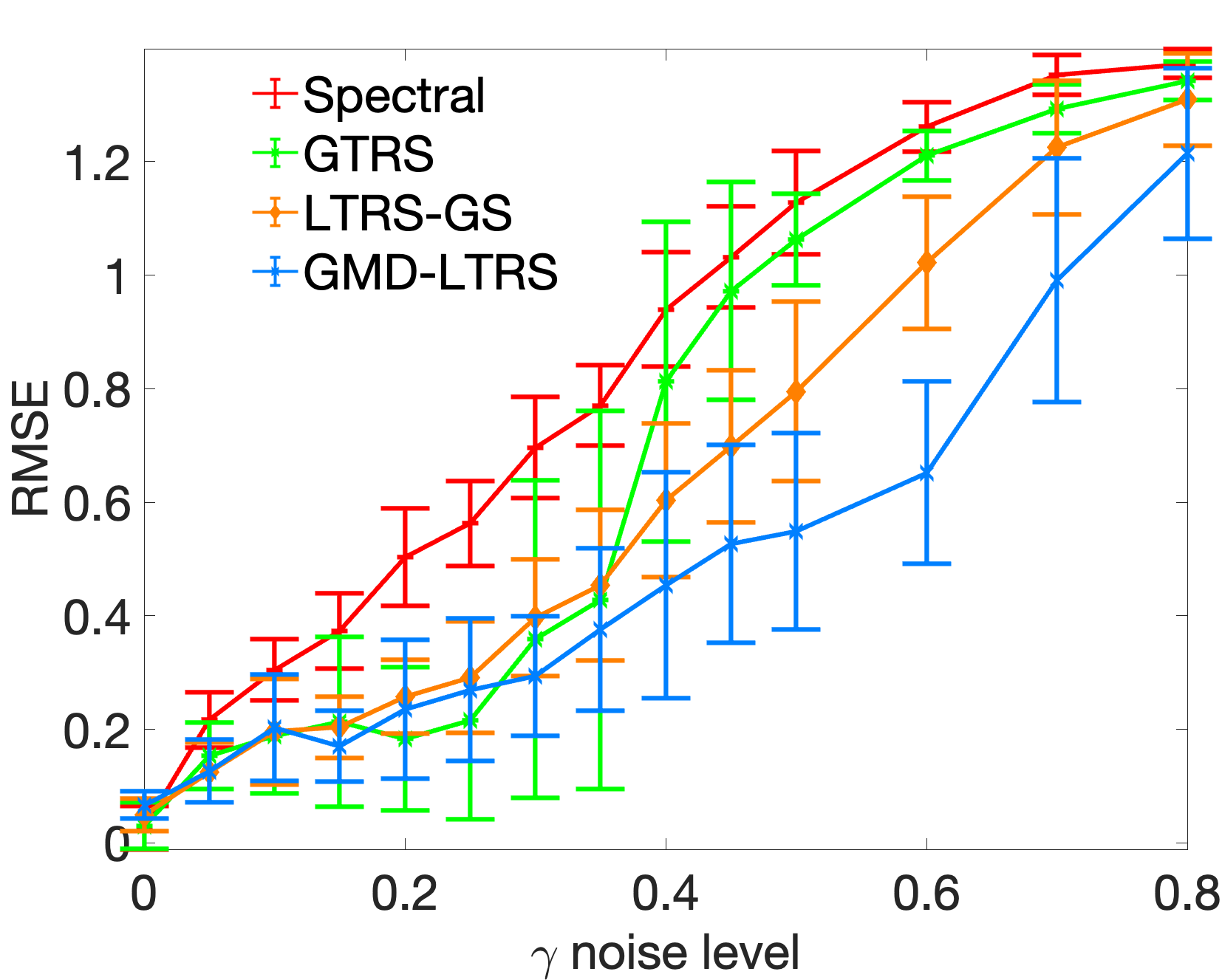}
  \caption{$S_T = 1$}
  %\label{fig:transync_alpha_ST_0}
\end{subfigure}%
\hfill
\begin{subfigure}{.5\textwidth}
  \centering
  \includegraphics[width=\linewidth]{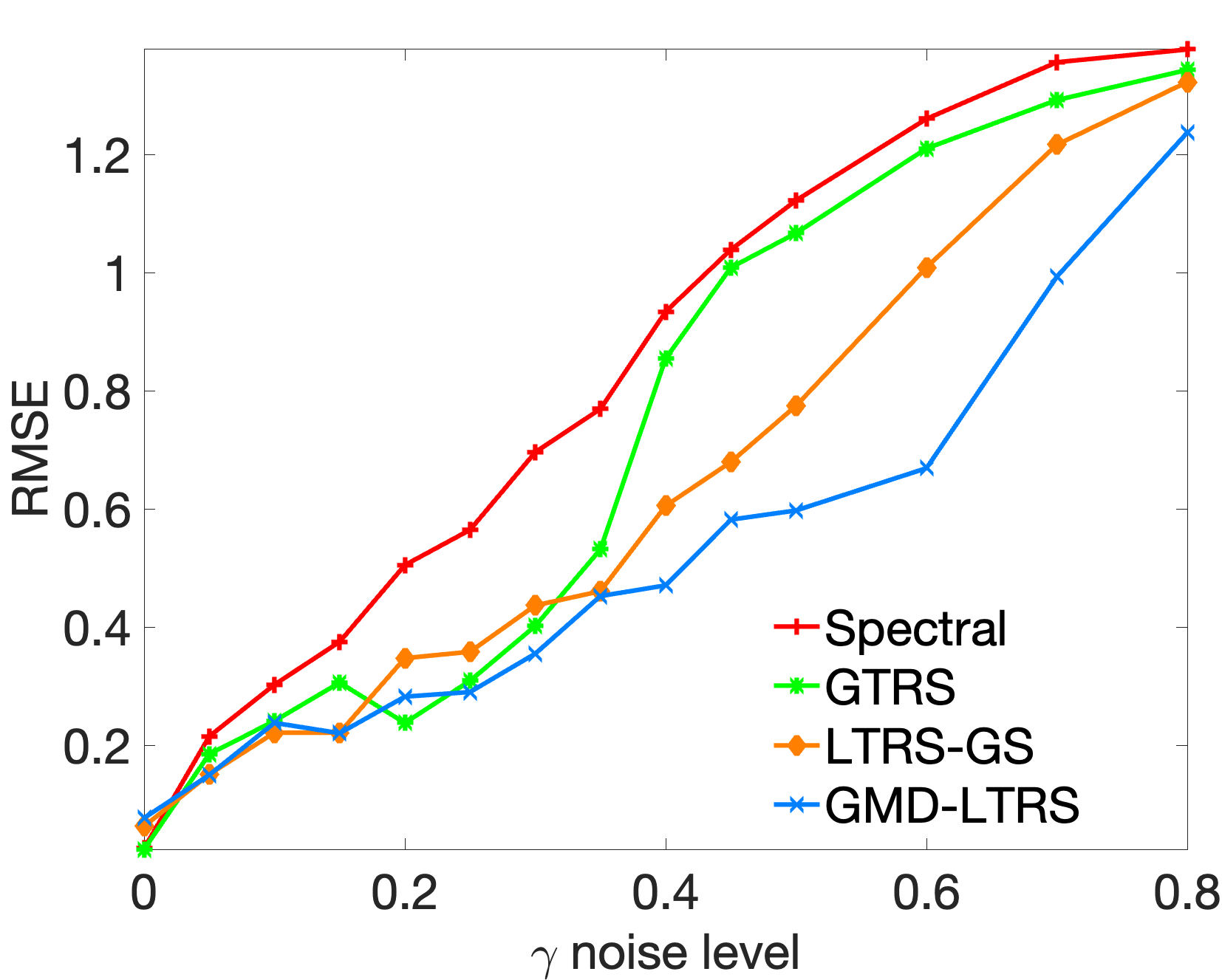}
  \caption{$S_T = T^{1/4}$}
  %\label{fig:transync_alpha_ST_m05}
\end{subfigure}%
\begin{subfigure}{.5\textwidth}
  \centering
  \includegraphics[width=\linewidth]{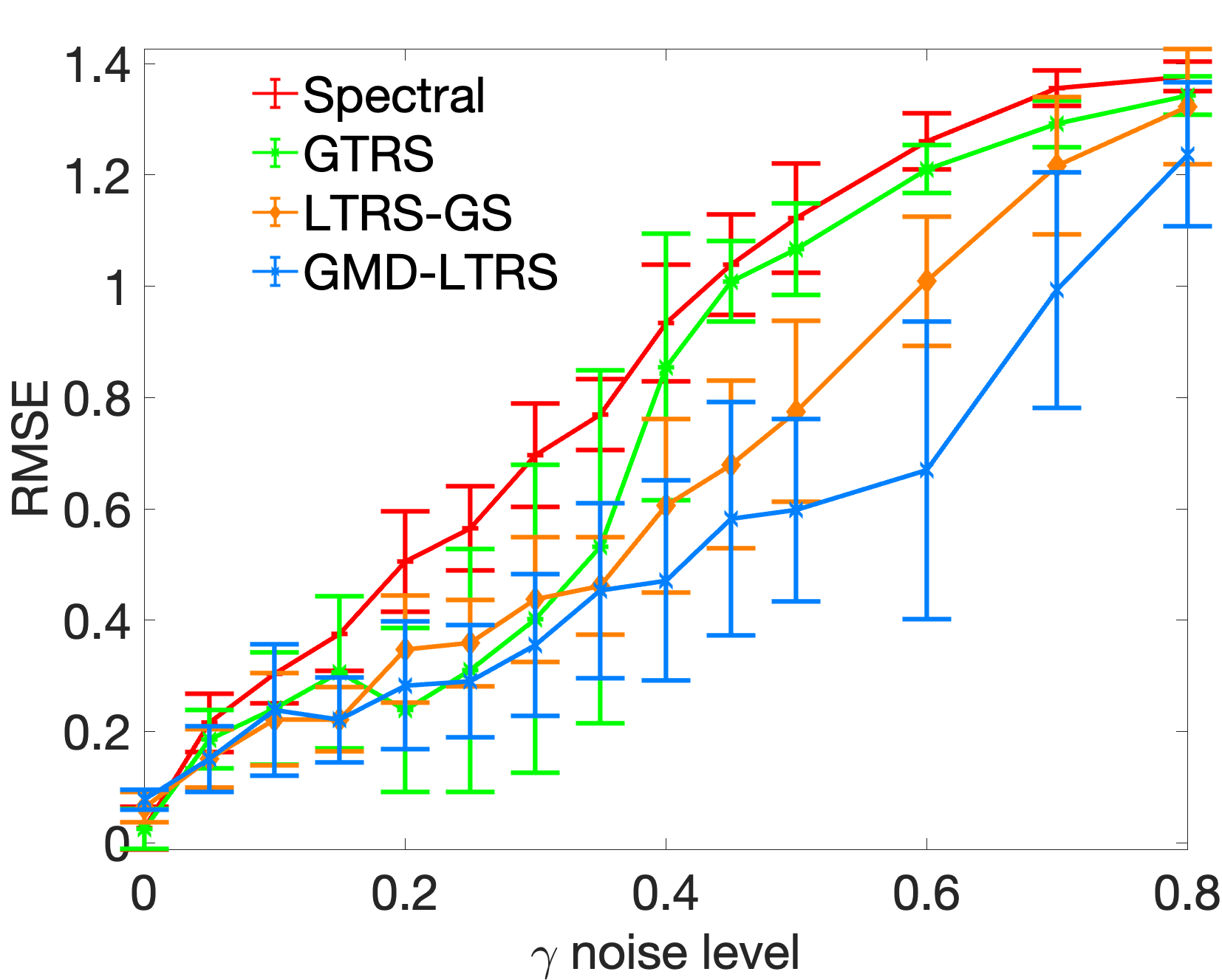}
  \caption{$S_T = T^{1/4}$}
  %\label{fig:transync_alpha_ST_0}
\end{subfigure}%
\caption{RMSE versus $\gamma (= \eta)$ for Outliers model ($n=30$, $p = 0.2$, $T = 20$) with $S_T \in \set{1/T, 1, T^{1/4}}$, results averaged over $20$ MC runs. We choose $\lambda_{\text{scale}} = 10$ for $\globaltrs$. Results on left show average RMSE,  the right panel shows average RMSE $\pm$ standard deviation.}
\label{fig:plots_123_rmse_noise_outliers_p0p2_T20_20runs}
\end{figure}

%
% Conclusion
%
\section{Concluding remarks}\label{sec:conclusion}
In this work we introduced the problem of dynamic angular synchronization, where the underlying latent angles, and the measurement graphs, both evolve with time. Assuming a smoothness condition on the evolution of the angles, with the smoothness measured by the quadratic variation of the angles w.r.t the path graph on $T$ vertices, we proposed three algorithms (namely, $\globaltrs$, $\localtrs$ and $\matdenoising$) for jointly estimating the angles at all time points. For Algorithm $\matdenoising$, we derived non-asymptotic error bounds for the AGN and Outliers noise models, and established conditions under which the MSE goes to zero as $T$ increases. In particular, this was shown to occur under much milder conditions than the static setup $(T = 1)$. It includes the setting where the measurement graphs are highly sparse and disconnected, and also when the measurement noise is large and can potentially increase with $T$. We validated our theoretical results with experiments on synthetic data which show that $\matdenoising$ typically outperforms the other two methods.

An interesting direction for future work would be to extend our framework to other groups such as $\text{SO}(d)$. Another direction would be to consider evolution models where the measurements are statistically dependent over time -- our analysis requires the measurements to be independent across time points.  

\clearpage
\newpage
\bibliographystyle{plain}
\bibliography{refs}

@article{unifiedOd_JACHA,
title = {A unified approach to synchronization problems over subgroups of the orthogonal group},
journal = {Applied and Computational Harmonic Analysis},
volume = {66},
pages = {320-372},
year = {2023},
author = {Huikang Liu and Man-Chung Yue and Anthony Man-Cho So}
}

@article{GNN_Sync,
      title={{Robust Angular Synchronization Using Directed Graph Neural Networks}}, 
      author={He, Yixuan and Reinert, Gesine and Wipf, David and   Cucuringu, Mihai},
      journal = {to appear at ICLR 2024; arXiv:2310.05842},
      year={2024}
}

@article{Z2synchronization,
  title={{Synchronization over $Z_2$ and community detection in signed multiplex networks with constraints}},
  author={Cucuringu, Mihai},
  journal={Journal of Complex Networks},
  volume={3},
  number={3},
  pages={469--506},
  year={2015},
  publisher={Oxford University Press}
}

@article{huang2017translation,
  title={Translation synchronization via truncated least squares},
  author={Huang, Xiangru and Liang, Zhenxiao and Bajaj, Chandrajit and Huang, Qixing},
  journal={Advances in Neural Information Processing Systems},
  volume={30},
  year={2017}
}

@article{AKT_dynamicRankRSync,
    author = {Araya, Ernesto and Karlé, Eglantine and Tyagi, Hemant},
    title = "{Dynamic ranking and translation synchronization}",
    journal = {Information and Inference: A Journal of the IMA},
    volume = {12},
    number = {3},
    pages = {2224-2266},
    year = {2023}
}

@article{KT_JMLR_BTL_NN,
  author  = {Eglantine Karlé and Hemant Tyagi},
  title   = {{Dynamic Ranking with the BTL Model: A Nearest Neighbor based Rank Centrality Method}},
  journal = {Journal of Machine Learning Research},
  year    = {2023},
  volume  = {24},
  number  = {269},
  pages   = {1--57}
}

@article{dyn_signed_comm20,
author = {Chen, Jianrui and Liu, Danwei and Hao, Fei and Wang, Hua},
year = {2020},
pages = {891–900},
title = {Community detection in dynamic signed network: an intimacy evolutionary clustering algorithm},
volume = {11},
journal = {Journal of Ambient Intelligence and Humanized Computing} 
}

@InProceedings{dynamicBTL_2020,
  title = 	 {{Nonparametric Estimation in the Dynamic Bradley-Terry Model}},
  author =       {Bong, Heejong and Li, Wanshan and Shrotriya, Shamindra and Rinaldo, Alessandro},
  booktitle = 	 {Proceedings of the Twenty Third International Conference on Artificial Intelligence and Statistics},
  pages = 	 {3317--3326},
  year = 	 {2020}
}

@article{survery_structure_motion, 
title={A survey of structure from motion.}, 
volume={26}, 
journal={Acta Numerica}, 
author={Özyeşil, Onur and Voroninski, Vladislav and Basri, Ronen and Singer, Amit}, 
year={2017},
pages={305–364}
}

@article{temporal_image_reg,
    author = {Dsilva, Carmeline J. and Lim, Bomyi and Lu, Hang and Singer, Amit and Kevrekidis, Ioannis G. and Shvartsman, Stanislav Y.},
    title = "{Temporal ordering and registration of images in studies of developmental dynamics}",
    journal = {Development},
    volume = {142},
    number = {9},
    pages = {1717-1724},
    year = {2015}
}

@article{shkolnisky2012viewing,
  title={{Viewing direction estimation in cryo-EM using synchronization}},
  author={Shkolnisky, Yoel and Singer, Amit},
  journal={SIAM journal on imaging sciences},
  volume={5},
  number={3},
  pages={1088--1110},
  year={2012},
  publisher={SIAM}
}

@INPROCEEDINGS{structFromMotion_Amit,
  author={M. {Arie-Nachimson} and S. Z. {Kovalsky} and I. {Kemelmacher-Shlizerman} and A. {Singer} and R. {Basri}},
  booktitle={2012 Second International Conference on 3D Imaging, Modeling, Processing, Visualization   Transmission}, 
  title={Global Motion Estimation from Point Matches}, 
  year={2012},
  volume={},
  number={},
  pages={81-88}
}

@article{bandeira2018notes,
      title={Notes on computational-to-statistical gaps: predictions using statistical physics}, 
      author={Afonso S. Bandeira and Amelia Perry and Alexander S. Wein},
      year={2018},
      journal={arXiv:1803.11132}
}

@article{perry2018message,
  title={{Message-Passing Algorithms for Synchronization Problems over Compact Groups}},
  author={Perry, Amelia and Wein, Alexander S. and Bandeira, Afonso S. and Moitra, Ankur},
  journal={Communications on Pure and Applied Mathematics},
  volume={71},
  number={11},
  pages={2275--2322},
  year={2018},
  publisher={Wiley Online Library}
}

@article{bandeira2017tightness,
  title={Tightness of the maximum likelihood semidefinite relaxation for angular synchronization},
  author={Bandeira, Afonso S. and Boumal, Nicolas and Singer, Amit},
  journal={Mathematical Programming},
  volume={163},
  number={1-2},
  pages={145--167},
  year={2017},
  publisher={Springer}
}

@article{negahban2015rank,
  title={Rank centrality: Ranking from pairwise comparisons},
  author={Negahban, Sahand and Oh, Sewoong and Shah, Devavrat},
  journal={Operations Research},
  volume={65},
  number={1},
  pages={266--287},
  year={2017},
  publisher={INFORMS}
}

@article{li2021recovery,
  title={Recovery Guarantees for Time-varying Pairwise Comparison Matrices with Non-transitivity},
  author={Li, Shuang and Wakin, Michael B},
  journal={arXiv:2106.09151},
  year={2021}
}

@article{zhong2018near,
  title={Near-optimal bounds for phase synchronization},
  author={Zhong, Yiqiao and Boumal, Nicolas},
  journal={SIAM Journal on Optimization},
  volume={28},
  number={2},
  pages={989--1016},
  year={2018},
  publisher={SIAM}
}

@article{SVDRank,
  author = {d'Aspremont, Alexandre and Cucuringu, Mihai and Tyagi, Hemant},
  title   = {{Ranking and synchronization from pairwise measurements via SVD}},
  journal = {Journal of Machine Learning Research (JMLR)},
  year    = {2021},
  volume  = {22},
  number  = {19},
  pages   = {1-63} 
}

@article{daviskahan,
author = {Chandler Davis and W. M. Kahan},
title = {The Rotation of Eigenvectors by a Perturbation. III},
journal = {SIAM Journal on Numerical Analysis},
volume = {7},
number = {1},
pages = {1-46},
year = {1970}
}

@article{bandeira2016,
author = "Bandeira, Afonso S. and van Handel, Ramon",
journal = "Ann. Probab.",
month = "07",
number = "4",
pages = "2479--2506",
title = "Sharp nonasymptotic bounds on the norm of random matrices with independent entries",
volume = "44",
year = "2016"
}

@article{syncRank,
author = "Mihai Cucuringu",
title = "{{Sync-Rank: Robust Ranking, Constrained Ranking and Rank Aggregation via Eigenvector and Semidefinite Programming Synchronization}}",
journal = "IEEE Transactions on Network Science and Engineering",
volume = "3",
number = "1",
pages = "58--79",
year = "2016"
}

@article{asap2d,
 author = {Cucuringu, Mihai and Lipman, Yaron and Singer, Amit},
 title = {Sensor network localization by eigenvector synchronization over the {E}uclidean group},
 journal = {ACM Trans. Sen. Netw.},
 volume = {8},
 number = {3},
 year = {2012},
 pages = {19:1--19:42},
 articleno = {19}
}

@inproceedings{timeoneSync,
author = "A. Giridhar and P. R. Kumar",
title = "Distributed clock synchronization over wireless networks: Algorithms and
analysis",
booktitle = "45th IEEE Conference on Decision and Control",
year = 2006,
pages = "4915--4920"}

@article{asap3d,
author = {Cucuringu, Mihai and Singer, Amit and Cowburn, David},
title = {Eigenvector synchronization, graph rigidity and the molecule problem},
volume = {1},
number = {1},
pages = {21-67},
year = {2012},
journal = {Information and Inference}
}

@article{sync,
author = "Amit Singer",
title = "Angular synchronization by eigenvectors and semidefinite programming",
journal = "Appl. Comput. Harmon. Anal.",
volume = "30",
number= "1",
pages = "20--36",
year = "2011"}

@article{qp_nphard2006,
author = "Shuzhong Zhang AND Yongwei Huan",
title = "Complex quadratic optimization and semidefinite programming",
journal ={SIAM Journal on Optimization},
volume ={16},
number ={3},
year ={2006},
pages ={871--890}
}

@inproceedings{SadhanalaTV16,
author = {Sadhanala, Veeranjaneyulu and Wang, Yu-Xiang and Tibshirani, Ryan J.},
title = {Total Variation Classes beyond 1d: Minimax Rates, and the Limitations of Linear Smoothers},
year = {2016},
pages = {3521–3529},
series = {NIPS'16}
}

@article{TRS_nakatsukasa,
author = {Adachi, Satoru and Iwata, Satoru and Nakatsukasa, Yuji and Takeda, Akiko},
title = {Solving the Trust-Region Subproblem By a Generalized Eigenvalue Problem},
journal = {SIAM Journal on Optimization},
volume = {27},
number = {1},
pages = {269-291},
year = {2017}
}

@article{Manchoso,
    author = {Wang, Peng and Zhou, Zirui and So, Anthony Man-Cho},
    title = {Non-convex exact community recovery in stochastic block model},
    journal = {Mathematical Programming},
    volume ={195},
    number = {1},
    pages={1--37},
    year = {2022}
}

@article{araya2023seeded,
  author  = {Ernesto Araya and Guillaume Braun and Hemant Tyagi},
  title   = {{Seeded Graph Matching for the Correlated Gaussian Wigner Model via the Projected Power Method}},
  journal = {Journal of Machine Learning Research},
  year    = {2024},
  volume  = {25},
  number  = {5},
  pages   = {1--43} 
}

@article{Chen2016ThePP,
  title={The Projected Power Method: An Efficient Algorithm for Joint Alignment from Pairwise Differences},
  author={Yuxin Chen and Emmanuel J. Cand{\`e}s},
  journal={Communications on Pure and Applied Mathematics},
  volume = {71},
  number = {8},
pages = {1648-1714},
  year={2016}
}

@article{ling2021generalized,
      title={{Generalized Orthogonal Procrustes Problem under Arbitrary Adversaries}}, 
      author={Shuyang Ling},
      year={2024},
      journal={arXiv:2106.15493},
}

@article{TRS_Hager,
author = {Hager, William W.},
title = {Minimizing a Quadratic Over a Sphere},
journal = {SIAM Journal on Optimization},
volume = {12},
number = {1},
pages = {188-208},
year = {2001}
}

@article{rudelson_vershynin,
author = {Mark Rudelson and Roman Vershynin},
title = {{Hanson-Wright inequality and sub-gaussian concentration}},
volume = {18},
journal = {Electronic Communications in Probability},
number = {},
pages = {1 -- 9},
year = {2013}
}

@article{boumal2016,
author = {Boumal, Nicolas},
title = {Nonconvex Phase Synchronization},
volume = {26},
number = {4},
pages = {2355-2377},
year = {2016},
journal = {SIAM Journal on Optimization}
}

@article{GaoZhang,
    author = {Gao, Chao and Zhang, Anderson Y},
    title = "{Optimal orthogonal group synchronization and rotation group synchronization}",
    journal = {Information and Inference: A Journal of the IMA},
    volume={12},
    year = {2023},
   pages={591–632}
}

@inbook{vershynin_2012, 
title={Introduction to the non-asymptotic analysis of random matrices}, booktitle={Compressed Sensing: Theory and Applications}, 
publisher={Cambridge University Press}, 
author={Vershynin, Roman}, 
year={2012}, 
pages={210–268}
}

@book{brouwer12,
  address = {New York, NY},
  author = {Brouwer, A.E. and Haemers, W.H.},
  title = {Spectra of Graphs},
  year = 2012
}

@article{Ostrovsky2014ExactVF,
  title={Exact value for subgaussian norm of centered indicator random variable},
  author={E. Ostrovsky and L. Sirota},
  journal={arXiv:1405.6749},
  year={2014}
}

@article{Liu2017,
author = {Liu, H. and Yue, Man-Chung and Man-Cho So, A.},
title = {On the Estimation Performance and Convergence Rate of the Generalized Power Method for Phase Synchronization},
journal = {SIAM Journal on Optimization},
volume = {27},
number = {4},
pages = {2426-2446},
year = {2017}
}

@article{jabin2015continuous,
  title={A continuous model for ratings},
  author={Jabin, Pierre-Emmanuel and Junca, St{\'e}phane},
  journal={SIAM Journal on Applied Mathematics},
  volume={75},
  number={2},
  pages={420--442},
  year={2015},
  publisher={SIAM}
}

@article{fahrmeir1994dynamic,
  title={Dynamic stochastic models for time-dependent ordered paired comparison systems},
  author={Fahrmeir, Ludwig and Tutz, Gerhard},
  journal={Journal of the American Statistical Association},
  volume={89},
  number={428},
  pages={1438--1449},
  year={1994},
  publisher={Taylor \& Francis}
}

@article{glickman1998state,
  title={A state-space model for National Football League scores},
  author={Glickman, Mark E. and Stern, Hal S.},
  journal={Journal of the American Statistical Association},
  volume={93},
  number={441},
  pages={25--35},
  year={1998},
  publisher={Taylor \& Francis}
}

@inproceedings{maystre2019pairwise,
  title={Pairwise comparisons with flexible time-dynamics},
  author={Maystre, Lucas and Kristof, Victor and Grossglauser, Matthias},
  booktitle={Proceedings of the 25th ACM SIGKDD International Conference on Knowledge Discovery \& Data Mining},
  pages={1236--1246},
  year={2019}
}

@article{lopez2018often,
  title={{How often does the best team win? A unified approach to understanding randomness in North American sport}},
  author={Lopez, Michael J. and Matthews, Gregory J. and Baumer, Benjamin S.},
  journal={The Annals of Applied Statistics},
  volume={12},
  number={4},
  pages={2483--2516},
  year={2018},
  publisher={Institute of Mathematical Statistics}
}

@article{cattelan2013dynamic,
  title={Dynamic {B}radley--{T}erry modelling of sports tournaments},
  author={Cattelan, Manuela and Varin, Cristiano and Firth, David},
  journal={Journal of the Royal Statistical Society: Series C (Applied Statistics)},
  volume={62},
  number={1},
  pages={135--150},
  year={2013} 
}

%\clearpage 

\appendix 

%
% Proof of Theorem \ref{thm:error_wigner}
%
\section{Proof of Theorem \ref{thm:error_wigner}}\label{app:proof_thm_wigner}
The bound in \eqref{eq:error_bound_Wigner} follows by plugging $\tau^*$ in \eqref{eq:error_proj_wigner}, and using the bound  $\tau^* \asymp \min\set{1 + \Big(\frac{T^2 S_T}{n \sigma^2}\Big)^{1/3}, T}$ along with the fact $(a+b)^2 \geq a^2 + b^2$ for $a,b \geq 0$.

To show the second part,  observe that if $\Big(\frac{T^2 S_T}{n \sigma^2}\Big)^{1/3} \geq 1$ then \eqref{eq:error_bound_Wigner} simplifies to
\begin{align*}
        \|\est G-\Gshift^*\|^2_F
        \lesssim  
        \max\set{n^{5/3} \sigma^{4/3} T^{2/3} S_T^{1/3}, n S_T} \indic_{\set{\tau^* < T}}   
        +  \min\set{\sigma^{4/3} T^{2/3} S_T^{1/3} n^{5/3}, T}  
+ \sigma^2 \log\paren{\frac4\delta} 
\end{align*}
and also $\tau^* \asymp \min\set{\Big(\frac{T^2 S_T}{n \sigma^2}\Big)^{1/3}, T}$. Now the bound in \eqref{eq:error_bound_wigner_simp} follows readily after minor simplifications provided $T$ satisfies \eqref{eq:T_cond_wigner}. We just remark that  
$T \gtrsim \frac{S_T}{n \sigma^2}$ implies $\tau^* \asymp \Big(\frac{T^2 S_T}{n \sigma^2}\Big)^{1/3}$, and in particular $\tau^* < T$.

%-----------------------------------------------
% Proof of Lemma \ref{lem:variance_term_outliers}
%
\section{Proof of Lemma \ref{lem:variance_term_outliers}}\label{app:proof_lemma_var_outliers}
Let $H(k) \in \set{0,1}^{n \times n}$ denote the (symmetric) adjacency matrix of $\calH_k$. Hence $(H_{ij}(k))_{i<j}$ are independent Bernoulli random variables with parameter $p(k)$, and $H_{ii}(k) = 0$ for all $i$. Also denote $Z(k) \in \matC^{n \times n}$ to be a Hermitian matrix with $Z_{ii}(k) = 0$ for each $i \in [n]$, and for each $i < j$ 
\begin{equation*}
    Z_{ij}(k) =     
    \begin{cases}
        g^*_i(k)\conj{g^*_j(k)} \text{ with probability } 1-\eta, \\
        e^{\iota \varphi_{ij}(k)} \enskip \enskip\quad\text{ with probability }\eta
    \end{cases}
\end{equation*}
where $(\varphi_{ij}(k))_{i<j} \stackrel{ind.}{\sim} \text{Unif}[0,2\pi)$. Then we can write $A(k) = Z(k) \circ H(k)$, and $A = Z \circ H$ where $Z \in \matC^{nT \times n}$ (resp. $H \in \set{0,1}^{nT \times n}$) is formed by stacking $Z(1),\dots Z(T)$ (resp. $H(1),\dots, H(T)$). Then we can write $R = Z \circ H - D \Gshift^*$. Now decompose $Z = Z_1 + Z_2$ where $Z_1(k)$ (resp. $Z_2(k)$) contains the upper triangular (resp. lower triangular) part of $Z(k)$. Write $H = H_1 + H_2$ and $\Gshift^* = \Gshift_1^* + \Gshift_2^*$ in an analogous manner. Then $R$ decomposes as 
\begin{equation}
    R = \underbrace{(Z_1 \circ H_1 - D \Gshift_1^*)}_{=: R_1} + \underbrace{(Z_2 \circ H_2 - D \Gshift_2^*)}_{=: R_2}
\end{equation}
with the entries of $R_1$ (resp. $R_2$) being independent, centered subgaussian random variables.

Now, $\|\projtaun R\|^2_F \leq 2(\|\projtaun R_1\|^2_F + \|\projtaun R_2\|^2_F)$ and it is sufficient to bound $\|\projtaun R_1\|^2_F$, since the same bound will follow for $\|\projtaun R_2\|^2_F$. Decomposing $R_1$ into its real and complex parts
\[
R_1= 
\underbrace{\real(R_1)}_{=:R_{1,\opR}}+\iota\underbrace{\imag(R_1)}_{=:R_{1,\opI}},
\]
we have $\|\projtaun R_1\|^2_F = \|\projtaun R_{1,\opR}\|^2_F + \|\projtaun R_{1,\opI}\|^2_F$. Again, we focus only on $\|\projtaun R_{1,\opR}\|^2_F$ as the same bound follows for $\|\projtaun R_{1,\opI}\|^2_F$ in an analogous manner.

Denote $u \in \matR^{n^2 T}$ to be the vector formed by stacking the columns of $R_{1,\opR} \in \matR^{nT \times n}$. Then we obtain 
\begin{equation*}
    \|\projtaun R_{1,\opR}\|^2_F = u^\top (I_n \otimes \projtaun) u
\end{equation*}
which we will bound via the Hanson-Wright inequality \cite[Thm. 1.1]{rudelson_vershynin}. To this end, we first derive an upper bound on $\expec[u^\top (I_n \otimes \projtaun) u]$ as follows. Denoting $Z_{1,\opR}(k)$ (resp. $\Gshift_{1,\opR}^* (k)$) to be the real part of $Z_{1}(k)$ (resp. $\Gshift_{1}^* (k)$), we obtain 
\begin{equation*}
    (R_{1,\opR}(k))_{ij} = (Z_{1,\opR}(k))_{ij} (H_{1}(k))_{ij} - (1-\eta)p(k)(\Gshift_{1,\opR}^* (k))_{ij}. 
\end{equation*}
Hence we can bound $\expec[(R_{1,\opR}(k))^2_{ij}]$ for all $i \neq j \in [n]$ as follows.
\begin{align*}
    \expec[(R_{1,\opR}(k))^2_{ij}] &= p(k) \expec[(Z_{1,\opR}(k))_{ij} - (1-\eta)p(k)(\Gshift_{1,\opR}^* (k))_{ij}]^2 + (1-p(k))p^2(k) (1-\eta)^2 (\Gshift_{1,\opR}^* (k))^2_{ij} \\
    &\lesssim p(k) (\text{Var}((Z_{1,\opR}(k))_{ij}) + (1-\eta)^2(1-p(k))^2) + (1-p(k))p^2(k) (1-\eta)^2.
\end{align*}
It is not difficult to verify that $\text{Var}((Z_{1,\opR}(k))_{ij}) \lesssim \eta$, which upon plugging in the above bound readily implies 
\begin{equation*}
   \expec[(R_{1,\opR}(k))^2_{ij}] \lesssim \underbrace{\pmax\eta  + (1-\eta)^2 V + \pmax V (1-\eta)}_{=: f(\eta,\pmax,V)}.
\end{equation*}
Thus we have shown that 
$$\expec[u^\top (I_n \otimes \projtaun) u] \leq C n^2 \tau f(\eta,\pmax,V)$$ 
for some constant $C > 0$. 

We next bound $\norm{(R_{1,\opR}(k))_{ij}}_{\psi_2}$ for all $i \neq j$ as follows. To begin with, we have 
\begin{align}
    \norm{(R_{1,\opR}(k))_{ij}}_{\psi_2} 
    &= \norm{(Z_{1,\opR}(k))_{ij} (H_{1}(k))_{ij} - (1-\eta)p(k)(\Gshift_{1,\opR}^* (k))_{ij}}_{\psi_2}  \nonumber \\
    &\leq \norm{(Z_{1,\opR}(k))_{ij} - (1-\eta)(\Gshift_{1,\opR}^* (k))_{ij}}_{\psi_2} + \norm{(H_{1}(k))_{ij} - p(k)}_{\psi_2} \label{eq:outlier_var_proof_temp_1}
\end{align}
where we used triangle inequality and the fact $(H_{1}(k))_{ij}, \abs{(\Gshift_{1,\opR}^* (k))_{ij}} \leq 1$ (uniformly over $i,j,k$). Now the second term in \eqref{eq:outlier_var_proof_temp_1} is bounded as
$$\norm{(H_{1}(k))_{ij} - p(k)}_{\psi_2} \leq Q(p(k)) \leq \max_{k \in [T]} Q(p(k)).$$
To bound the first term in \eqref{eq:outlier_var_proof_temp_1}, we begin by writing 
$$(Z_{1,\opR}(k))_{ij} - (1-\eta)(\Gshift_{1,\opR}^* (k))_{ij} = B_{ij}(k) \cos (\varphi_{ij}(k)) - (B_{ij}(k) -\eta) \cos(\theta_i(k) - \theta_j(k))$$ 
where $B_{ij}(k)$ is a Bernoulli random variable with parameter $\eta$. Then it is not difficult to verify  
$$\norm{B_{ij}(k) \cos (\varphi_{ij}(k)) - (B_{ij}(k) -\eta) \cos(\theta_i(k) - \theta_j(k))}_{\psi_2} \lesssim Q(\eta) + \eta.$$ 
Thus, we have shown that 
$$\norm{(R_{1,\opR}(k))_{ij}}_{\psi_2} \lesssim Q(\eta) + \eta + \max_{k \in [T]} Q(p(k)) = : Q.$$

Upon applying the above bounds in conjunction with the Hanson Wright inequality, we readily obtain that with probability at least $1-\delta$ ,
\begin{align*}
    \|\projtaun R_{1,\opR}\|^2_F 
    &\leq C\paren{n^2 \tau f(\eta,\pmax,V) + Q^2 n \sqrt{\tau\log(1/\delta)} + Q^2 \log(1/\delta)}  \\
    &\leq C'\paren{n^2 \tau f(\eta,\pmax,V) + Q^2 n \sqrt{\tau} \log(1/\delta)}
\end{align*}
where we used $\sqrt{\log(1/\delta)} \leq \log(1/\delta)$ if $\delta \in (0,c)$ for a suitably small constant $c$. The proof is completed by using the same error bound on the other terms discussed earlier, applying a union bound, and adjusting the constants.

%-----------------------------------------------
% Proof of Theorem \ref{thm:error_outliers}
%
\section{Proof of Theorem \ref{thm:error_outliers}}\label{app:proof_thm_outliers}
Using \eqref{eq:bias_bd_outliers} and Lemma \ref{lem:variance_term_outliers} leads to the bound
\begin{equation*}
    \norm{\est{G} - D\Gshift^*}_F^2 \leq  C_1 \frac{\mu \bar{d}^2 n T^2 S_T}{\tau^2} \indic_{\tau < T} + C_2 \left(n^2\tau f(\eta, \pmax, V) + Q^2 n \sqrt{\tau} \log\paren{\frac1\delta}\right). 
\end{equation*}
Following the same argument as in Section \ref{sec:analysis_spiked_wigner}, we arrive at the choice $\tau = \tau^*$ which upon plugging in the RHS of the above bound leads to the stated error bound in part (i) (after minor simplification).

For part (ii), we simply note that if $1 \leq \paren{\frac{\mu \bar{d}^2 T^2 S_T}{n \tilde{f}(\eta,\pmax,V,Q,\delta)}}^{1/3} \lesssim T$ then $$\tau^* \asymp \paren{\frac{\mu \bar{d}^2 T^2 S_T}{n \tilde{f}(\eta,\pmax,V,Q,\delta)}}^{1/3},$$ and consequently,  \eqref{eq:outlier_gen_matden_bd} implies 
\begin{align*}
    \norm{\est{G} - D\Gshift^*}_F^2  
    &\lesssim n^{5/3} \tilde{f}^{2/3}(\eta,\pmax,V,Q,\delta) (\mu \bar{d}^2)^{1/3} T^{2/3} S_T^{1/3} \\
    &+ \max\set{\tilde{f}^{2/3}(\eta,\pmax,V,Q,\delta) n^{5/3} (\mu \bar{d}^2)^{1/3} T^{2/3} S_T^{1/3}, \mu \bar{d}^2 n S_T}.
\end{align*}
We then arrive at the simplified bound in \eqref{eq:outlier_simp_matden_bd} provided $T \geq \frac{\mu \bar{d}^2 S_T}{n \tilde{f}^{2/3}(\eta,\pmax,V,Q,\delta)}$.

%---------------------------------------------
% Proof of Proposition for spectral recovery
%---------------------------------------------
\section{Proof of Proposition \ref{prop:recovery_spectral}}\label{proof:prop_recovery}
We will use the following useful inequality from \cite{Liu2017}.
\begin{proposition}[\cite{Liu2017}] \label{prop:entry_proj_ineq}
For any $q \in [1,\infty]$, $w  \in \mathbb{C}^n$ and $g \in \calC_n$, we have
\begin{equation*}
    \norm{\calP_{\calC_n}(w) - g}_q \leq 2 \norm{w - g}_q.
\end{equation*}
\end{proposition}
\begin{proof}[Proof of Proposition \ref{prop:recovery_spectral}]
For any $k \in [T]$ we obtain
\begin{align}
  \norm{\est{g}(k) - g^*(k)}_2 
  &=  \left \|\calP_{\calC_n}\paren{\est{g}'(k)}e^{-\iota \teonek}-g^*(k)\right\|_2 \nonumber \\
  &=  \left \|\calP_{\calC_n}\paren{\sqrt{n}\est{g}'(k) e^{-\iota \teonek}} -g^*(k)\right\|_2 \nonumber \\
  &\leq 2\left \|\sqrt{n}\est{g}'(k) e^{-\iota \teonek} -g^*(k)\right\|_2 \tag{using Proposition \ref{prop:entry_proj_ineq}} \nonumber \\
   &\leq 2\left\|\sqrt{n}\est{g}'(k) e^{-\iota \varphi(k)} - g^*(k)\right\|_2 
   + 2\sqrt{n}\left\|\est{g}'(k)\left(e^{-\iota \teonek}-e^{-\iota \phik }\right)\right\|_2, \label{eq:propspec_triang_ineq}
\end{align}
where \eqref{eq:propspec_triang_ineq} follows from the triangle inequality. 
Using \eqref{eq:block_daviskahan}, \eqref{eq:symmetr_norm_ineq}, we bound the first term as
\begin{equation}\label{eq:bound_firs_term_propspec}
    \left \|\sqrt{n}\est{g}'(k)e^{-\iota \phik} -g^*(k)\right\|_2\leq C \frac{\|\est{G}'(k)-s(k)\Gshift^*(k)\|_{F}}{s(k) \sqrt{n}} \leq C \frac{\|\est{G}(k)-s(k)\Gshift^*(k)\|_{F}}{s(k) \sqrt{n}}.
\end{equation}
For the second term, we note that  
\begin{align}
    \sqrt{n}\left\|\est{g}'(k)\left(e^{-\iota \teonek}-e^{-\iota \phik }\right)\right\|_2 
    & = \sqrt{n}\left|e^{\iota \teonek}-e^{\iota \phik }\right| \tag{since $\norm{\est{g}'(k)}_2 = 1$} \nonumber \\
    &= \sqrt{n}\abs{\calP_{\calC_1}\paren{e^{\iota \teonek}} - e^{\iota \phik}} \nonumber \\
    &= \sqrt{n} \abs{\calP_{\calC_1}\paren{\sqrt{n} \underbrace{\abs{\est{g}_1'(k)} e^{\iota \teonek}}_{= \est{g}_1'(k)}} - e^{\iota \phik}} \nonumber \\
    &\leq 2\sqrt{n} \abs{\sqrt{n}\est{g}_1'(k) - e^{\iota \phik}} \tag{using Proposition \ref{prop:entry_proj_ineq}} \nonumber \\
    &\leq 2{n}  \left\|\est{g}'(k)-e^{\iota \phik}\frac{g^*(k)}{\sqrt n}\right\|_2 \tag{ since $g_1^*(k) = 1$} \nonumber \\
    &\leq C \frac{\|\est{G}(k)-s(k)\Gshift^*(k)\|_{F}}{s(k)}  \label{eq:bound_sec_term_propspec}  \qquad \text{(using \eqref{eq:block_daviskahan}, \eqref{eq:symmetr_norm_ineq})}. 
\end{align}
Plugging \eqref{eq:bound_firs_term_propspec}, \eqref{eq:bound_sec_term_propspec} in \eqref{eq:propspec_triang_ineq} leads to the stated bound.
\end{proof}
%---------------------------
% Additional simluations
%---------------------------
%
\section{Additional simulations}
\subsection{RMSE versus noise-level}
%
% \globaltrs, \localtrs \ and \ \matdenoising,
%
%
We show additional plots for RMSE versus noise-level ($\gamma$) for both the AGN (Fig. \ref{fig:plots_123_rmse_noise_wigner_T100_20runs}) and Outliers (Fig. \ref{fig:plots_123_rmse_noise_outliers_p0p2_T100_20runs}) noise models but for $T = 100$ (recall we considered $T = 20$ in Figs. \ref{fig:plots_123_rmse_noise_wigner_T20_20runs} and \ref{fig:plots_123_rmse_noise_outliers_p0p2_T20_20runs} in Section \ref{sec:experiments}). We see for the AGN model that out of the three methods, $\matdenoising$ is  most robust for moderate to high levels of noise. For the Outliers model, we find that $\globaltrs$ performs the best for low level of noise, while $\matdenoising$ is again the best method for moderate to high levels of noise.
%
%----------------------------------------------------------------
% Figures for RMSE versus noise level for Wigner (PLOTS_{1,2,3})
%----------------------------------------------------------------
 \begin{figure}[!ht]
\centering
\begin{subfigure}{.5\textwidth}
  \centering
  \includegraphics[width=\linewidth]{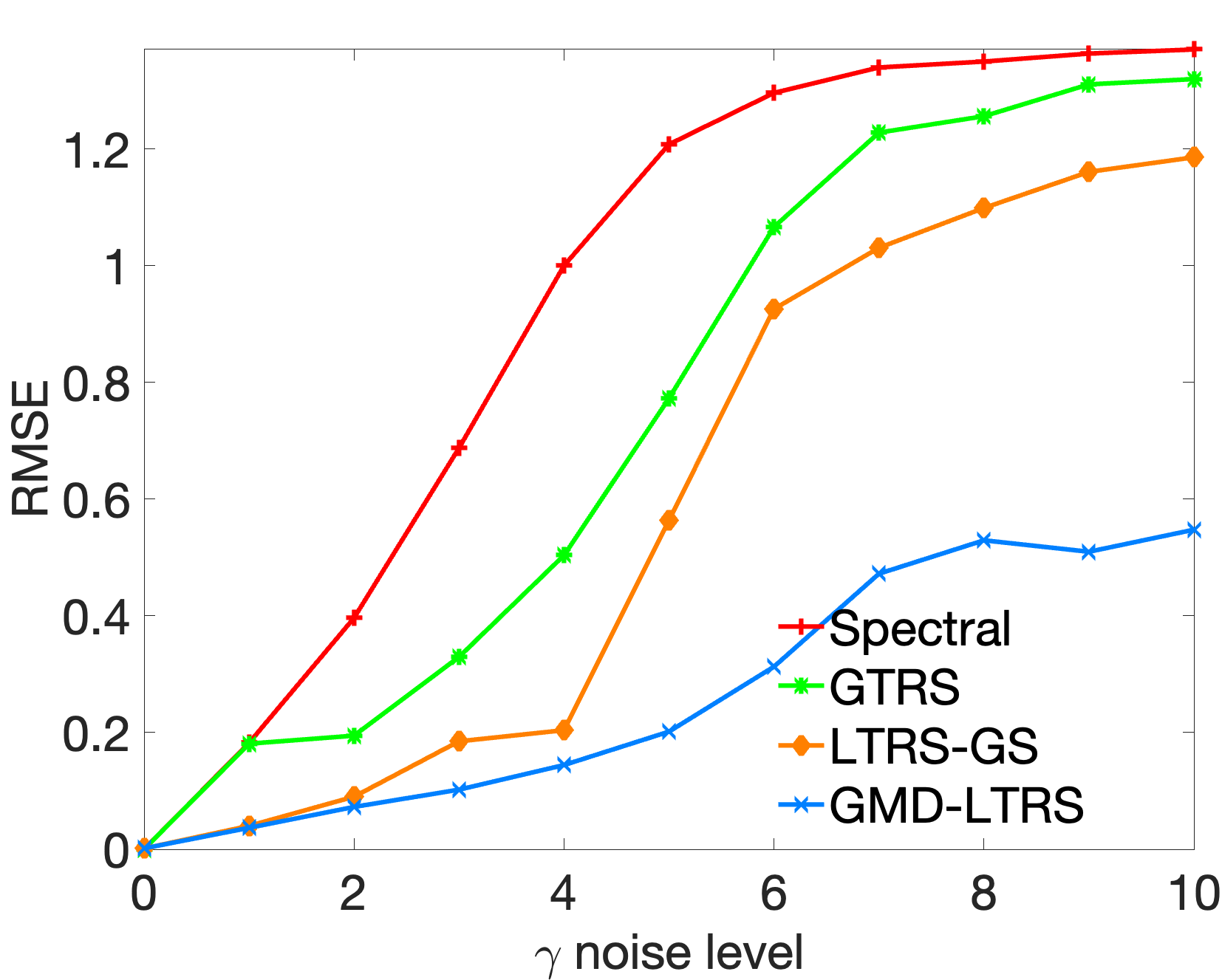}
  \caption{$S_T = 1/T$}
  %\label{fig:transync_alpha_ST_1}
\end{subfigure}%
\begin{subfigure}{.5\textwidth}
  \centering
  \includegraphics[width=\linewidth]{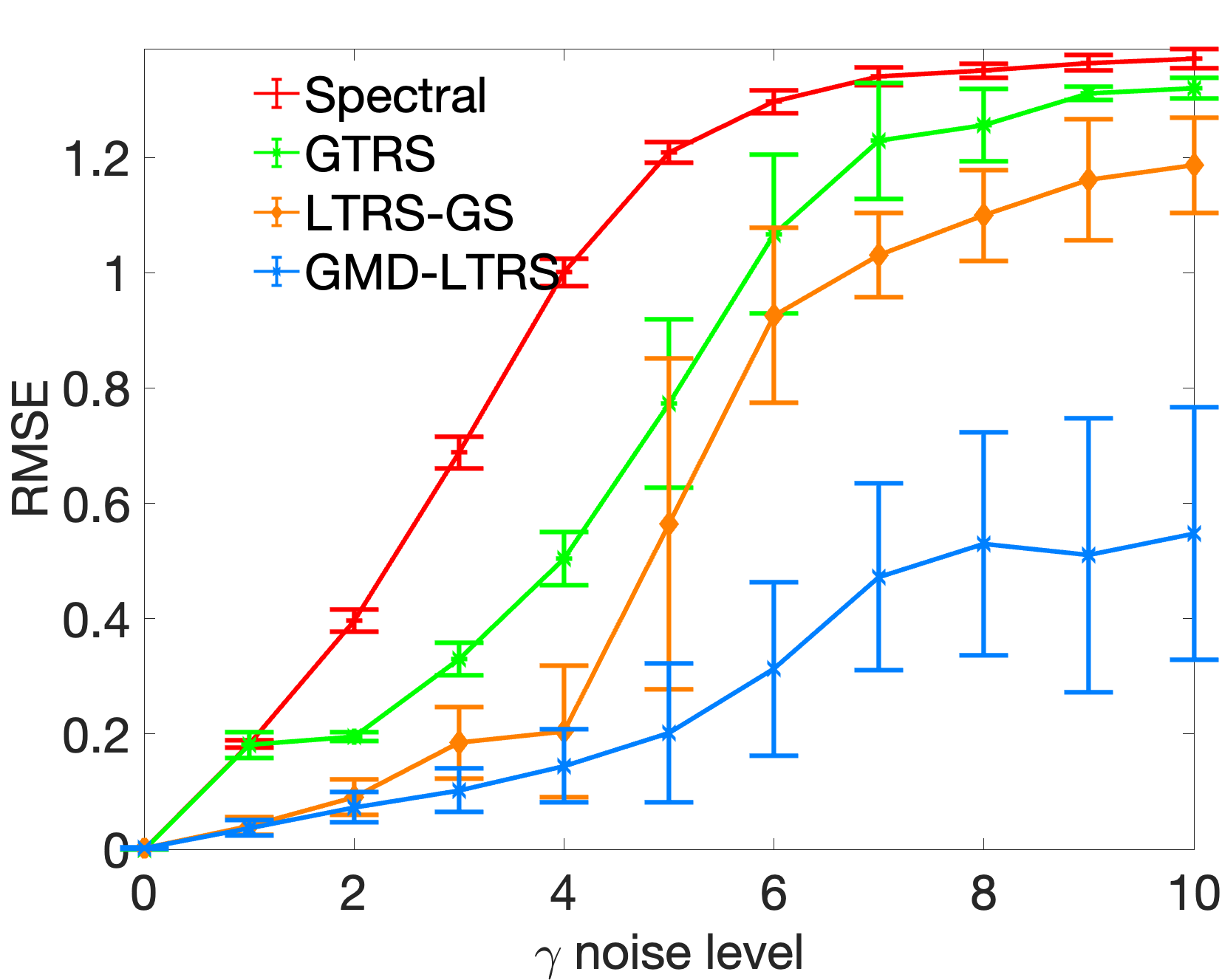}
  \caption{$S_T = 1/T$}
  %\label{fig:transync_alpha_ST_05}
\end{subfigure}%
%\caption{}
%\label{fig:plots_3_rmse_T_wigner_sigma_3_20runs}
\hfill
\begin{subfigure}{.5\textwidth}
  \centering
  \includegraphics[width=\linewidth]{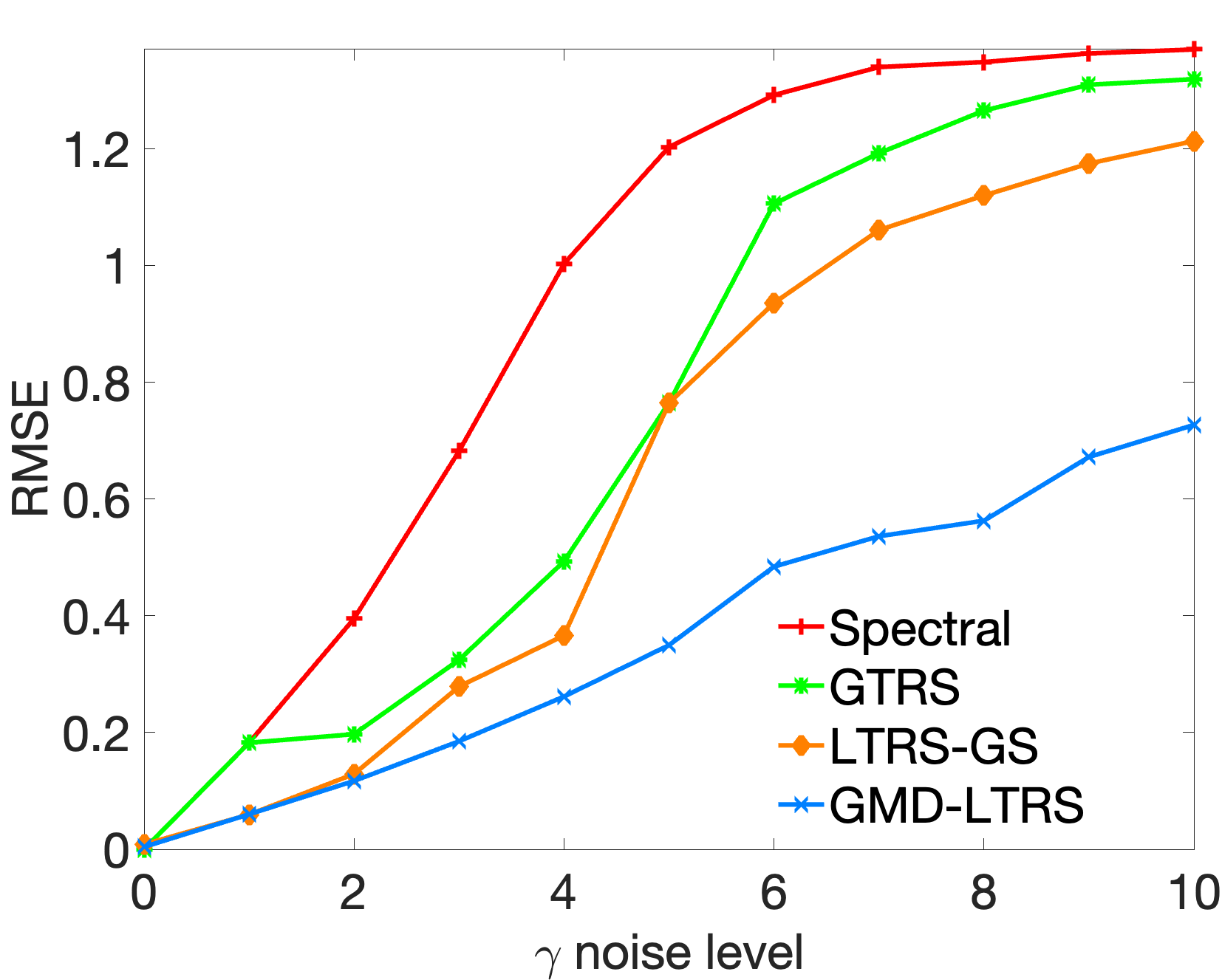}
  \caption{$S_T = 1$}
  %\label{fig:transync_alpha_ST_m05}
\end{subfigure}%
\begin{subfigure}{.5\textwidth}
  \centering
  \includegraphics[width=\linewidth]{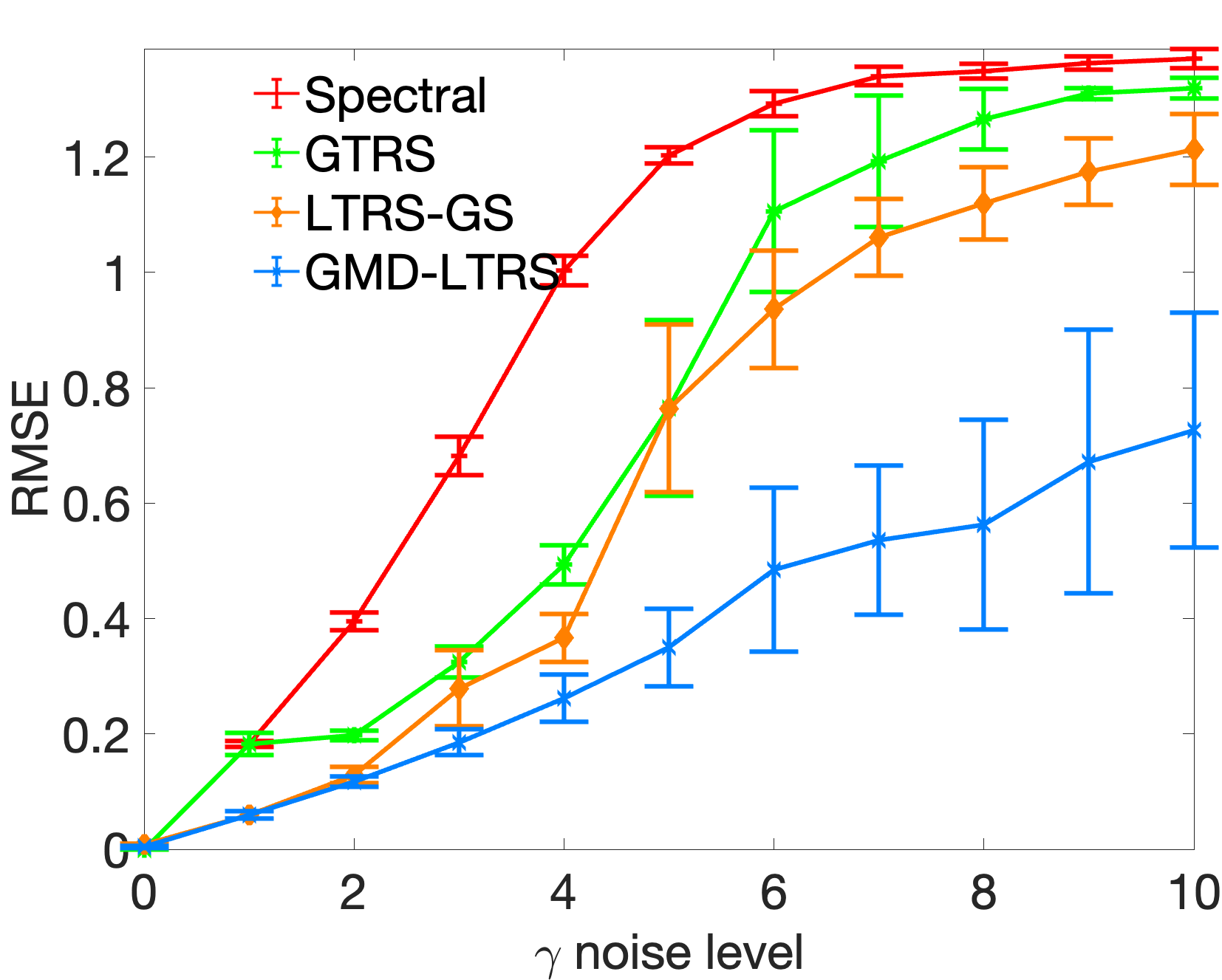}
  \caption{$S_T = 1$}
  %\label{fig:transync_alpha_ST_0}
\end{subfigure}%
\hfill
\begin{subfigure}{.5\textwidth}
  \centering
  \includegraphics[width=\linewidth]{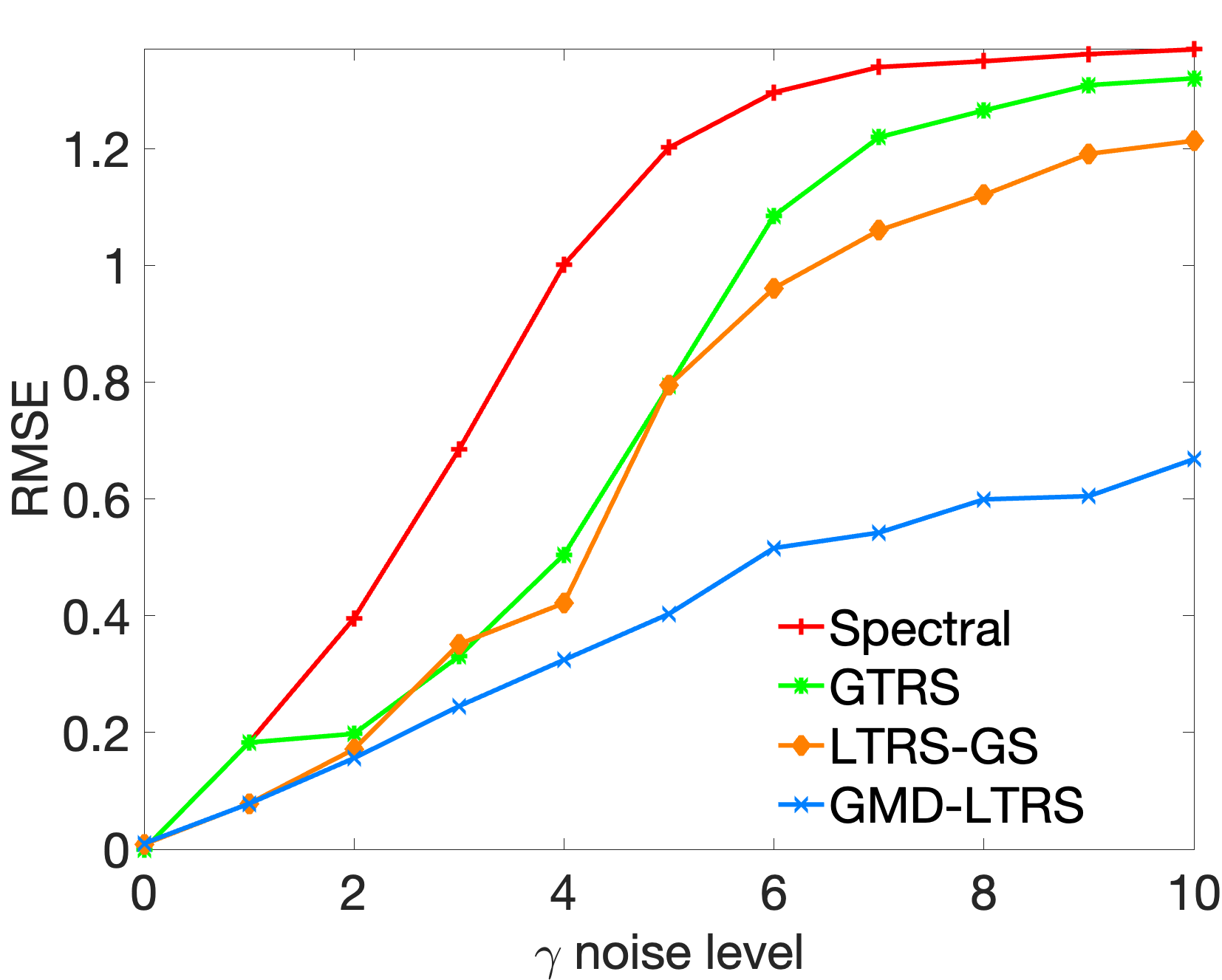}
  \caption{$S_T = T^{1/4}$}
  %\label{fig:transync_alpha_ST_m05}
\end{subfigure}%
\begin{subfigure}{.5\textwidth}
  \centering
  \includegraphics[width=\linewidth]{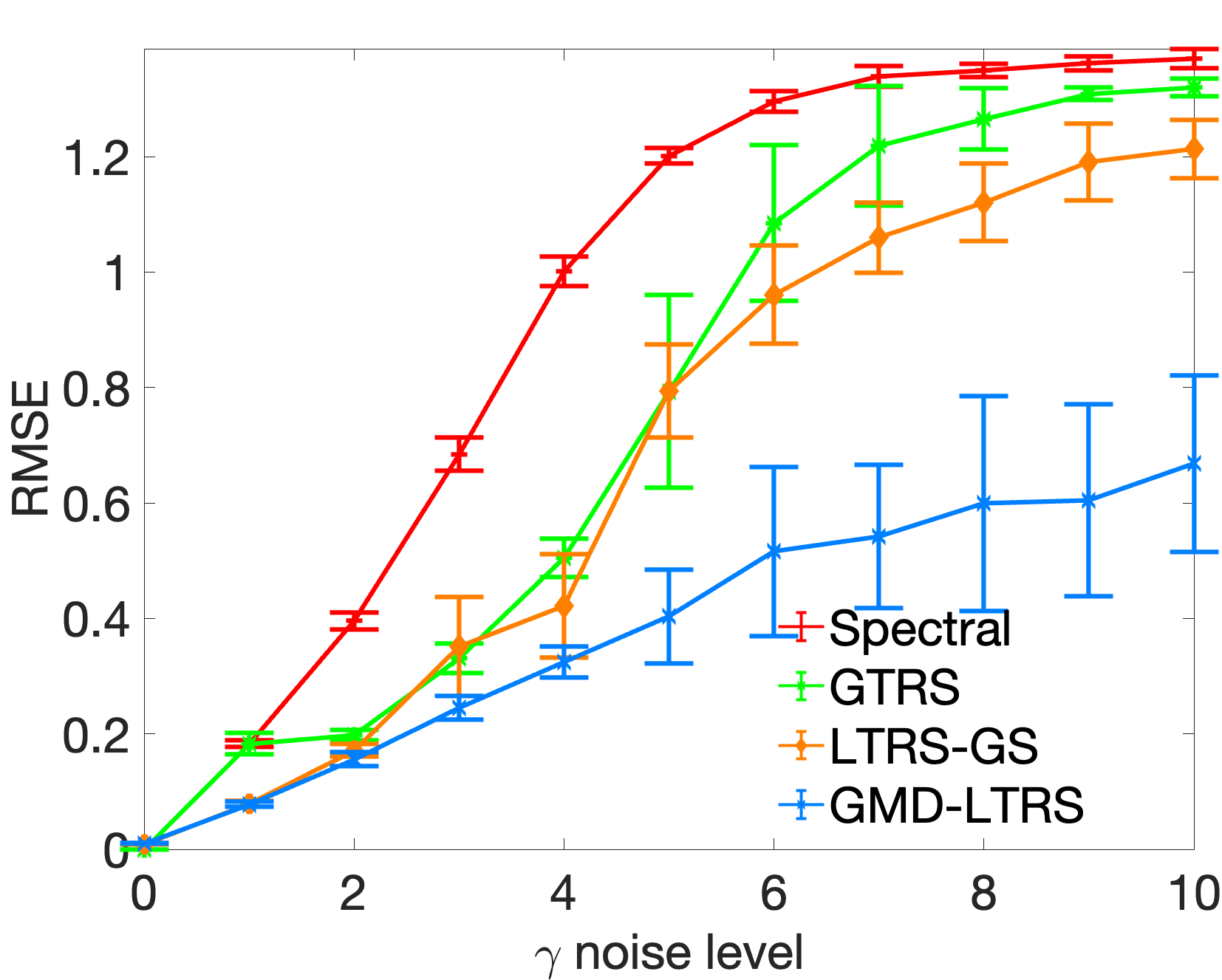}
  \caption{$S_T = T^{1/4}$}
  %\label{fig:transync_alpha_ST_0}
\end{subfigure}%
\caption{RMSE versus $\gamma (= \sigma)$ for the AGN model ($n=30$, $T = 100$) with $S_T \in \set{1/T, 1, T^{1/4}}$, results averaged over $20$ MC runs. We choose $\lambda_{\text{scale}} = 10$ for $\globaltrs$. Results on left show average RMSE,  the right panel shows average RMSE $\pm$ standard deviation.}
\label{fig:plots_123_rmse_noise_wigner_T100_20runs}
\end{figure}

%
%----------------------------------------------------------------
% Figures for RMSE versus noise level for Outliers model (PLOTS_{1,2,3})
%----------------------------------------------------------------
 \begin{figure}[!ht]
\centering
\begin{subfigure}{.5\textwidth}
  \centering
  \includegraphics[width=\linewidth]{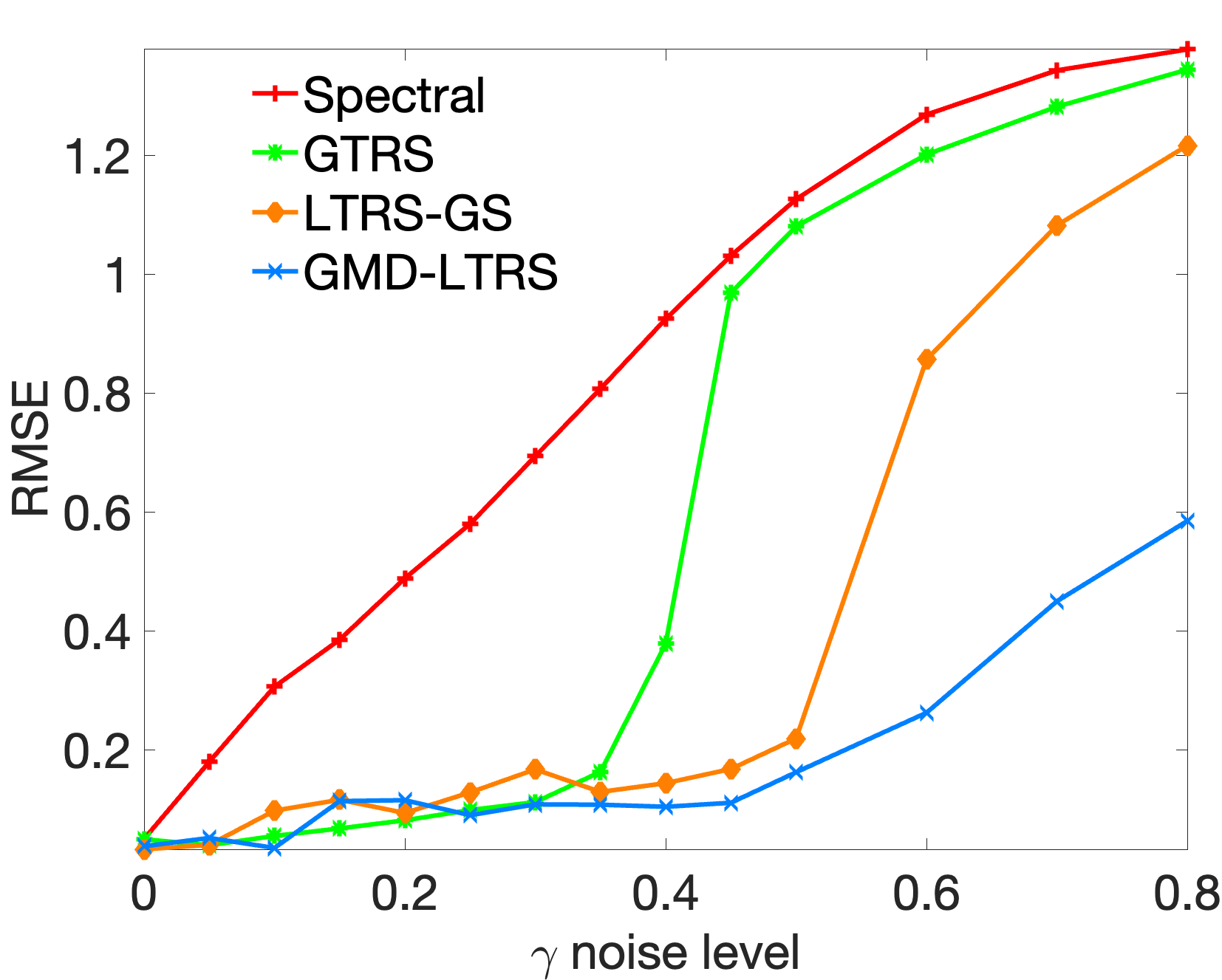}
  \caption{$S_T = 1/T$}
  %\label{fig:transync_alpha_ST_1}
\end{subfigure}%
\begin{subfigure}{.5\textwidth}
  \centering
  \includegraphics[width=\linewidth]{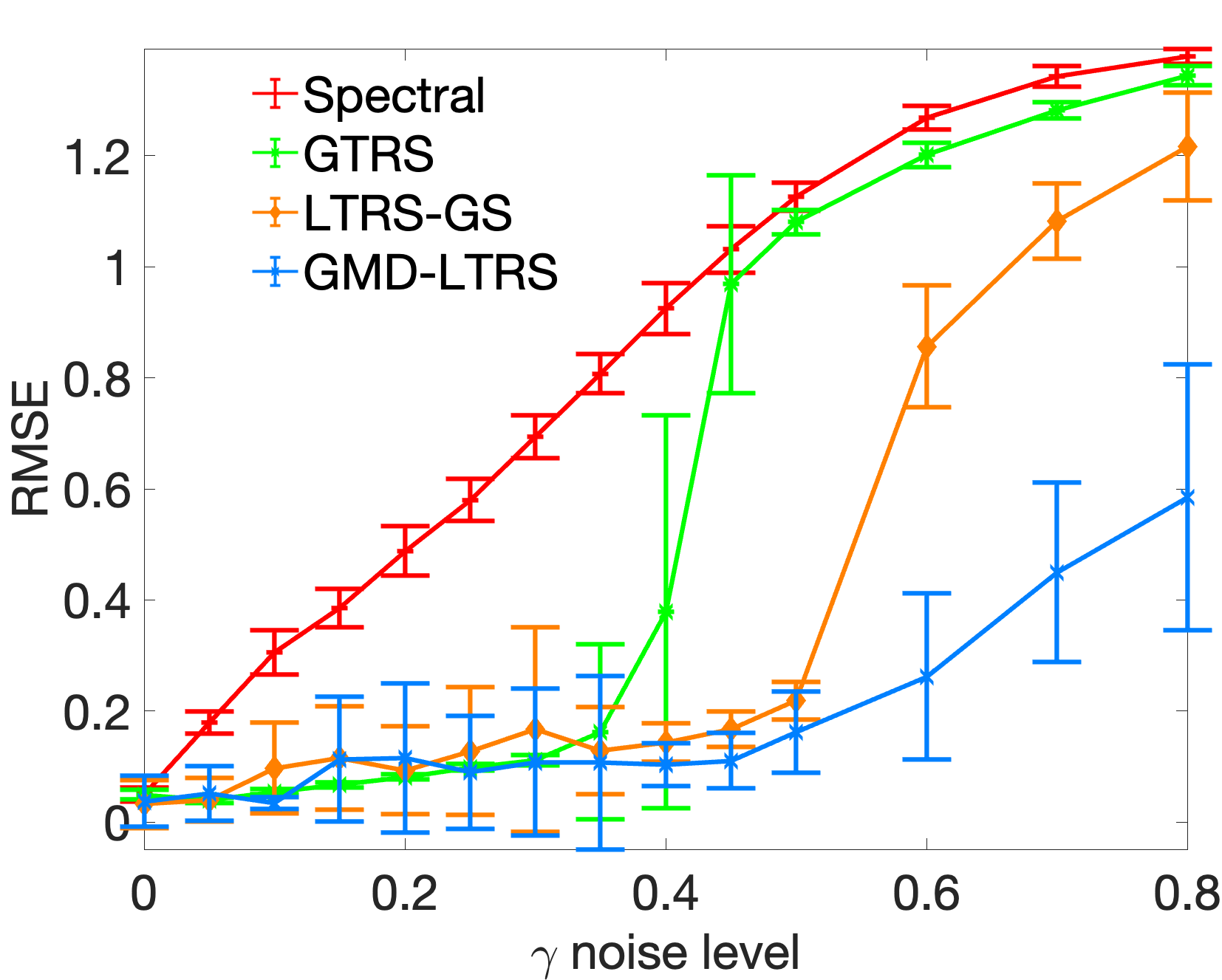}
  \caption{$S_T = 1/T$}
  %\label{fig:transync_alpha_ST_05}
\end{subfigure}%
%\caption{}
%\label{fig:plots_3_rmse_T_wigner_sigma_3_20runs}
\hfill
\begin{subfigure}{.5\textwidth}
  \centering
  \includegraphics[width=\linewidth]{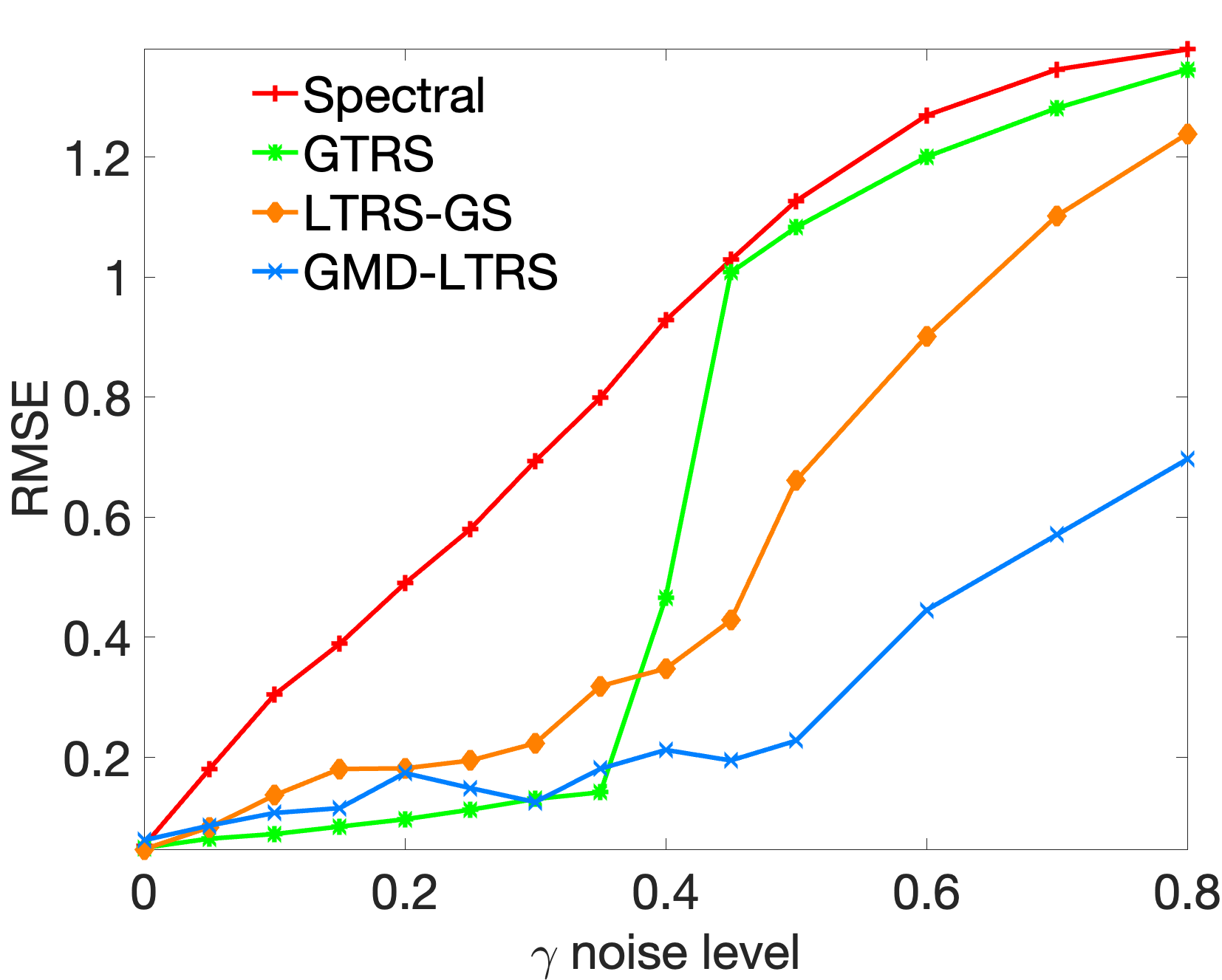}
  \caption{$S_T = 1$}
  %\label{fig:transync_alpha_ST_m05}
\end{subfigure}%
\begin{subfigure}{.5\textwidth}
  \centering
  \includegraphics[width=\linewidth]{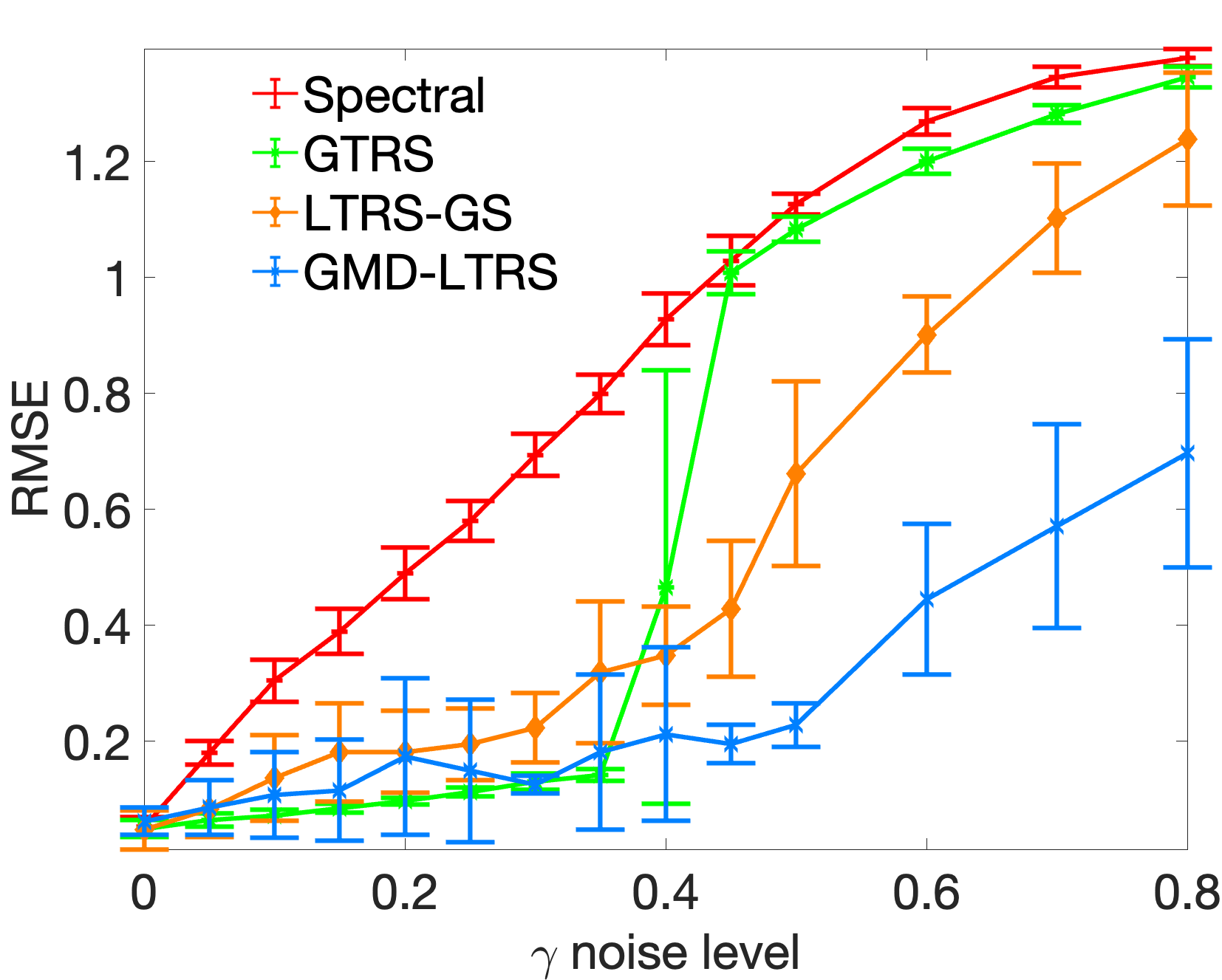}
  \caption{$S_T = 1$}
  %\label{fig:transync_alpha_ST_0}
\end{subfigure}%
\hfill
\begin{subfigure}{.5\textwidth}
  \centering
  \includegraphics[width=\linewidth]{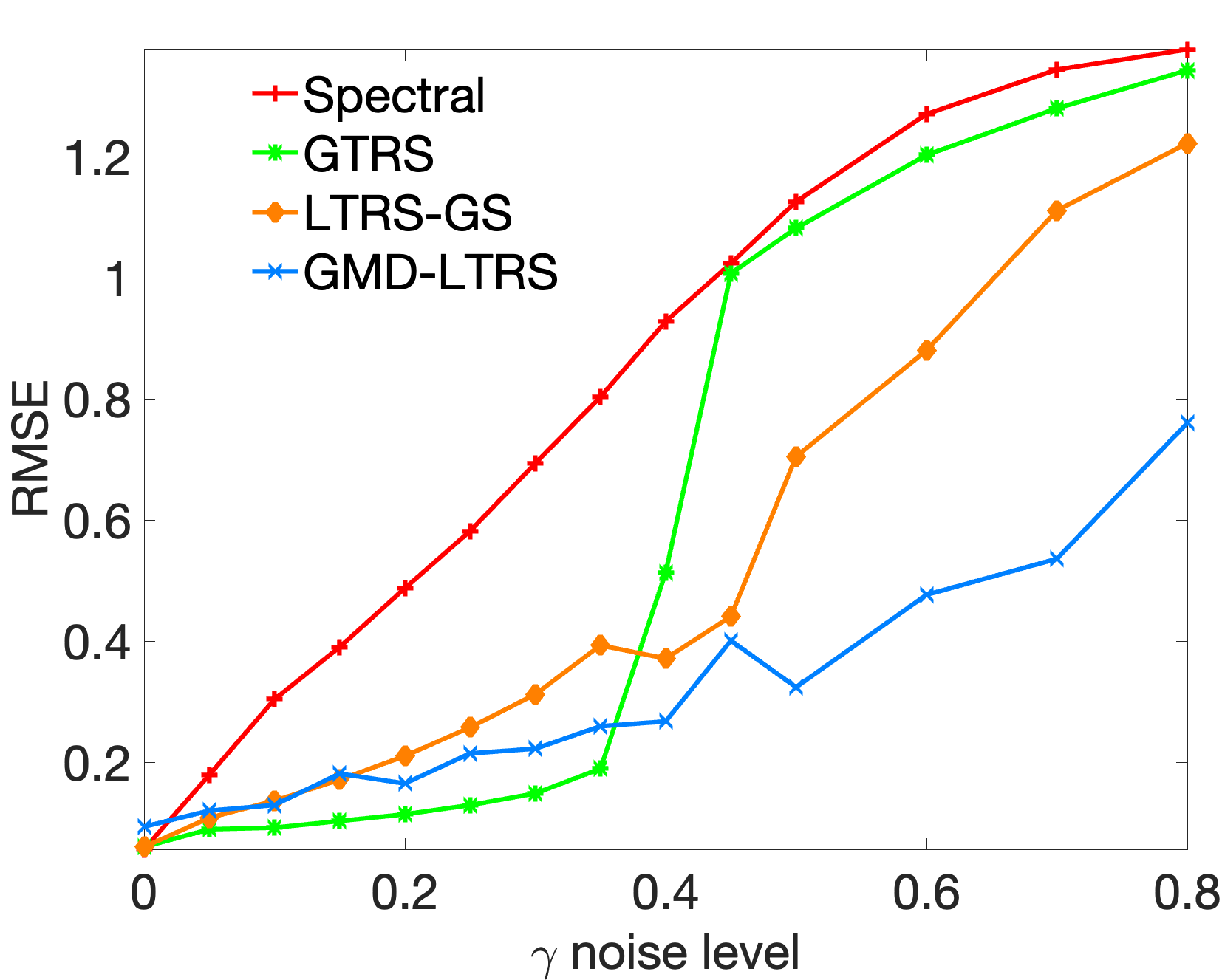}
  \caption{$S_T = T^{1/4}$}
  %\label{fig:transync_alpha_ST_m05}
\end{subfigure}%
\begin{subfigure}{.5\textwidth}
  \centering
  \includegraphics[width=\linewidth]{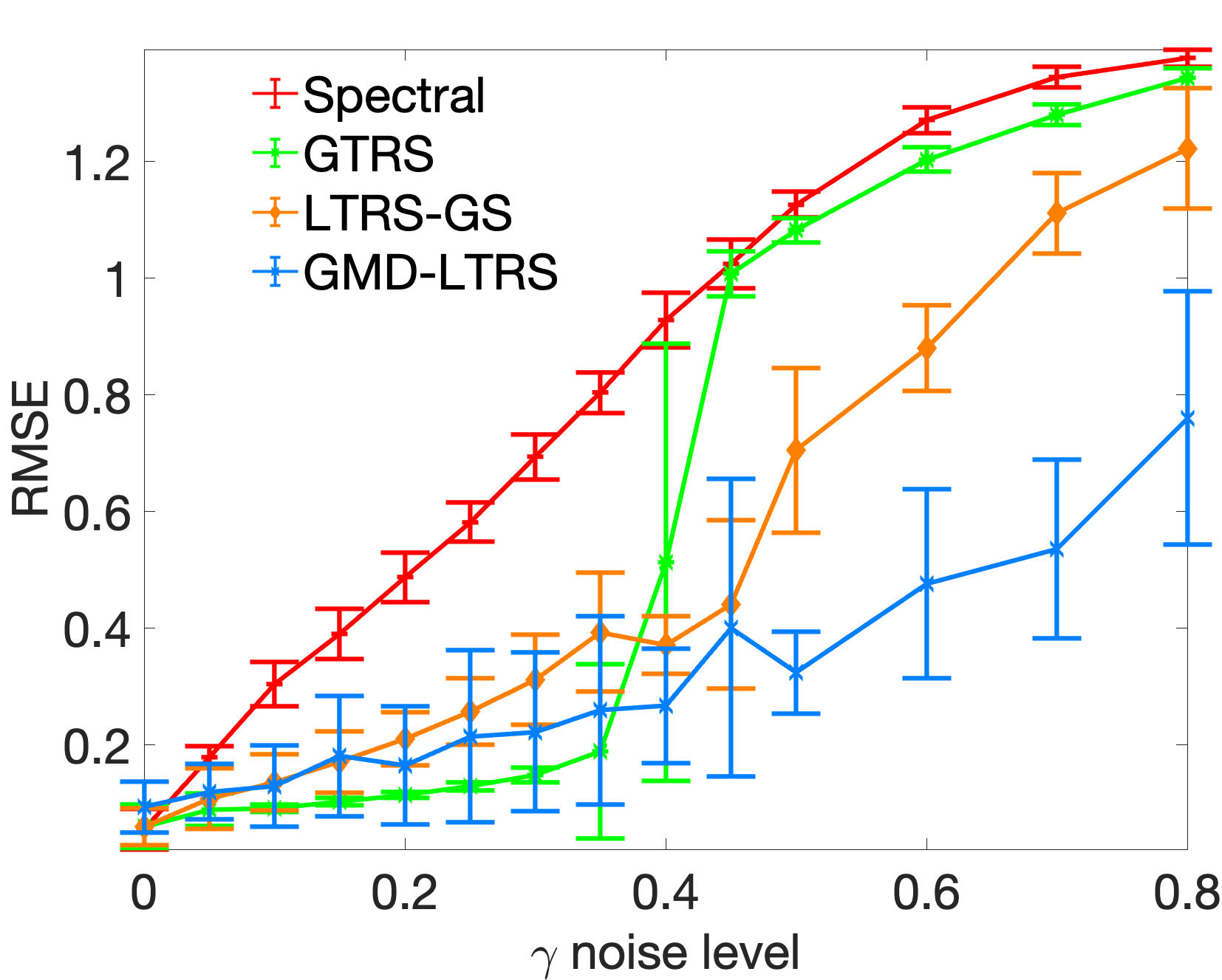}
  \caption{$S_T = T^{1/4}$}
  %\label{fig:transync_alpha_ST_0}
\end{subfigure}%
\caption{RMSE versus $\gamma (= \eta)$ for the Outliers model ($n=30$, $p = 0.2$, $T = 100$) with $S_T \in \set{1/T, 1, T^{1/4}}$, results averaged over $20$ MC runs. We choose $\lambda_{\text{scale}} = 10$ for $\globaltrs$. Results on left show average RMSE,  the right panel shows average RMSE $\pm$ standard deviation.}
\label{fig:plots_123_rmse_noise_outliers_p0p2_T100_20runs}
\end{figure}

%------------------------
% Simulations for PPM
%-----------------------
%
\subsection{Simulations for a Projected Power Method (PPM)}
We now discuss an iterative projected power method for solving \eqref{eq:opt_denoising_vector} that needs to be initialized with a ``seed'' $\tilde{g}^{(0)} \in \calC_{(n-1)T}$  that is sufficiently close to the ground truth $\tilde{g}^*$.  This method, also referred to as the generalized power method in the literature, has had success in tackling several non-convex problems such as orthogonal procrustes \cite{ling2021generalized}, multireference alignment \cite{Chen2016ThePP}, graph clustering \cite{Manchoso}, graph matching \cite{araya2023seeded} and group synchronization \cite{boumal2016,GaoZhang, Liu2017}, to name a few. 

The version that we consider for dynamic synchronization is outlined in Algorithm \ref{algo:PPM}, dubbed $\ppm$. At its core, it involves iteratively refining the seed by sequentially performing the following operations.
\begin{itemize}
    \item \textbf{Power step}. Pre-multiply the current iterate by the gradient of the objective function in \eqref{eq:opt_denoising_vector}.
    
    \item \textbf{Projection step}. The  result is projected onto the set $\calC_{(n-1)T}$ using the operator $\calP_{\calC_{(n-1)T}}$.
\end{itemize}
For initialization, we will employ $\globaltrs, \matdenoising$, $\localtrs$ and the naive spectral algorithm considered in the experiments (Section \ref{sec:experiments}). For $\ppm$, the parameter $\lambda$ is chosen via the same grid search procedure as for $\globaltrs$ (see Section \ref{sec:experiments}); we chose $\lambda_{scale} = 10$ for both $\ppm$ and $\globaltrs$. Surprisingly, we found that $\ppm$ does not improve the performance of the seed, and in general worsens the estimation error. While a more detailed investigation of $\ppm$ is left for future work, we illustrate its performance in Fig. \ref{fig:plots_rmse_ppm_noise_wigner_T100_20runs} for the AGN model. Similar behaviour was observed for the Outliers model, and this was the case for different choices of $\lambda_{scale}$ for $\ppm$.
\begin{algorithm}[!ht]
    \caption{Projected Power Method for Dynamic Synchronization (\ppm)}\label{algo:PPM}
    \begin{algorithmic}[1]
    \State {\bf Input:} Matrix $A\in\matC^{nT\times n}$ of observations, parameter $\lambda >0$, number of iterations $S$, initial point $\tilde{g}^{(0)}$.
    \State Form $\tilde{A}_{block}=\blkdiag\big(\tilde{A}(1),\cdots,\tilde{A}(T)\big)\in\matC^{(n-1)T\times (n-1)T}$.
    %\State Initialize $\tilde{g}^{(0)} \in \calC_{(n-1)T}$.
    \For{$s=1,2\dots ,S$}
    \State $\tilde{g}^{(s)} = \calP_{\calC_{(n-1)T}}\left((\tilde{A}_{block}-\lambda MM^\top\otimes I_{n-1})\tilde{g}^{(s-1)}+b\right)$.
    \EndFor
     \For{$k=1,\dots ,T$}
     \State Define $\est g(k)=\begin{bmatrix}
         1\\{\tilde{g}}^{(S)}(k)
     \end{bmatrix}$.
     \EndFor
    \State {\bf Output:} Estimator $\est g=\concat\big(\est g(1),\cdots,\est g(T)\big)\in \calC_{nT}$.
    \end{algorithmic}
\end{algorithm}
%
%
%
%----------------------------------------------------------------
% Figures for RMSE versus noise level for Wigner (PLOTS_{1,2,3}) for PPM with different initializations
%----------------------------------------------------------------
 \begin{figure}[!htp]
\centering
\begin{subfigure}{.5\textwidth}
  \centering
  \includegraphics[width=1\linewidth]{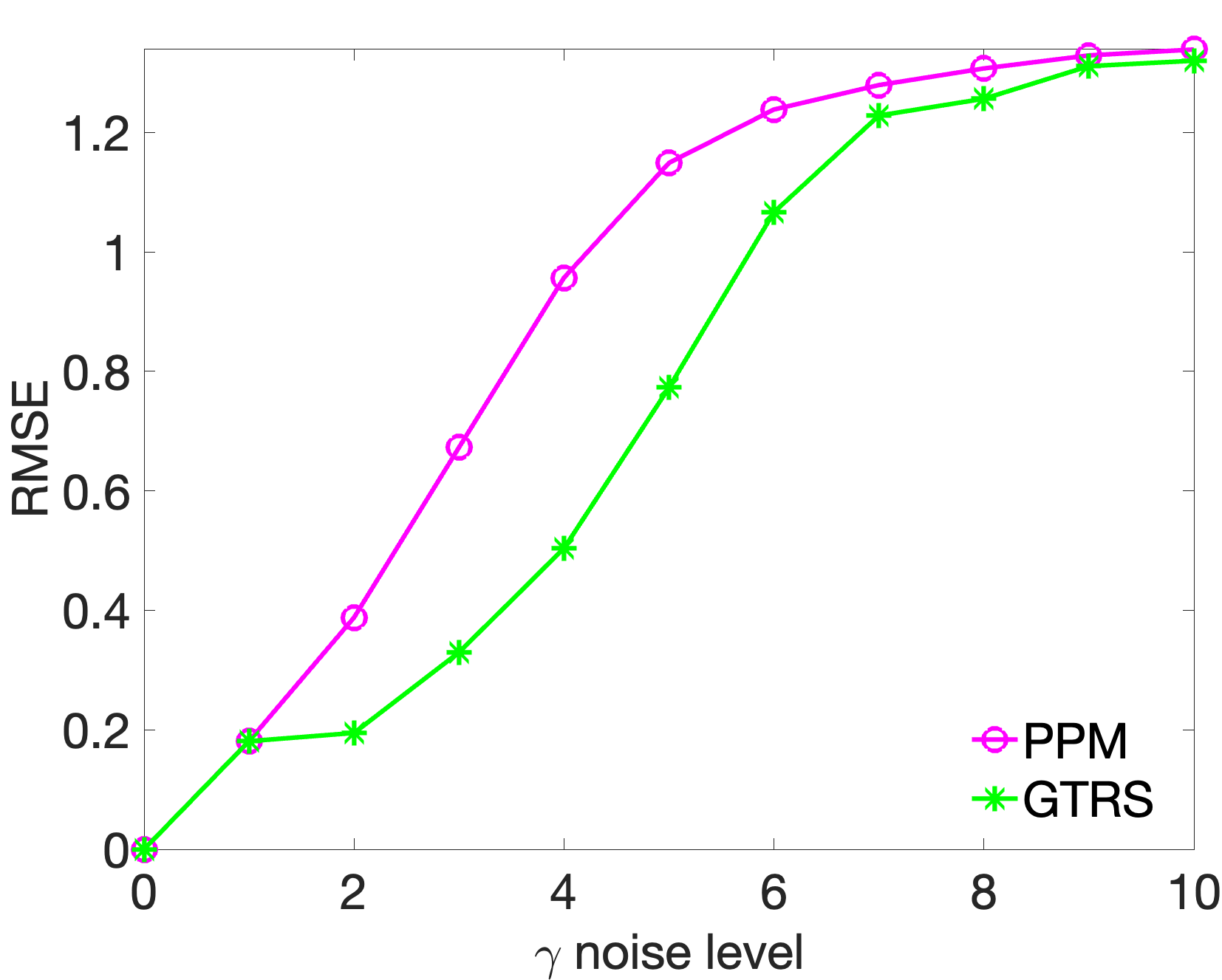}
  \caption{$\globaltrs$}
  %\label{fig:transync_alpha_ST_1}
\end{subfigure}%
\begin{subfigure}{.5\textwidth}
  \centering
  \includegraphics[width=1\linewidth]{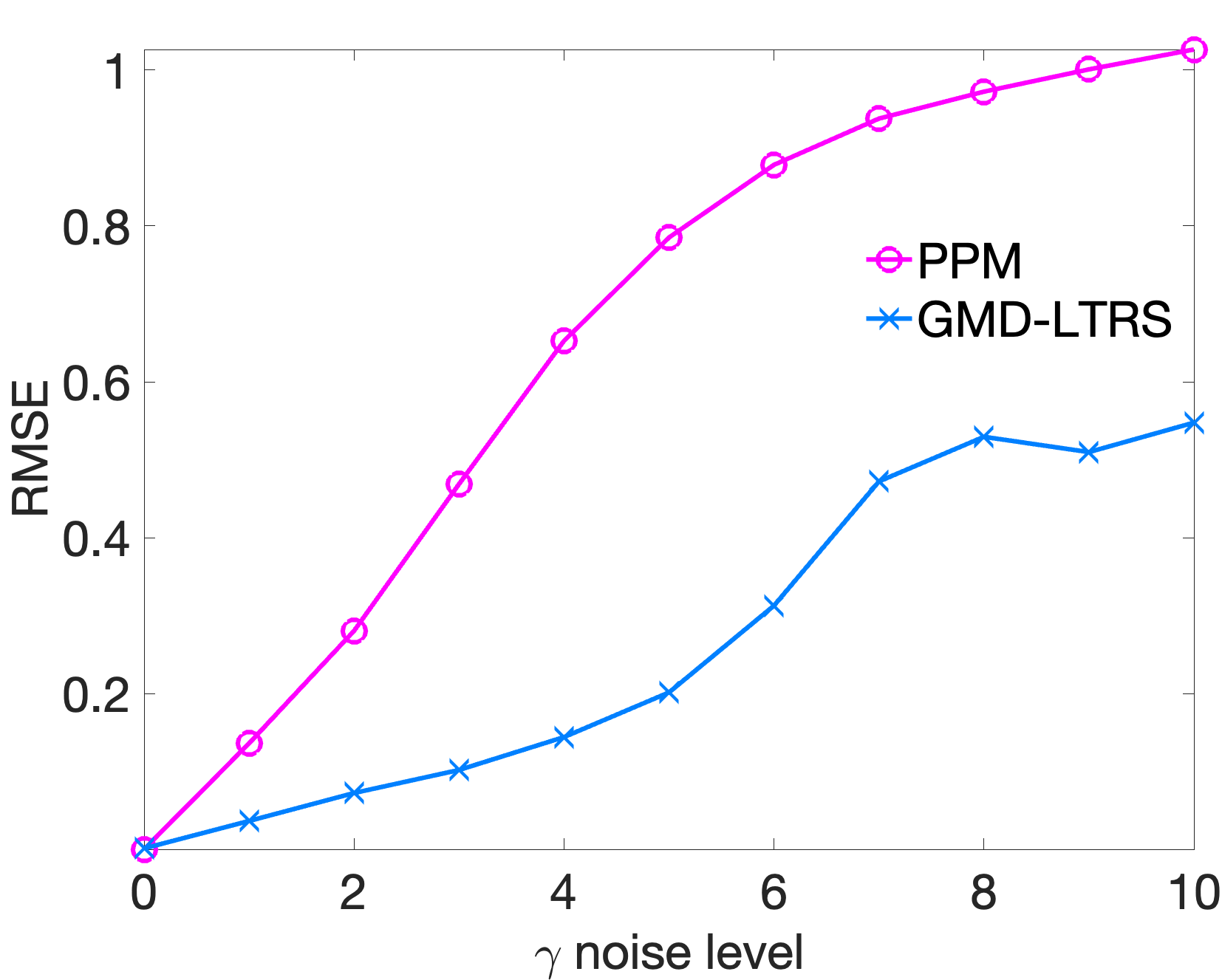}
  \caption{$\matdenoising$}
  %\label{fig:transync_alpha_ST_05}
\end{subfigure}%
%\caption{}
%\label{fig:plots_3_rmse_T_wigner_sigma_3_20runs}
\hfill
\begin{subfigure}{.5\textwidth}
  \centering
  \includegraphics[width=1\linewidth]{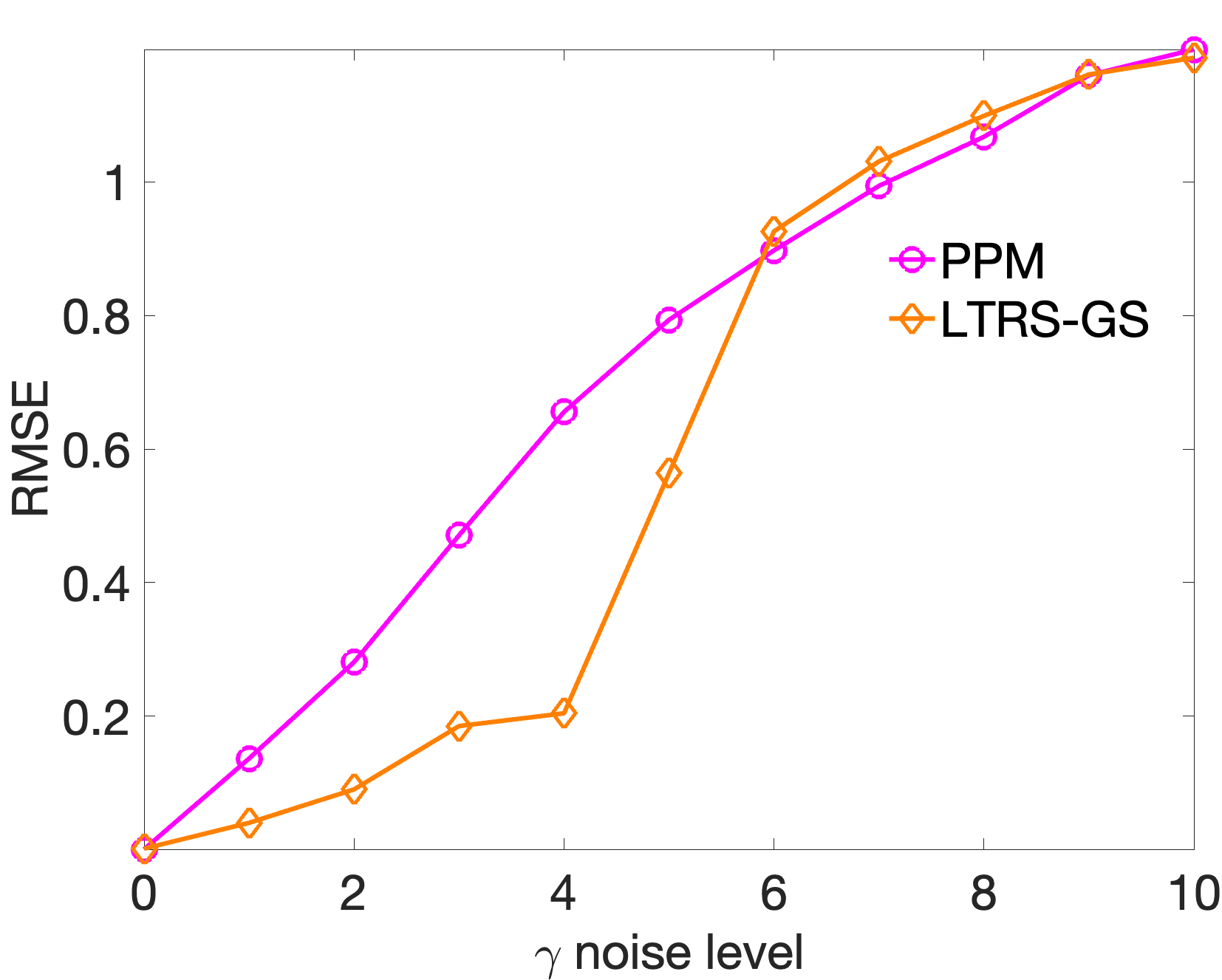}
  \caption{$\localtrs$}
  %\label{fig:transync_alpha_ST_m05}
\end{subfigure}%
\begin{subfigure}{.5\textwidth}
  \centering
  \includegraphics[width=1\linewidth]{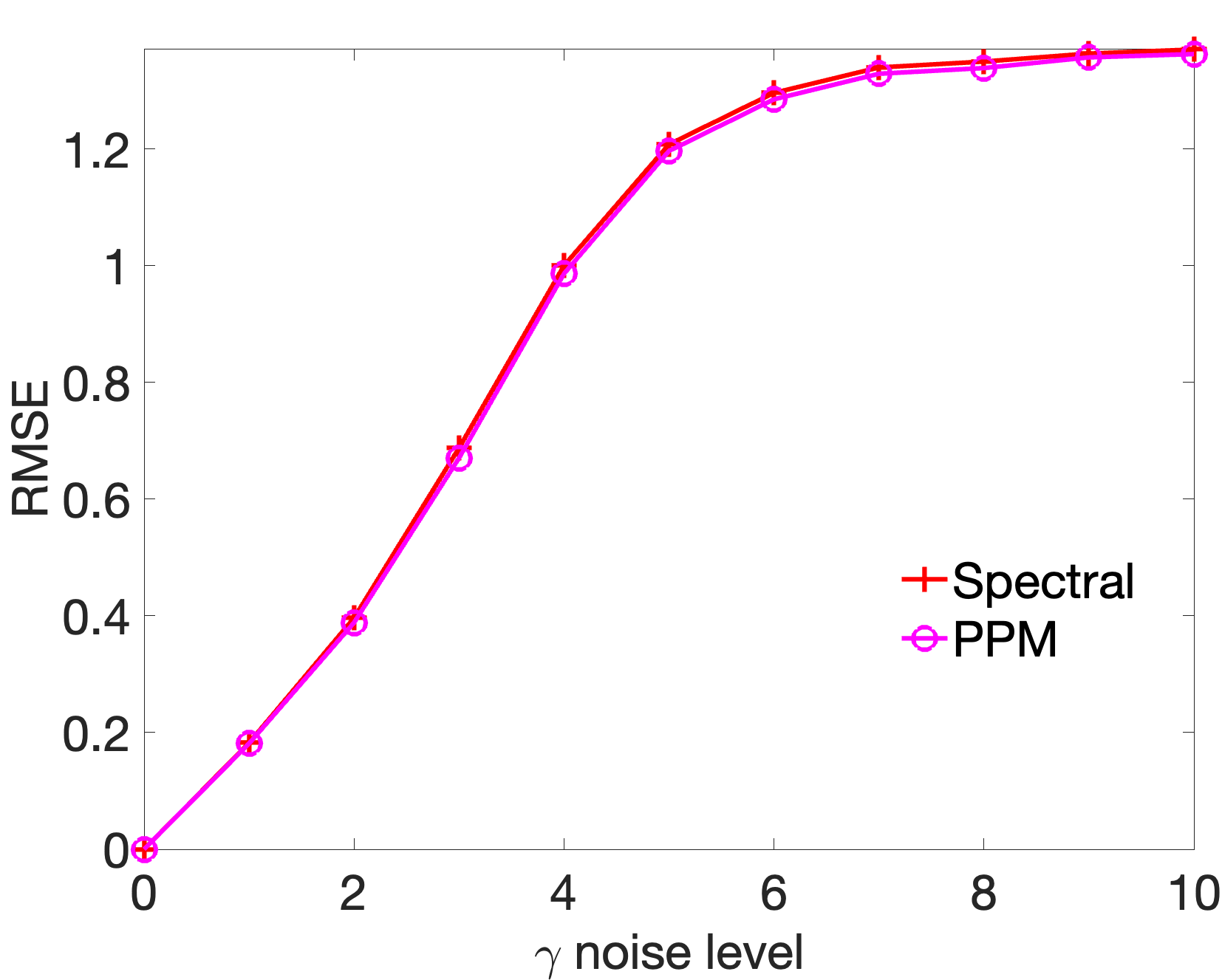}
  \caption{Spectral}
  %\label{fig:transync_alpha_ST_0}
\end{subfigure}
\caption{RMSE versus $\gamma (= \sigma)$ for $\ppm$ for the AGN model ($n=30$, $T = 100$) with $S_T = 1/T$ and different initializations. Results averaged over $20$ MC runs. We choose $\lambda_{\text{scale}} = 10$ for $\globaltrs$ and $\ppm$.}
\label{fig:plots_rmse_ppm_noise_wigner_T100_20runs}
\end{figure}

\end{document}